%% file: main.tex
\documentclass[10pt,twocolumn,letterpaper]{article}

\usepackage{cvpr}              


\usepackage{graphicx, lipsum}
\usepackage{url}            
\usepackage{booktabs}       
\usepackage{amsfonts}       
\usepackage{nicefrac}       
\usepackage{microtype}      
\usepackage{xcolor}         
\usepackage{enumitem}
\usepackage{pifont}
\usepackage{graphicx}
\usepackage{cancel}
\usepackage{subcaption}
\usepackage{afterpage}
\usepackage{overpic}
\usepackage{amsmath,amsfonts,amssymb,amsthm}
\usepackage{mathtools}
\usepackage{multirow}
\usepackage{indentfirst}
\usepackage{stfloats}
\usepackage{tablefootnote}
\usepackage{algorithm}
\usepackage{algpseudocode}
\usepackage{adjustbox}
\usepackage{fontawesome5}
\usepackage[accsupp]{axessibility}
\input{math_commands.tex}

\newcommand{\blfootnote}[1]{%
  \begingroup
  \renewcommand{\thefootnote}{}%
  \begin{NoHyper}
  \footnotetext{#1}%
  \end{NoHyper}%
  \endgroup
}
\def\paperID{*****} 
\def\confName{CVPR}
\def\confYear{2026}

\title{Dynamic Momentum Recalibration in Online Gradient Learning}

\author{%
Zhipeng Yao \textsuperscript{1, 2} $^{\star}$\quad
Rui Yu \textsuperscript{2} $^{\dagger}$ \quad 
Guisong Chang \textsuperscript{3} \quad 
Ying Li \textsuperscript{1} \quad 
Yu Zhang \textsuperscript{1} \quad 
Dazhou Li \textsuperscript{1} $^{\dagger}$ \\ 
\textsuperscript{1}Shenyang University of Chemical Technology \quad
\textsuperscript{2}University of Louisville \quad \textsuperscript{3}Northeastern University \\
{\tt \small yiucp@outlook.com, rui.yu@louisville.edu, gschang@mail.neu.edu.cn} 
\\ 
{\tt \small  Gooddayli12358@outlook.com, zhangy@syuct.edu.cn,  lidazhou@syuct.edu.cn}
\\   \href{https://github.com/LilYau350/SGDF-Optimizer}{\faGithub\ https://github.com/LilYau350/SGDF-Optimizer}
}

\begin{document}
\maketitle
\blfootnote{\textsuperscript{$^{\star}$}\,Work initiated at \textsuperscript{1} and further developed at \textsuperscript{2}, \textsuperscript{$\dagger$}\, Equal advising.}
\input{texfiles/Abstract}

\input{texfiles/Introduction}

\input{texfiles/Preliminary}
\input{texfiles/Method}

\input{texfiles/Experiments}

\input{texfiles/Related_work}

\input{texfiles/Limited}

\input{texfiles/Conclusion}
\input{texfiles/Acknowledgement}
{
    \small
    \bibliographystyle{ieeenat_fullname}
    \bibliography{main}
}

\newpage
\appendix
\onecolumn

\input{texfiles/Appendix}

\end{document}

%% file: math_commands.tex
\usepackage{amsmath,amsfonts,amssymb,amsthm,bm}
\usepackage{mathtools}

\usepackage{xcolor}         
\definecolor{cvprblue}{rgb}{0.21,0.49,0.74}

\usepackage[pagebackref,breaklinks,colorlinks,linkcolor=red,citecolor=cvprblue,urlcolor=cvprblue]{hyperref}       

\usepackage[capitalize,noabbrev]{cleveref}

\theoremstyle{plain}
\newtheorem{theorem}{Theorem}[section]

\newtheorem{lemma}[theorem]{Lemma}

\theoremstyle{definition}
\newtheorem{definition}[theorem]{Definition}
\newtheorem{assumption}[theorem]{Assumption}
\theoremstyle{remark}
\newtheorem{remark}[theorem]{Remark}

\crefname{theorem}{Theorem}{theorems}
\crefname{property}{Property}{properties}
\crefname{proposition}{Proposition}{propositions}
\crefname{lemma}{Lemma}{lemmas}
\crefname{corollary}{Corollary}{corollaries}
\crefname{definition}{Definition}{definitions}
\crefname{assumption}{Assumption}{assumptions}
\crefname{remark}{Remark}{remarks}
\crefname{conjecture}{Conjecture}{conjectures}

\providecommand{\eg}{{\sl e.g.}}
\providecommand{\ie}{{\sl i.e.}}

\providecommand{\etal}{{\sl et al.}}

\def\Figref#1{Fig.~\ref{#1}}
\def\Tabref#1{Tab.~\ref{#1}}
\def\Secref#1{Sec.~\ref{#1}}
\def\Eqref#1{Eq.~(\ref{#1})}

\def\eqref#1{equation~\ref{#1}}

\def\1{\bm{1}}

\DeclareMathAlphabet{\mathsfit}{\encodingdefault}{\sfdefault}{m}{sl}
\SetMathAlphabet{\mathsfit}{bold}{\encodingdefault}{\sfdefault}{bx}{n}



%% file: texfiles/Abstract.tex
\begin{abstract}
Stochastic Gradient Descent (SGD) and its momentum variants form the backbone of deep learning optimization, yet the underlying dynamics of their gradient behavior remain insufficiently understood. In this work, we reinterpret gradient updates through the lens of signal processing and reveal that fixed momentum coefficients inherently distort the balance between bias and variance, leading to skewed or suboptimal parameter updates. To address this, we propose SGDF (SGD with Filter), an optimizer inspired by the principles of Optimal Linear Filtering. SGDF computes an online, time-varying gain to dynamically refine gradient estimation by minimizing the mean-squared error, thereby achieving an optimal trade-off between noise suppression and signal preservation. Furthermore, our approach could extend to other optimizers, showcasing its broad applicability to optimization frameworks. Extensive experiments across diverse architectures and benchmarks demonstrate SGDF surpasses conventional momentum methods and achieves performance on par with or surpassing state-of-the-art optimizers. 
\end{abstract}

%% file: texfiles/Introduction.tex
\section{Introduction}
In deep learning optimization, the optimizer plays a pivotal role in refining model parameters to capture underlying data patterns, while strategically navigating complex loss landscapes~\cite{du2018power} to identify regions that promote strong generalization~\cite{keskar2022large}. Its choice profoundly affects training efficiency, convergence speed, generalization, and robustness to data shifts~\cite{2007Scaling}, with suboptimal selections potentially causing slow or failed convergence, whereas effective ones accelerate learning and bolster model resilience~\cite{2016An}. Thus, the design and refinement of optimizers remain essential challenges in enhancing the capabilities of models.

Stochastic Gradient Descent (SGD)~\cite{1951a} and its variants, including momentum-based methods~\cite{polyak1964some,sutskever2013importance} and adaptive techniques like Adam~\cite{kingma2014adam} and RMSprop~\cite{hinton2012neural}, which have advanced training efficiency~\cite{chandramoorthy2022generalization}. However, these approaches face challenges in high-dimensional, non-convex settings~\cite{goodfellow2016deep}, where adaptive methods often yield rapid convergence but suffer from poor generalization~\cite{keskar2017improving}. Efforts to mitigate this have led to Adam variants~\cite{chen2018closing, luo2019adaptive, liu2019variance, zhuang2020adabelief} that refine adaptive learning rates, yet they fall short of fully closing the generalization gap, highlighting the ongoing need for innovative optimization strategies that better balance estimation accuracy and practical performance.

Actually, the issues that arise from the optimizer during training, particularly in terms of optimization and generalization, are inherently tied to the trade-off between bias and variance~\cite{2014Neural, nguyen2022algorithmic}. High bias leads to underfitting, while high variance results in overfitting. Similarly, the gradients used by the optimizer to update weights also face this challenge. Intuitively, high bias in the gradients may lead to convergence at a suboptimal plateau~\cite {2016Understanding,yang2023stochastic}, while high variance can lead to instability in the optimization path, causing oscillations that hinder convergence~\cite{bottou2018optimization, duchi2019variance}. Therefore, a good optimizer should strike a balance between the bias and variance in its gradient estimates.

From a statistical signal processing perspective, we analyze the mechanism behind optimizer updates. Specifically, we decompose the optimizer's gradients used for updating model weight into bias and variance components. Then, We identify a key limitation in momentum-based optimization techniques supplemented with examining the statistical distribution of gradients within the model: they struggle to balance bias and variance components in gradients, often introducing a gradient shift phenomenon, which we term \textit{bias gradient estimate}. This bias estimate, arising from fixed momentum coefficients, accumulates over time, leading to bias. As a result, the model may struggle to adapt to variations in curvature across different layers, resulting in suboptimal or directionally skewed updates~\cite{zhang2017yellowfin, dozat2016incorporating}.

To address this issue, we introduce SGDF, a novel method that uses principles from Optimal Linear Filter to adjust gradient estimation dynamically. SGDF derives an optimal, time-varying gain to minimize mean-squared error in gradient estimation, balancing noise reduction with signal preservation. This filter mechanism provides a more accurate first-order gradient estimate and avoids the limitations of fixed momentum parameters, allowing SGDF to adjust dynamically throughout training. Additionally, SGDF’s flexibility extends to other optimization frameworks, which enhance performance across a range of tasks. Through extensive empirical validation across diverse model architectures and visual tasks, we demonstrate that SGDF consistently outperforms traditional momentum-based and variance reduction methods, achieving competitive or superior results relative to state-of-the-art optimizers. 
 
The main contributions can be summarized as follows:
\vspace{-0.3em}
\begin{itemize}[leftmargin=*]
    \item We quantify the bias-variance trade-off in momentum-based gradient estimation (EMA and CM) using a unified SDE framework, revealing their static limitations.
    \item We introduce SGDF, an optimizer that combines historical and current gradient data to estimate the gradient, addressing the trade-off between bias and variance in the momentum method.
    \item We theoretically analyze the convergence property of SGDF in both convex optimization and non-convex stochastic optimization (\Secref{sec:section3.4}), and empirically verify the effectiveness of SGDF (\Secref{sec:section4}).
    \item We preliminarily explore the extension of SGDF's first-moment filter estimation to adaptive optimization algorithms (\eg, Adam), which shows a promising enhancement in their generalization capability (\Secref{sec:section4.4}), surpassing traditional momentum-based methods.
\end{itemize}

%% file: texfiles/Preliminary.tex
\section{The Gradient Estimation Dilemma}
\label{sec:dilemma}
\subsection{Bias and Variance}
Stochastic gradient-based optimization lies at the core of modern machine learning.~We revisit this and found that it grapples with a fundamental challenge: the trade-off between gradient bias and variance. To dissect this dilemma, we begin by unifying two prominent momentum strategies under a single framework. The proof of this section can be found in \cref{sec:appendixa}.

\begin{definition}
\label{def:main_unify_momentum}
The unified momentum update rule is defined as:
\begin{equation}
    m_t = \beta m_{t-1} + u g_t, \quad \theta_t = \theta_{t-1} - \alpha m_t,
\end{equation}
where $\alpha$ is the learning rate, $\beta \in [0, 1)$ is the momentum coefficient, and $u \ge 1 - \beta$ scales the current gradient. For all $\theta \in \mathbb{R}^d$, $f_t(\theta) = f(\theta; \xi_t)$ denotes the stochastic objective at iteration $t$ with data sample $\xi_t \sim \mathcal{D}$. The expected objective is $f(\theta) = \mathbb{E}_{\xi}[f(\theta; \xi)]$, and $g_t = \nabla f_t(\theta_{t}) + \epsilon_t$ (where $\epsilon_t \sim \mathcal{N}(0, \sigma^2 I)$) is the stochastic gradient. Specific cases include:
\begin{itemize}
    \item $u = 1 - \beta$: Exponential Moving Average (EMA),
    \item $u = 1$: Classical Momentum (CM)~\cite{polyak1964some,sutskever2013importance}.
\end{itemize}
\end{definition}

This formulation encapsulates EMA and CM, two cornerstones of gradient estimation, differing in how they weight the current gradient against historical trends. EMA through a balanced mean update, while CM aggressively incorporates the gradient direction. We dissect the nature of the two methods, quantified by the mean square error.

\begin{lemma}
For any gradient estimator $\hat{g}_t = \mathcal{A}(g_1,...,g_t)$, the estimation of the mean square error decomposes as:
\begin{equation}
    \mathbb{E}[(\hat{g}_t - \nabla f(\theta_t))^2] = \underbrace{(\mathbb{E}[\hat{g}_t] - \nabla f(\theta_t))^2}_{\mathrm{Bias}^2} + \underbrace{(\hat{g}_t - \mathbb{E}[\hat{g}_t])^2}_{\mathrm{Variance}}
\end{equation}
\label{lem:main_bias_variance}
\end{lemma}
\vspace{-1.5em}
\cref{lem:main_bias_variance} establishes that the error in gradient estimation arises from two sources: bias, reflecting systematic deviation from the true gradient, and variance, capturing sensitivity to stochastic fluctuations. To explore how EMA and CM navigate this trade-off, we extend prior work on stochastic differential equations (SDEs) for vanilla SGD~\cite{stephan2017stochastic}, reformulating momentum in continuous time.

\begin{theorem}
\label{main:thm:bias_variance_momentum}
Consider the unified momentum estimator $m(t)$ defined by the stochastic differential equation (SDE) from \cref{lem:unify_momentum_sde}, with solution given in \cref{lem:unify_momentum_solution}. Let the bias be defined relative to the expected true gradient: $\mathrm{Bias}(m(t)) = \mathbb{E}[m(t)] - \mathbb{E}[\nabla f(\theta(t))]$. Assuming that the gradient $\nabla f(\theta(t))$ is bounded and Lipschitz continuous, the asymptotic bounds (as $t \to \infty$) for the bias and variance of $m(t)$ as an estimator are given by:
\begin{multline}
    \left\| \mathrm{Bias}(m(t)) \right\|^2 \leq \Bigg( \frac{u^2 \alpha L G}{(1 - \beta)^3} + \frac{u^2 \alpha \sigma L }{\sqrt{2}(1 - \beta)^{2.5}} \\
    + \left( \frac{u}{1 - \beta} - 1 \right) G \Bigg)^2,
\end{multline}

where $L$ is the Lipschitz constant, $G$ bounds the gradient norm $\|\nabla f(\theta(t))\|$, and the second term inside the parenthesis explicitly captures the parameter-shift bias induced by the stochastic noise $\sigma$.

\begin{equation}
    \mathrm{Var}(m(t)) \leq \frac{u^2 \sigma^2}{1 - \beta} + \frac{2 u^2 V^2}{(1 - \beta)^2},
\end{equation}
where $\sigma^2$ is the total variance of the stochastic gradient noise, and $V^2$ conservatively bounds the variance of the true gradient sequence, \ie, $\mathrm{Var}(\nabla f(\theta(t))) \leq V^2$.

\end{theorem}

\begin{table}[ht]
\centering
\caption{Bias and variance bounds for different momentum estimators. As $\beta \to 1$, static estimators suffer from either diverging bias or diverging variance, highlighting the fundamental trade-off.}
\label{tab:bias_var_tradeoff}
\resizebox{\columnwidth}{!}{%
\begin{tabular}{lccc}
\toprule
\textbf{Method} & \textbf{Bias Bound} & \textbf{Variance Bound} & \textbf{Limit as $\beta \to 1$} \\
\midrule
\textbf{SGD} ($\beta=0$) & $0$ (Assuming $\alpha \to 0$) & $\sigma^2 + 2V^2$ & N/A \\
\addlinespace
\textbf{EMA} ($u=1-\beta$) & $\left( \frac{\alpha L G}{1 - \beta} + \frac{\alpha \sigma L }{\sqrt{2(1 - \beta)}} \right)^2$ & $(1 - \beta)\sigma^2 + 2V^2$ & Bias $\to \infty$, Var $\to 2V^2$ \\
\addlinespace
\textbf{CM} ($u=1$) & $\left( \frac{\alpha L G}{(1 - \beta)^3} + \frac{\alpha L \sigma}{\sqrt{2}(1 - \beta)^{2.5}} + \frac{\beta G}{1 - \beta} \right)^2$ & $\frac{\sigma^2}{1 - \beta} + \frac{2V^2}{(1 - \beta)^2}$ & Bias $\to \infty$, Var $\to \infty$ \\
\bottomrule
\end{tabular}%
}
\vspace{-2em}
\end{table}

Prior analyses typically regarded momentum-based gradient estimators as unbiased under the assumption that the parameter $\theta_t$ remains stationary~\cite{sutskever2013importance, kingma2014adam, cutkosky2020momentum}. However, as $\theta_t$ evolves over training, an additional \emph{parameter-shift bias} arises. \cref{main:thm:bias_variance_momentum} explicitly quantifies this effect to fill a critical gap left by prior analyses, revealing that the stochastic noise $\sigma$ itself exacerbates this shift.

To illustrate the implications of this theorem, \cref{tab:bias_var_tradeoff} summarizes the bias and variance bounds for standard momentum formulations and their asymptotic behaviors. For EMA ($u = 1 - \beta$), the initialization bias strictly vanishes. EMA effectively acts as a low-pass filter: as $\beta \to 1$, it successfully damps high-frequency stochastic noise (variance drops), but at the severe expense of unbounded bias driven by outdated gradients and accumulated noise-drift. Consequently, EMA is unbiased only under strict, often unrealistic conditions in deep learning, such as the learning rate or curvature approaching zero ($\alpha, L \to 0$). Conversely, CM ($u = 1$) exhibits an aggressive momentum nature. As shown in \cref{tab:bias_var_tradeoff}, both its bias and variance diverge sharply as $\beta \to 1$, amplifying systematic lag and noise susceptibility (the complex mechanics of CM are further discussed in \cref{sec:appendix_discuss}).

Consider the other extreme: when $\beta = 0$, both methods reduce to vanilla SGD, minimizing momentum-induced bias but retaining the full base variance. These bounds expose a fundamental limitation of conventional optimization paradigms: static choices of $u$ and $\beta$ lock the estimator into a rigid, predetermined trade-off, rendering it ill-suited to the dynamic noise and curvature of objective functions. 

This analysis reveals an inherent dilemma in momentum methods: structurally reducing variance inevitably amplifies bias, while minimizing bias exposes the estimator to higher variance. This naturally raises a fundamental question: Can we design an adaptive gain that dynamically reduces the dependence on momentum to minimize bias during low-variance phases, while heavily utilizing the momentum update to aggressively filter noise when variance is high?

%% file: texfiles/Method.tex
\section{Method}
In the \cref{sec:dilemma}, we analyzed a challenge in momentum-based methods: how to effectively balance the bias and variance in gradient estimation. The analysis suggests that introducing a variable gain could help mitigate this dilemma by adaptively adjusting the contribution of past and current gradients. Motivated by this, and inspired by the minimum mean square error (MMSE) principle~\cite{kay1993fundamentals} in Optimal Linear Filter~\cite{Kalman1960}, we design a novel online gradient estimator tailored for stochastic optimization. In the following, we describe the estimator's design and theoretical grounding.

\begin{algorithm}[htbp]
\caption{SGDF: Online Filter Estimate Gradient (element-wise).}
\label{alg:SGDF}
\textbf{Input}: $\theta_0$: initial parameter, $f(\theta)$: stochastic objective function\\
\textbf{Parameter}: $\{\alpha_t\}_{t=1}^{T}$: step size, $\{\beta_1, \beta_2\}$: attenuation coefficients, $\{\varepsilon\}$: numerical stability,$\{\textcolor{cvprblue}{\gamma}\}$: power scaling.\\
\textbf{Output}: $\theta_{T}$: resulting parameters.
\begin{algorithmic}[1]
\State Initialize: $m_0 \leftarrow 0$, $s_0 \leftarrow 0$
\While{$t \leq T$}  
    \State $g_{t} \leftarrow \nabla f_t(\theta_{t-1})$
    \State $m_{t} \leftarrow \beta_1 m_{t-1} + (1 - \beta_1) g_t$
    \State $s_{t} \leftarrow \beta_2 s_{t-1} + (1 - \beta_2) (g_t - m_t)^2$
    \State $\widehat{m}_t \leftarrow \dfrac{m_t}{1 - \beta_1^t}$, 
           $\widehat{s}_t \leftarrow 
           \dfrac{{\color{cvprblue}(1-\beta_1)(1-\beta_1^{2t})}\, s_t}
           {{\color{cvprblue}(1+\beta_1)}(1 - \beta_2^t)}$
    \State $K_{t} \leftarrow \dfrac{\widehat{s}_{t}}{\widehat{s}_t + (g_t - \widehat{m}_t)^2 +\varepsilon} $
    \State $\widehat{g}_t \leftarrow \widehat{m}_{t} + K_t^{\begingroup\color{cvprblue}\gamma\endgroup} (g_t - \widehat{m}_{t})$
    \State $\theta_{t} \leftarrow \theta_{t-1} - \alpha_t \widehat{g}_t$
\EndWhile
\State \textbf{return} $\theta_{T}$
\end{algorithmic}
\end{algorithm}


\subsection{SGDF General Introduction}
\label{sec:method intro}
In this subsection, we introduce the SGDF by building on principles of optimal gradient estimation. The detailed derivation of the SGDF is provided in \cref{proof_method}. The complete algorithm is summarized in \cref{alg:SGDF}. Below, we summarize its key steps and motivation.

In stochastic gradient descent (SGD), the core challenge is to estimate a reliable gradient direction under noise. Suppose we are given a sequence of stochastic gradients $\{g_i\}_{i=1}^{t}$, our goal is to estimate a smoothed direction $\hat{g}_t$ that effectively combines the current observation $g_t$ and past gradients. 

Inspired by the principle of minimum mean squared error (MMSE), we begin with a naive average:
\begin{equation}
\hat{g}_t = \frac{1}{t} \sum_{i=1}^{t-1} g_i + \frac{1}{t} g_t = \left(1 - \frac{1}{t}\right)\bar{g}_{1:t-1} + \frac{1}{t}g_t,
\end{equation}
where $\bar{g}_{1:t - 1} = \frac{1}{t - 1} \sum_{i=1}^{t - 1} g_i$ denotes the averaging of the gradient under different model parameters. Then, we rewrite this as a linear interpolation:
\begin{equation}
\hat{g}_t = (1 - K_t)\bar{g}_{1:t-1} + K_t g_t, \quad \text{where} \quad K_t = \frac{1}{t}.
\end{equation}

To enable efficient computation, we replace $\bar{g}_{1:t-1}$ with a bias-corrected momentum estimate $\widehat{m}_t$, giving:
\begin{equation}
\label{interpolation formula}
\hat{g}_t = \widehat{m}_t + K_t (g_t - \widehat{m}_t).
\end{equation}

This form mirrors the update structure of an Optimal Linear Filter, which recursively refines estimates by weighting prediction and observation with gain derived from their respective uncertainties. Here, $K_t$ acts as a gain to balance trust between the prior $\widehat{m}_t$ and the observation $g_t$.

To find an optimal gain, we minimize the variance of $\hat{g}_t$, assuming $\widehat{m}_t$ and $g_t$ are independent:
\begin{equation}
\mathrm{Var}(\hat{g}_t) = (1 - K_t)^2 \mathrm{Var}(\widehat{m}_t) + K_t^2 \mathrm{Var}(g_t).
\end{equation}
Solving $\frac{\mathrm{d}}{\mathrm{d}K_t} \mathrm{Var}(\hat{g}_t) = 0$ yields the optimal gain:
\begin{equation}
K_t^\star = \frac{\mathrm{Var}(\widehat{m}_t)}{\mathrm{Var}(\widehat{m}_t) + \mathrm{Var}(g_t)}.
\end{equation}

This form naturally down-weights noisy gradients and shifts the estimate toward the more reliable direction. 

To estimate $\mathrm{Var}(\widehat{m}_t)$ in practice, we follow~\cite{zhuang2020adabelief} and compute the second-order moment $s_t$ as the exponential moving average of $(g_t - m_t)^2$, applying Adam’s bias correction~\cite{kingma2014adam} to obtain $\widehat{m}_t$ and $\widehat{s}_t$. To further refine this estimate, we introduce a variance correction factor \textcolor{cvprblue}{$(1-\beta_1)(1-\beta_1^{2t})/(1+\beta_1)$}, derived in Appendix~\ref{proof_correction}, which improves $\widehat{s}_t$ under independent gradients with bounded variance. The \Figref{fig:correction} in \cref{app:subsec:ablation} compares the performance with and without this correction factor, showing that this correction improves performance and supports our theoretical framework. Finally, to improve responsiveness in noisy regimes, we scale $K_t$ by $\textcolor{cvprblue}{\gamma=\frac{1}{2}}$, an operation we formally justify in \cref{proof_scaling} as mathematically equivalent to modulating the effective observation variance.

\subsection{Fusion of Gaussian Distributions}
\label{sec:section3.2}
Building on the MMSE-based principle in \cref{sec:method intro}, we now seek a deeper statistical view of the SGDF. Specifically, we interpret the gain-controlled interpolation between the momentum estimate $\widehat{m}_t$ and the current gradient $g_t$ as the fusion of two uncertain sources. This subsection presents both a variance-weighted linear view and a Bayesian interpretation, showing that SGDF performs optimal Gaussian fusion under the assumption of independent noise.

\textbf{Linear combination view.}
Starting from the interpolation formula \Eqref{interpolation formula}, we write the following:
\begin{equation}
    \hat{g}_t = (1 - K_t)\widehat{m}_t + K_t g_t .
\end{equation}

We assume that the stochastic gradients can be modeled as Gaussian distributions~\cite{bernstein2018signsgd,liu2019variance}, \ie, $\widehat{m}_t \sim \mathcal{N}(\mu_m,\sigma_m^2)$ and $g_t \sim \mathcal{N}(\mu_g,\sigma_g^2)$, with independence between the two sources.

Under these assumptions, the fused mean and variance are
\begin{align}
\mathbb{E}[\hat{g}_t] &= (1-K_t)\mu_m + K_t\mu_g  = \frac{\sigma_g^2\mu_m + \sigma_m^2\mu_g}{\sigma_m^2+\sigma_g^2},\\
\mathrm{Var}(\hat{g}_t) &= (1-K_t)^2\sigma_m^2 + K_t^2\sigma_g^2 = \frac{\sigma_m^2\sigma_g^2}{\sigma_m^2+\sigma_g^2}.
\end{align}

\textbf{Bayesian fusion view.}
An equivalent result is obtained by multiplying the two Gaussian densities and normalising~\cite{faragher2012understanding}:
\begin{equation}
    p(\hat{g}_t) \propto
    \exp\!\Biggl[-\frac{(\hat{g}_t-\mu_m)^2}{2\sigma_m^2}
               -\frac{(\hat{g}_t-\mu_g)^2}{2\sigma_g^2}\Biggr],
\end{equation}
which yields the posterior
\begin{equation}
\widehat{g}_t \sim \mathcal{N}  \Biggl(
  \dfrac{\sigma_g^2\mu_m+\sigma_m^2\mu_g}{\sigma_m^2+\sigma_g^2},
  \dfrac{\sigma_m^2\sigma_g^2}{\sigma_m^2+\sigma_g^2}\Biggr).
\end{equation}

The fused mean \(\mu_{\hat{g}_t}\) is a variance-weighted average of \(\mu_{m}\) and \(\mu_{g}\), assigning greater weight to the source with lower variance to reflect higher confidence in more stable estimates. Similarly, the fused variance \(\sigma_{\hat{g}_t}^2\) is smaller than both \(\sigma_{m}^2\) and \(\sigma_{g}^2\), indicating reduced uncertainty in the gradient estimate. This reduction results from the Optimal Linear Filter's optimality in minimizing the mean squared error. The full derivation is provided in \cref{proof_fusion}.

\input{resources/main_figure_cifar}

\subsection{Convex and Non-convex Convergence Analysis}
\label{sec:section3.4}
Finally, we provide the convergence property of SGDF as shown in \cref{thm:main_convex} and \cref{thm:main_non_convex}. The assumptions are common and standard when analyzing the convergence of convex and non-convex functions via SGD-based methods~\cite{kingma2014adam,chen2018convergence}. Proofs for convergence in convex and non-convex cases are provided in \cref{convex} and \cref{non-convex}, respectively. 

\begin{theorem}[Convergence in Convex Optimization]
\label{thm:main_convex}
Assume that the objective functions $f_t$ are convex. Let the gradients be bounded such that $\|\nabla f_t\|_\infty \leq G_\infty$, and the optimization domain be bounded with $\|\theta_m - \theta_n\|_\infty \leq D_\infty$. Suppose the momentum coefficient $\beta_1, \beta_2 \in [0, 1)$ is constant, the power scaling factor follows $\gamma_t = \gamma_0 / \sqrt{t}$ for some $\gamma_0 > 0$. For a learning rate $\alpha_t = \alpha/\sqrt{t}$, SGDF achieves the following cumulative regret bound for all $T \geq 1$:
\vspace{-1em}
\begin{multline}
    R(T) \leq \sum_{i=1}^{d} \left( \frac{D_\infty^2}{2\alpha} + \alpha G_\infty^2 \frac{1+\beta_1}{1-\beta_1} \right) \sqrt{T} \\
    + \sum_{i=1}^{d} \frac{\beta_1 G_\infty D_\infty}{1-\beta_1} \left( 2 + \sum_{t=1}^{T-1} \vert K_{t,i} - K_{t+1,i} \vert \right)
\end{multline}

\end{theorem}

In Adam-type optimizers, decaying $\beta_{1,t}$ to zero is crucial for convex analysis~\cite{kingma2014adam,zhuang2020adabelief}. By contrast, SGDF maintains constant coefficients ($\beta_1, \beta_2$) and delegates the stabilization role to the dynamic power-scaled gain $K_t^{\gamma_t}$. With $\gamma_t = \mathcal{O}(1/\sqrt{t})$ that \(\sum \vert K_{t,i} - K_{t+1,i} \vert \leq \mathcal{O}(1) + \mathcal{O}(\sqrt{T})\). Therefore, in the convex case, \cref{thm:main_convex} establishes that the regret of SGDF is upper bounded by $\mathcal{O}(\sqrt{T})$. 
 
\begin{theorem}{(Convergence for Non-convex Stochastic Optimization)}
\label{thm:main_non_convex}
Assume Assumptions 1--4 hold and the step size is $\alpha_t = \alpha/\sqrt{t}$. For all $T \geq 1$, SGDF satisfies:
\begin{enumerate}[leftmargin=*]
    \item \textbf{Bounded Variables:} $\|\theta-\theta^{*}\|_{2} \leq D$ (or $\|\theta_{i}-\theta_{i}^{*}\|_{2} \leq D_{i}$ for each dimension $i$) for all $\theta, \theta^{*}$.
    \item \textbf{Unbiased Noise:} The noise $\zeta_{t} = g_{t} - \nabla f(\theta_{t})$ satisfies $\mathbb{E}[\zeta_{t}] = 0$, $\mathbb{E}[\|\zeta_{t}\|_{2}^{2}] \leq \sigma^{2}$, and $\zeta_{t_1}, \zeta_{t_2}$ are independent for $t_1 \neq t_2$.
    \item \textbf{Bounded Gradients:} Both the true and noisy gradients are uniformly bounded, i.e., $\|\nabla f(\theta_t)\|_2 \leq G$ and $\|g_t\|_2 \leq G$ for all $t \geq 1$.
    \item \textbf{Function Properties:} $f$ is differentiable, lower bounded, and $L$-smooth, i.e., $\|\nabla f(x) - \nabla f(y)\|_2 \leq L\|x-y\|_2$ for all $x, y$.
\end{enumerate}

The convergence guarantee is given by:
\begin{equation}
    \mathbb{E}(T) \leq \frac{C_{7}\alpha^2 (\log T + 1) + C_{8}}{\alpha\sqrt{T}},
\end{equation}
where $\mathbb{E}(T) = \min_{t=1,\ldots,T} \mathbb{E}_{t-1}[\|\nabla f(\theta_{t})\|_{2}^{2}]$ is the minimum expected squared gradient norm, $\alpha$ is the initial learning rate, and $C_7, C_8$ are constants independent of $T$ ($C_7$ also independent of $d$). The expectation is taken over all randomness in $\{g_t\}$.
\end{theorem}

\cref{thm:main_non_convex} shows that the convergence rate of SGDF in the non-convex setting is $\mathcal{O}(\log T / \sqrt{T})$, which matches the rates established for Adam-type optimizers~\cite{reddi2018convergence,chen2018convergence}. In our analysis, the terms involving the estimated gain $K_t$ were upper bounded by their maximal possible values to simplify the derivation of the final bound. We adopted the general $L$-smoothness assumption to obtain this rate. Furthermore, if a fixed decay schedule $\alpha/\sqrt{T}$ is used, where $T$ represents the total number of iterations, instead of $\alpha/\sqrt{t}$ for infinite iterations, the convergence rate improves to $\mathcal{O}(1/\sqrt{T})$~\cite{wang2023closing}.

%% file: resources/main_figure_cifar.tex
\begin{figure*}[t]
    \vspace{0.2em}
    \centering
    \setlength{\tabcolsep}{0pt}
    \begin{tabular}{ccc}
        \subfloat[VGG11 on CIFAR-10]{\includegraphics[width=0.33\linewidth]{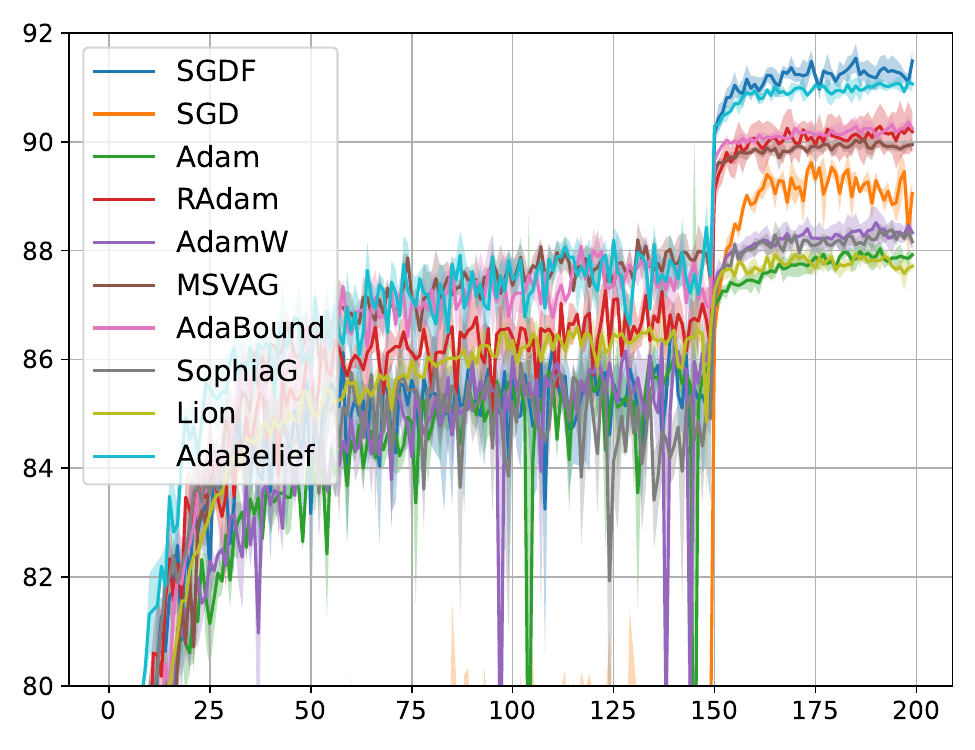}\label{subfig:main_vgg_test_cifar10}} &
        \subfloat[ResNet34 on CIFAR-10]{\includegraphics[width=0.33\linewidth]{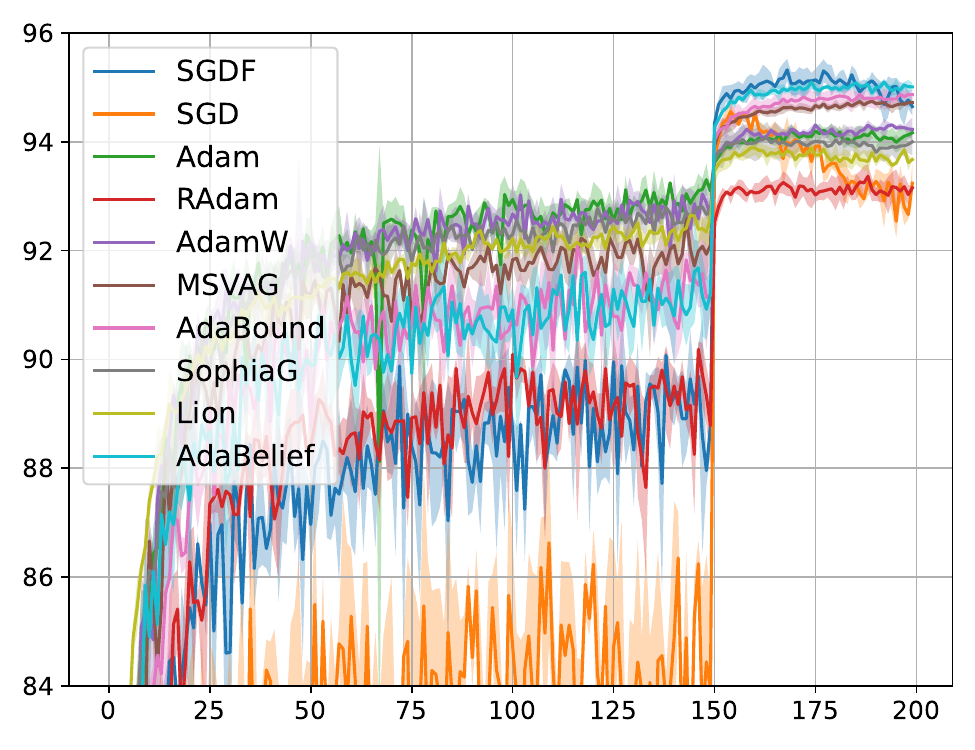}\label{subfig:main_resnet_test_cifar10}} &
        \subfloat[DenseNet121 on CIFAR-10]{\includegraphics[width=0.33\linewidth]{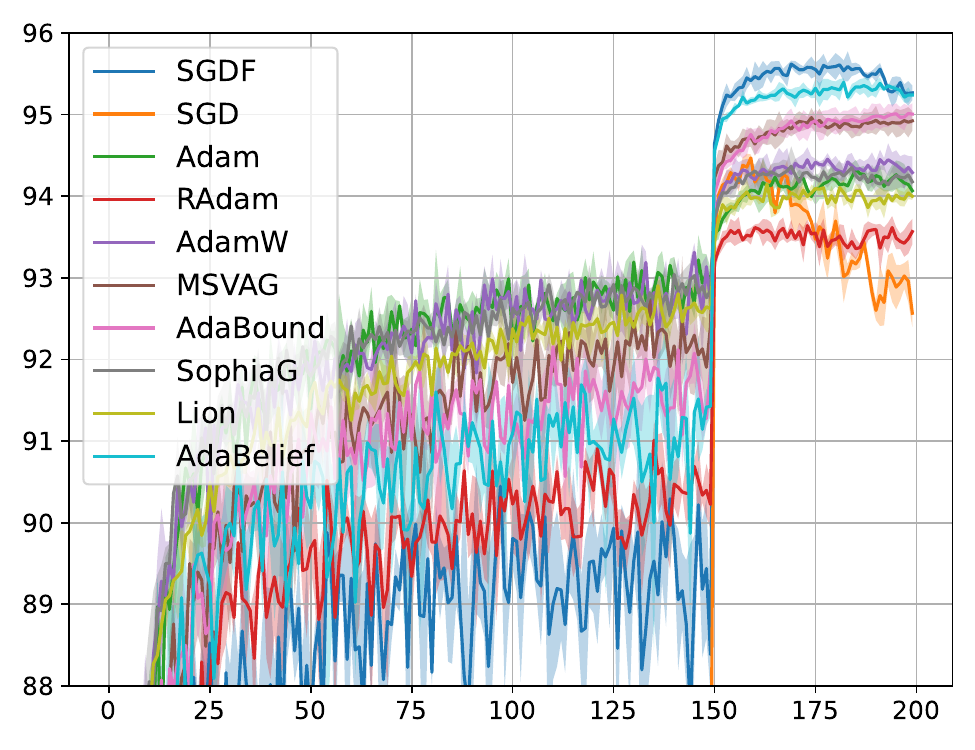}\label{subfig:main_densenet_test_cifar10}} \\
        
        \subfloat[VGG11 on CIFAR-100]{\includegraphics[width=0.33\linewidth]{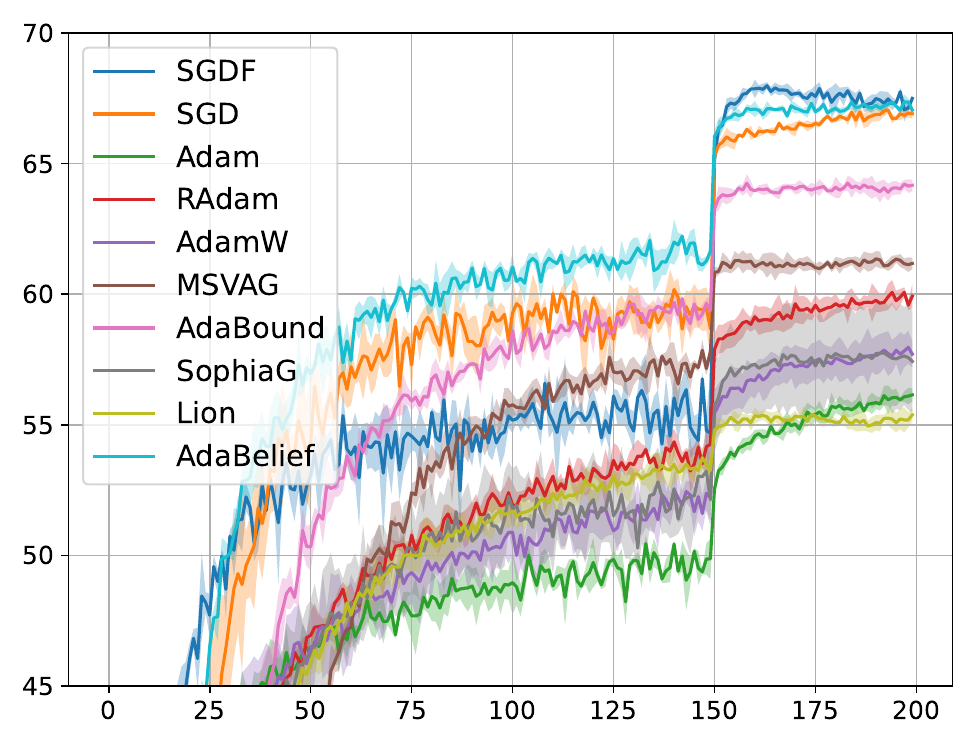}\label{subfig:main_vgg_test_cifar100}} &
        \subfloat[ResNet34 on CIFAR-100]{\includegraphics[width=0.33\linewidth]{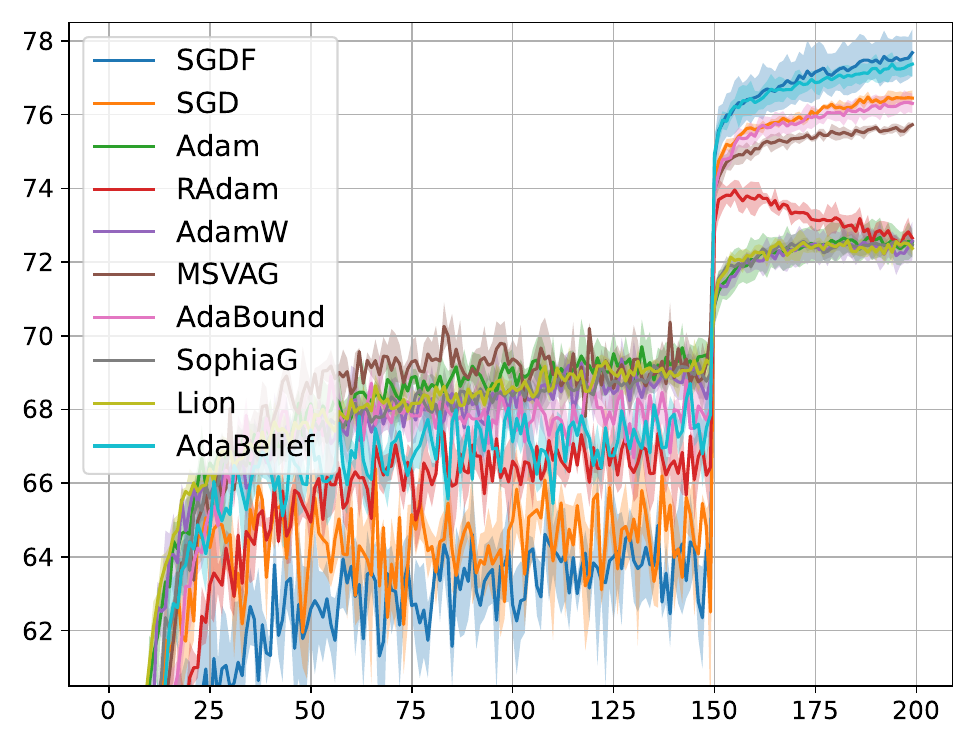}\label{subfig:main_resnet_test_cifar100}} &
        \subfloat[DenseNet121 on CIFAR-100]{\includegraphics[width=0.33\linewidth]{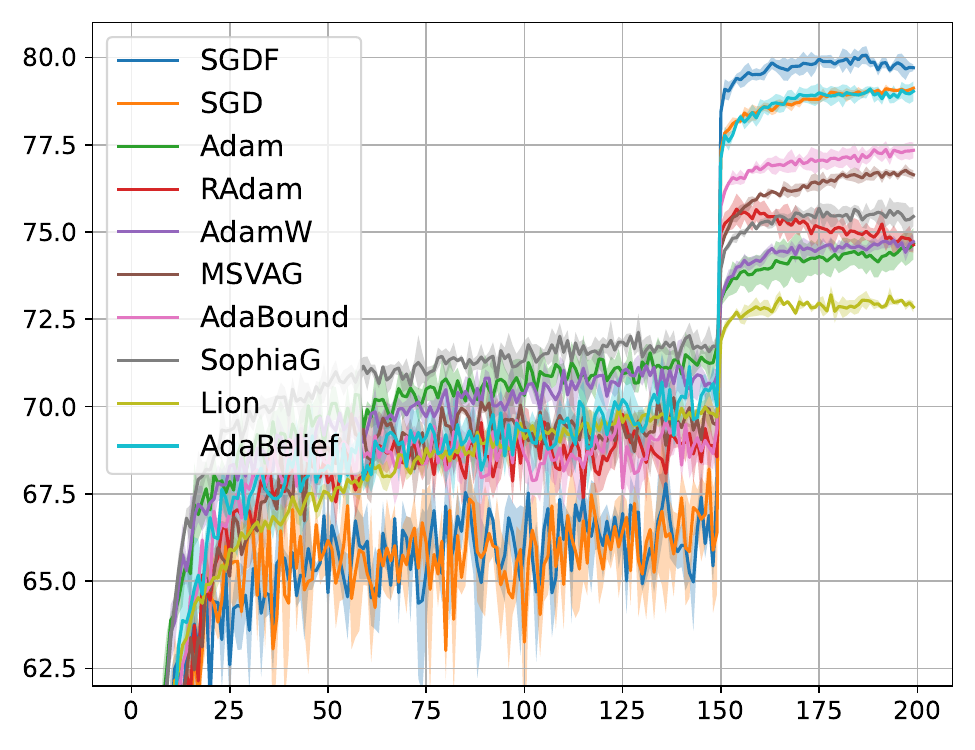}\label{subfig:main_densenet_test_cifar100}}
    \end{tabular}
    \caption{Test accuracy ([$\mu \pm \sigma$]) on CIFAR.}
    \label{fig:main_cifar_experiments}
    \vspace{0.2em}
\end{figure*}

%% file: texfiles/Experiments.tex
\section{Experiments}
\label{sec:section4}
\subsection{Empirical Evaluation}
\label{sec:section4.1}
In this study, we focus on the following tasks: \textbf{Image Classification.} We employed the VGG~\cite{2014Very}, ResNet~\cite{he2016deep}, and DenseNet~\cite{huang2017densely} models for image classification tasks on the CIFAR dataset~\cite{krizhevsky2009learning}. The major difference between these three network architectures is the residual connectivity, which we will discuss in \Secref{sec:section4.4}. We evaluated and compared the performance of SGDF with other optimizers such as SGD, Adam, RAdam~\cite{liu2019variance}, AdamW~\cite{loshchilov2017decoupled}, MSVAG~\cite{balles2018dissecting}, Adabound~\cite{luo2019adaptive}, Sophia~\cite{liu2023sophia}, Lion~\cite{chen2023symbolic}, and AdaBelief~\cite{zhuang2020adabelief}, all of which were implemented based on the official PyTorch. Additionally, we further tested the performance of SGDF on the ImageNet dataset~\cite{Deng2009} using the ResNet model. \textbf{Object Detection.} Object detection was performed on the PASCAL VOC dataset~\cite{pascal-voc-2010} using Faster-RCNN~\cite{faster_rcnn} integrated with FPN. \textbf{Post-training in ViT.} We test the performance of transformer architecture networks by post-training ViT~\cite{dosovitskiy2020image} on six benchmark dataset. \textbf{More experimental results are summarized in \cref{sec:appendixe}.} 

\textbf{Hyperparameter tuning.} Following \cite{zhuang2020adabelief}, we delved deep into the optimal hyperparameter settings for our experiments. 
\begin{itemize}[leftmargin=*]
\item \textit{SGDF:} We followed to Adam's original parameter values except learning rate: $\alpha=0.5$, $\beta_1=0.9$, $\beta_2=0.999$, $\epsilon=10^{-8}$. The learning rate was searched same as SGD research set.
\item \textit{SGD:} We set the momentum 0.9, which is the default for networks like ResNet and DenseNet. The learning rate was searched in the set \{10.0, 1.0, 0.5, 0.1, 0.01, 0.001\}.
\item \textit{Adam, RAdam, MSVAG, AdaBound, AdaBelief:} Exploring the hyperparameter space, we tested \(\beta_{1} \) values in \(\{0.5,0.6,0.7,0.8,0.9\}\), examined \(\alpha \) as in SGD, while keeping other parameters to their literary defaults.
\item \textit{AdamW, SophiaG, Lion:} Following Adam's parameter search pattern, we fixed weight decay at \(5 \times 10^{-4}\); yet for AdamW, whose optimal decay often exceeds norms~\citep{loshchilov2017decoupled}, we ranged weight decay over \(\left\{10^{-4}, 5 \times 10^{-4}, 10^{-3}, 10^{-2}, 10^{-1}\right\}\).
\item \textit{SophiaG, Lion:} We searched for the learning rate among $\{10^{-3}, 10^{-4}, 10^{-5}\}$ and adjusted Lion's learning rate~\cite{liu2023sophia}. Following ~\cite{liu2023sophia, chen2023symbolic}, we set $\beta_{1}$=0.965, 0.9 and default parameters $\beta_2$=0.99.
\end{itemize}

\input{resources/main_table_imagenet}

\textbf{CIFAR-10/100 Experiments.} We trained on the CIFAR-10 and CIFAR-100 datasets using the VGG, ResNet and DenseNet models and access the performance of the SGDF optimizer. In these experiments, we employ basic data augmentation techniques such as random horizontal flip and random cropping. The results represent the mean and standard deviation of 3 runs by fixing random seeds \{0, 1, 2\}. 

As \Figref{fig:main_cifar_experiments} shows, it is evident that the SGDF optimizer exhibited convergence speeds comparable to adaptive optimization algorithms. Additionally, SGDF's final test set accuracy was better than others. These consistent results across multiple architectures indicate that SGDF can effectively adapt to networks of varying depths and complexities. We summarize the numerical results for the mean best test accuracies, standard deviations, and parameter details in \cref{app:subsec:cifar} (\Tabref{tab:cifarhyperparameters} and \Tabref{tab:cifar_values}).


\textbf{ImageNet Experiments.} We applied standard data augmentation techniques, including random cropping and random horizontal flipping~\cite{zhuang2020adabelief}, using random seeds \{0, 1, 2\}. To eliminate the effect of stepwise learning rate schedules, we adopted cosine learning rate decay as suggested by~\cite{chen2023symbolic,Zhang2023}. Following~\cite{chen2018closing,zhuang2020adabelief}, ResNet18 was trained for 100 epochs to compare with popular optimizers using their best-reported results~\cite{chen2018closing,liu2019variance,zhuang2020adabelief}.

We further trained additional architectures for 90 epochs following the experimental setup of GAM~\cite{Zhang2023} to benchmark against SGD, including VGG11/13 (with BN), ResNet34/50, and DenseNet121/161. Our results in \Tabref{tab:resnet18_imagenet} and \Tabref{tab:more_imagenet} show that SGDF consistently outperforms SGD. These results highlight SGDF’s robustness and scalability across networks of different capacities. Moreover, the optimizer maintained smooth convergence behavior throughout training, demonstrating strong generalization on large-scale datasets. Detailed training/test curves and hyperparameter settings are provided in \cref{app:subsec:imagenet} (\Figref{fig:imagenet_curve} and \Tabref{tab:imagenethyperparameters}).

\input{resources/main_table_detection}
\textbf{Object Detection.} We conducted object detection experiments on the PASCAL VOC dataset~\citep{pascal-voc-2010}. The model used in these experiments was pre-trained on the COCO dataset~\citep{coco-dataset}, obtained from the official website. We trained this model on the VOC2007 and VOC2012 train-val dataset (17K) and evaluated it on the VOC2007 test dataset (5K). The utilized model was Faster-RCNN~\cite{faster_rcnn} with FPN, and the backbone was ResNet50~\citep{he2016deep}. Results are summarized in ~\Tabref{tab:object_detection}. To facilitate result reproduction, we provide the parameter table for this subpart in \cref{app:subsec:detection} \Tabref{tab:hyperparameters_object_detection}. As expected, SGDF outperforms other methods in detection accuracy and stability. These results demonstrate the efficiency and robustness of our method in complex detection tasks, highlighting its consistent optimization behavior across vision architectures.
%

\textbf{Post-training in ViT.} To evaluate SGDF, we conducted ViT post-training by following the standard protocol in~\cite{dosovitskiy2020image}, which serves as a critical benchmark where SGD with momentum sets the state-of-the-art. We tested on six datasets (CIFAR-10/100, Oxford-IIIT-Pets~\cite{parkhi2012cats}, Oxford Flowers-102~\cite{nilsback2008automated}, Food101~\cite{bossard2014food}, ImageNet-1K) using ViT-B/32 and ViT-L/32 that were pre-trained on ImageNet-21K. Consistent with ViT's original setup, the MLP classification head was replaced to match dataset categories, while backbone weights remained frozen to preserve pre-trained representations. We upsized images to $384 \times 384$ with 2D-interpolated position encodings and used cosine learning rate decay, no weight decay, batch size 512, and gradient clipping (norm 1), all of which faithfully replicate ViT's training details. All runs were trained for 10 epochs with seed \{0, 1, 2\}. Results in \Tabref{tab:vitresult} and hyperparameters to \cref{app:subsec:post-training} \Tabref{tab:vithyperparameters} validate SGDF's performance against the established SGD.

\input{resources/main_table_vit}

\subsection{Extensibility of Filter-Estimated Gradients}
\label{sec:section4.4}
Beyond vanilla SGD, our optimal linear filter serves as a plug-and-play module that enhances first-moment gradient estimates across diverse optimization frameworks. To demonstrate its broad compatibility, we integrate it with Adam, Sign-based optimizers~\cite{bernstein2018signsgd}, and the Muon~\cite{jordan2024muon}.

\input{resources/main_table_adam_filter}

\textbf{Integration with Adam.} We first evaluated our filter's integration with Adam by substituting the first-moment gradient estimates in the vanilla optimizer with our filtered counterparts.  As shown in \Tabref{tab:accuracy_comparison} (with detailed test curves in \cref{app:subsec:extensibility} \Figref{fig:compare}), experiments on the CIFAR-100 dataset reveal distinct behaviors across network architectures. For networks without Batch Normalization (BN) like VGG, the filter significantly improves performance by providing more accurate gradient estimates and reducing noise-induced errors. For architectures like ResNet and DenseNet that inherently promote stable gradient updates through BN and residual/dense connections, the filter still maintains highly competitive performance, albeit with less pronounced marginal gains. This structural dynamic effectively explains the variance in improvements across architectures.

\textbf{Extension to Sign-based Optimizers.} To further evaluate the robustness of our gradient estimation in different update paradigms, we tested a sign-based variant. Specifically, we updated the parameters using the sign of our filtered gradient ($\text{sign}(\widehat{g}_t)$)~\cite{bernstein2018signsgd}. We evaluated this \textit{Sign SGDF} on diffusion models, specifically DiT/SiT-Base (130M parameters)~\cite{peebles2023scalable,ma2024sit} in ImageNet. For a fair comparison, we aligned all hyperparameters with Adam and used an unscaled $K_t$. As shown in \Figref{fig:diffusion_fid}, Sign SGDF converges significantly faster and achieves a better FID score than the Adam baseline, proving that our filter effectively captures the true gradient direction even when the magnitude is discarded.

\input{resources/main_figure_sign_diffusion} 

\textbf{Compatibility with Muon.} Finally, we explored the compatibility of our filter with Muon, a recently proposed optimizer that relies on orthogonal momentum. We conducted a preliminary experiment on DiT-B/4 by replacing Muon's standard momentum formulation with our filter-estimated gradient. As presented in \Tabref{tab:sgdf_muon}, this integration yields superior results compared to the standard Adam baseline across multiple generation metrics at 400K training iterations. This underscores the potential of our optimal linear filter to serve as a fundamental gradient refinement step for state-of-the-art training recipes.

\input{resources/main_table_muon}

\subsection{Top Eigenvalues of Hessian and Hessian Trace}
The success of optimization algorithms in deep learning depends on both minimizing training loss and the quality of the solutions they find. So we numerically verified the hessian matrix properties between the different methods. We computed the Hessian spectrum of ResNet-18 trained on the CIFAR-100 dataset for 200 epochs. These experiments ensure that all methods achieve similar results on the training set. We employed power iteration~\cite{2018Hessian} to compute the top eigenvalues of Hessian and Hutchinson’s method~\cite{2020PyHessian} to compute the Hessian trace. Histograms illustrating the distribution of the top 50 Hessian eigenvalues for each optimization method are presented in \Figref{fig:main_hessian_spectrum}. SGDF brings lower eigenvalue and trace of the hessian matrix, which explains the fact that SGDF demonstrates better performance than SGD as the categorization category increases.

\input{resources/main_figure_hessian}




%% file: resources/main_table_imagenet.tex
\begin{table*}[t]
       \vspace{-1mm}
	\centering
	\caption{Top-1, 5 accuracy (\%) of ResNet18 on ImageNet. $^{*}$ \ $^{\dagger}$ \ $^{\ddagger}$ is reported in~\cite{zhuang2020adabelief, chen2018closing, liu2019variance}.}
    \resizebox{1.0\linewidth}{!}{
	\begin{tabular}{c|cccccccccc}
		\toprule 
		Method & SGDF & SGD & AdaBelief & PAdam &AdaBound & Yogi & MSVAG & Adam & RAdam & AdamW \\
		\midrule 
		Top-1 & $\textbf{70.51}_{\pm 0.05}$ & 70.23$^{\dagger}$ & 70.08$^{*}$& 70.07$^{\dagger}$ & 68.13$^{\dagger}$ & 68.23$^{\dagger}$ & 65.99$^{*}$ & 63.79$^{\dagger}$ (66.54$^{\ddagger}$) & 67.62$^{\ddagger}$ & 67.93$^{\dagger}$ \\
		Top-5 & $\textbf{89.69}_{\pm 0.16}$ & 89.40$^{\dagger}$& -& 89.47$^{\dagger}$ &88.55$^{\dagger}$& 88.59$^{\dagger}$ & - & 85.61$^{\dagger}$& - & 88.47$^{\dagger}$ \\
		\bottomrule 
	\end{tabular}
    }
	\label{tab:resnet18_imagenet}
       \vspace{-0mm}
\end{table*}

\begin{table*}[htbp]
    \vspace{-0mm}
	\centering
	\caption{Comparison of top-1 accuracy (\%) across different model variants and optimizers on the ImageNet classification.}
    \resizebox{0.80\linewidth}{!}{
	\begin{tabular}{c|cccccc}
		\toprule
		Model & VGG11 & VGG13 & ResNet34 & ResNet50 & DenseNet121 & DenseNet161 \\
		\midrule
		SGD\tablefootnote{Results from \href{https://pytorch.org/vision/main/models.html}{PyTorch official pre-trained models}.}  & 70.37 & 71.58 & 73.31 & 76.13 & 74.43 & 77.13 \\		
		\midrule
        SGDF & $\textbf{71.34}_{\pm 0.21}$ & $\textbf{72.74}_{\pm 0.07}$ & $\textbf{74.07}_{\pm 0.21}$ & $\textbf{76.72}_{\pm 0.09}$ & $\textbf{75.75}_{\pm 0.09}$ & $\textbf{78.34}_{\pm 0.08}$\\
		\bottomrule
	\end{tabular}
    }
	\label{tab:more_imagenet}
	\vspace{-1mm}
\end{table*}


%% file: resources/main_table_detection.tex
\begin{table}[htbp]
\vspace{-0.5mm}
\centering
\caption{The mAP on PASCAL VOC. $^{*}$\,$^{\dagger}$ is reported in~\cite{zhuang2020adabelief, yuan2020eadam}.}
\begin{adjustbox}{max width=1.0\linewidth}
\begin{tabular}{c|ccccccc}
    \toprule 
    Method & SGDF & AdaBelief & EAdam & SGD & Adam & AdamW & RAdam \\
    \midrule 
    mAP & \textbf{83.81} & 81.02$^{*}$ & 80.62$^{\dagger}$ & 80.43 & 78.67 & 78.48 & 75.21 \\
    \bottomrule 
\end{tabular}
\label{tab:object_detection}
\end{adjustbox}
\vspace{-0.5mm}
\end{table}

%% file: resources/main_table_vit.tex
\begin{table*}[ht]
\vspace{-1em}
\centering
\caption{Post-training in ViT. We report Top-1 accuracy (\%).}
\resizebox{1.0\linewidth}{!}{
\begin{tabular}{@{}lccccccc@{}}
\toprule
Model & Method & CIFAR-10 & CIFAR-100 & Oxford-IIIT-Pets & Oxford Flowers-102 & Food101 & ImageNet \\
\midrule
\multirow{2}{*}{ViT-B/32}  
& SGD  & $98.71_{\pm 0.03}$ & $90.62_{\pm 0.07}$  & $89.71_{\pm 0.32}$ & $96.79_{\pm 0.29}$ & $88.56_{\pm 0.05}$ & $81.42_{\pm 0.04}$ \\
& SGDF & $\textbf{98.74}_{\pm 0.10}$ & $\textbf{91.44}_{\pm 0.13}$   & $\textbf{92.68}_{\pm 0.04}$ & $\textbf{97.17}_{\pm 0.47}$ & $\textbf{89.35}_{\pm 0.09}$ & $\textbf{81.52}_{\pm 0.02}$ \\
\midrule
\multirow{2}{*}{ViT-L/32} 
& SGD & $98.73_{\pm 0.05}$  & $91.30_{\pm 0.17}$ & $85.21_{\pm 0.39}$ & $96.52_{\pm 0.15}$ & $89.13_{\pm 0.20}$ & $81.28_{\pm 0.04}$ \\ 
& SGDF & $\textbf{98.83}_{\pm 0.04}$ & $\textbf{92.20}_{\pm 0.14}$   & $\textbf{91.96}_{\pm 0.18}$  & $\textbf{96.79}_{\pm 0.12}$ & $\textbf{90.04}_{\pm 0.08}$ & $\textbf{81.38}_{\pm 0.01}$   \\
\bottomrule
\end{tabular}
}
\label{tab:vitresult}
\vspace{-1.0em}
\end{table*}

%% file: resources/main_table_adam_filter.tex
\begin{table}[ht]
\vspace{-0.2em}
\centering
\caption{Accuracy comparison between Adam and Filter-Adam.}
\vspace{-0.5em}
\begin{adjustbox}{max width=0.8\linewidth}
\begin{tabular}{c|cccc}
    \toprule
    Model & VGG11 & ResNet34 & DenseNet121  \\
    \midrule
    Filter-Adam & \textbf{62.64} & \textbf{73.98} & \textbf{74.89} \\
    \midrule
    Vanilla-Adam & 56.73 & 72.34 & \textbf{74.89}\\
    \bottomrule
\end{tabular}
\end{adjustbox}
\label{tab:accuracy_comparison}
\vspace{-0.2em}
\end{table}

%% file: resources/main_figure_sign_diffusion.tex
\begin{figure}[h]
\vspace{-0.0em}
\centering
\includegraphics[width=0.9\linewidth]{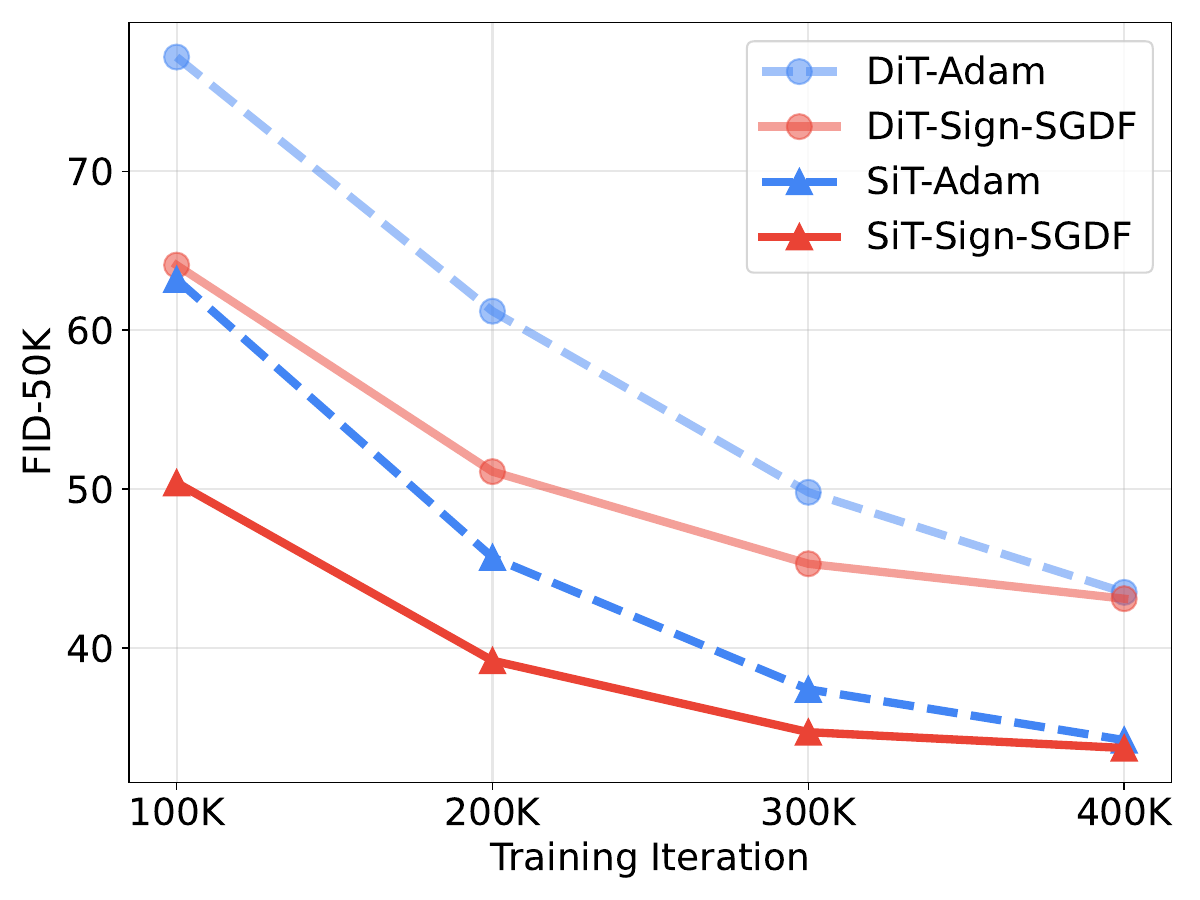}
\vspace{-1.0em}
\caption{Convergence comparison between Sign SGDF and Adam.}
\label{fig:diffusion_fid}
\vspace{-0.0em}
\end{figure}

%% file: resources/main_table_muon.tex
\begin{table}[ht]
\centering
\caption{Compatibility of filter-estimated gradient with Muon.}
\vspace{-0.5em}
\label{tab:sgdf_muon}
\begin{adjustbox}{max width=0.9\linewidth}
\begin{tabular}{lccccc}
\toprule
\textbf{Method} & FID$\downarrow$ & sFID$\downarrow$ & IS$\uparrow$ & Pre.$\uparrow$ & Rec.$\uparrow$ \\
\midrule
Adam & 68.32 & 13.63 & 20.51 & 0.36 & 0.53 \\
Muon + SGDF & \textbf{64.24} & \textbf{12.43} & \textbf{22.26} & \textbf{0.37} & \textbf{0.59}\\
\bottomrule
\end{tabular}
\end{adjustbox}
\vspace{-0.0em}
\end{table}

%% file: resources/main_figure_hessian.tex
\begin{figure}[htbp] 
\centering
\begin{subfigure}[t]{0.48\linewidth}
    \centering
    \begin{overpic}[width=\linewidth]{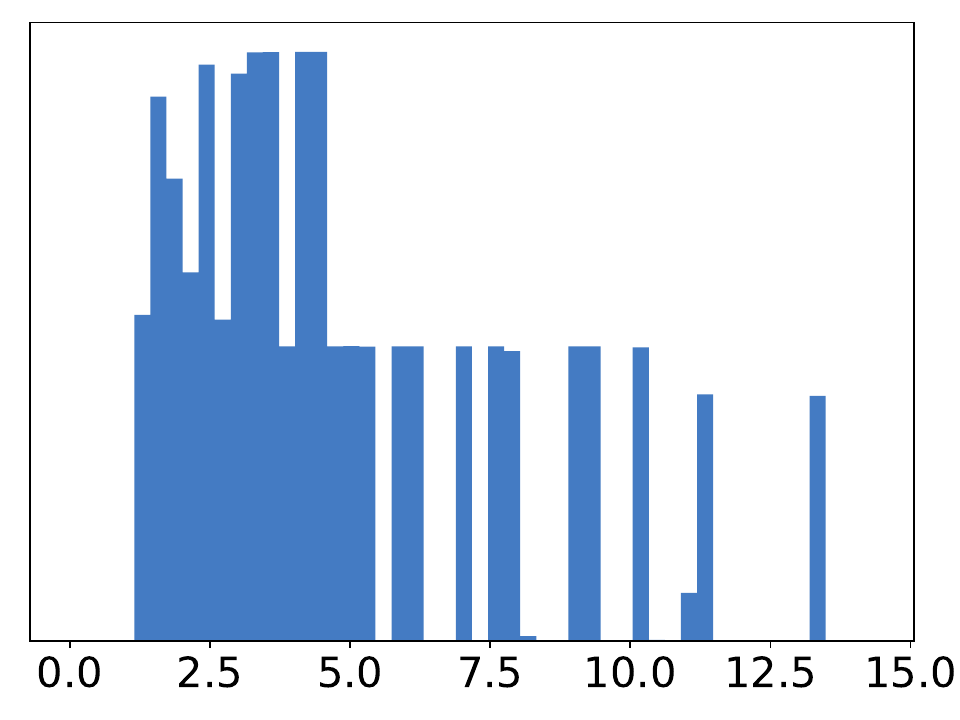}
        \put(60,65){\scriptsize Trace: 192.47 }
        \put(64,55){\scriptsize $\lambda_{\text{max}}$: 13.32}
    \end{overpic}
    \caption{SGDF}
    \label{subfig:main_hessian_sgdf}
\end{subfigure}%
\begin{subfigure}[t]{0.48\linewidth}
    \centering
    \begin{overpic}[width=\linewidth]{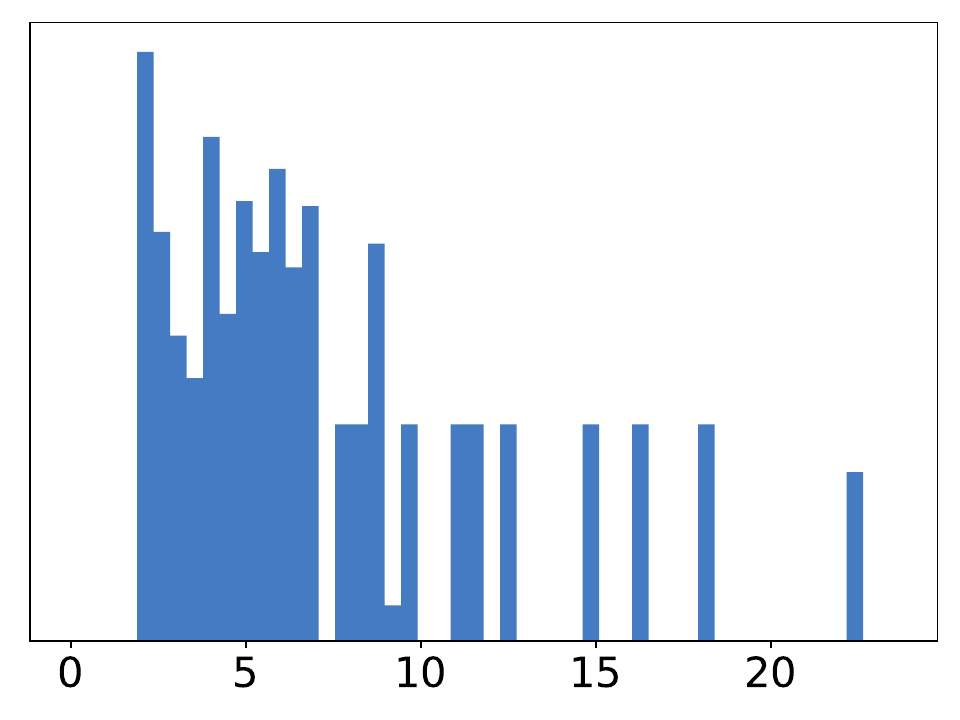}
        \put(62,65){\scriptsize Trace: 419.30}
        \put(67,55){\scriptsize $\lambda_{\text{max}}$: 22.51}
    \end{overpic}
    \caption{SGD}
    \label{subfig:main_hessian_sgdm}
\end{subfigure}%

\begin{subfigure}[t]{0.48\linewidth}
    \centering
    \begin{overpic}[width=\linewidth]{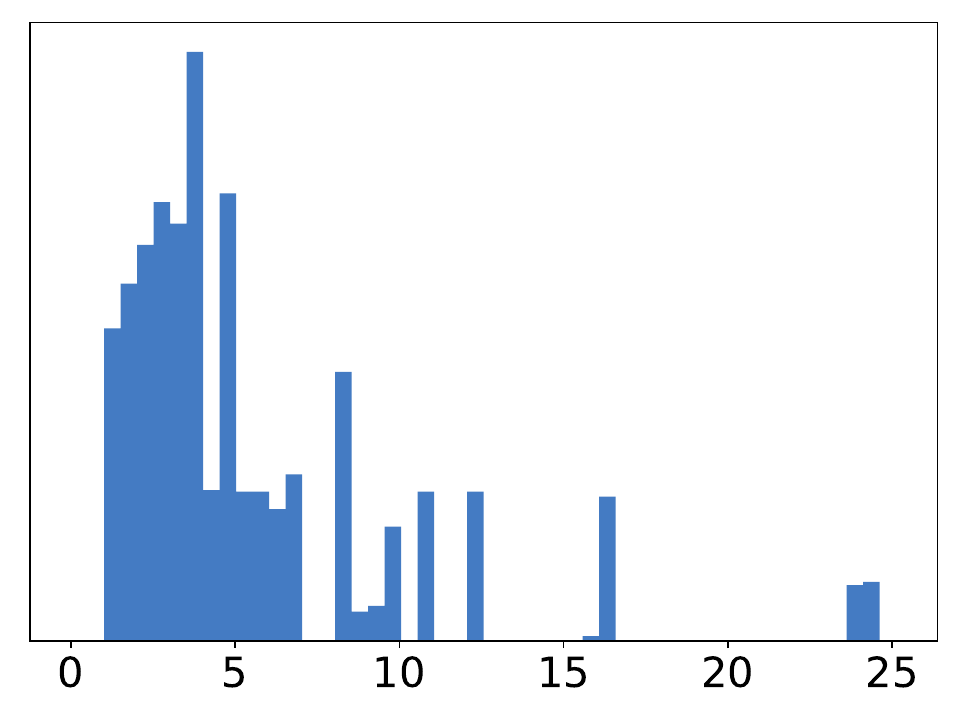}
        \put(62,65){\scriptsize Trace: 284.38}
        \put(67,55){\scriptsize $\lambda_{\text{max}}$: 24.11}
    \end{overpic}
    \caption{SGD-EMA}
    \label{subfig:main_hessian_sgd}
\end{subfigure}%
\begin{subfigure}[t]{0.48\linewidth}
    \centering
    \begin{overpic}[width=\linewidth]{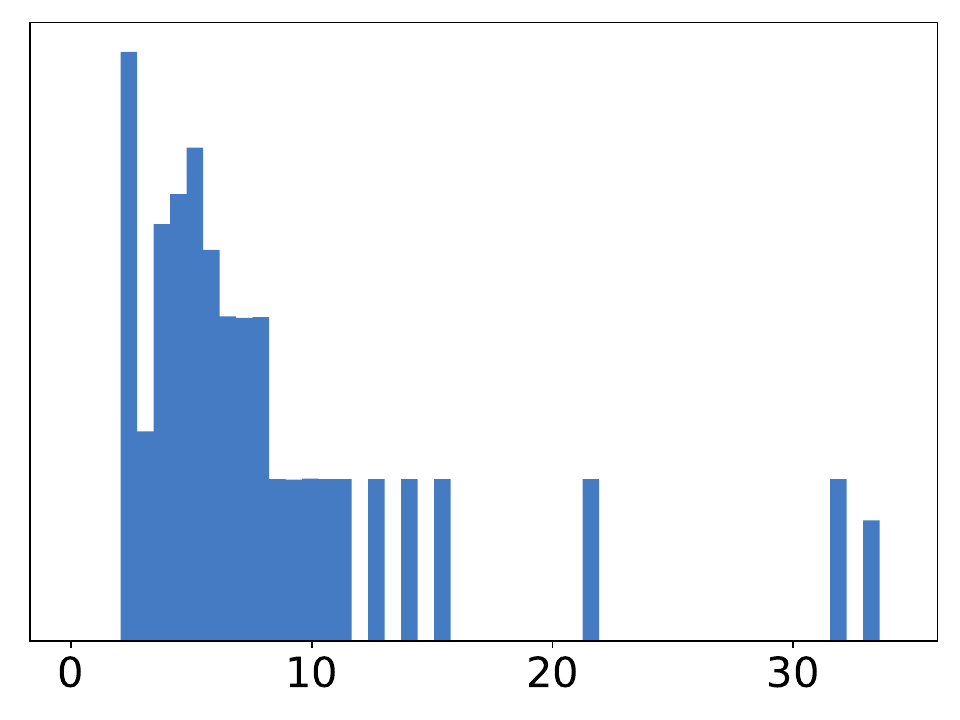}
        \put(62,65){\scriptsize Trace: 491.63}
        \put(67,55){\scriptsize $\lambda_{\text{max}}$: 32.61}
    \end{overpic}
    \caption{SGD-CM}
    \label{subfig:main_hessian_adam}
\end{subfigure}
\caption{Histogram of Top 50 Hessian Eigenvalues. Lower values indicate better performance on the test dataset.}
\label{fig:main_hessian_spectrum}
\vspace{-1em}
\end{figure}
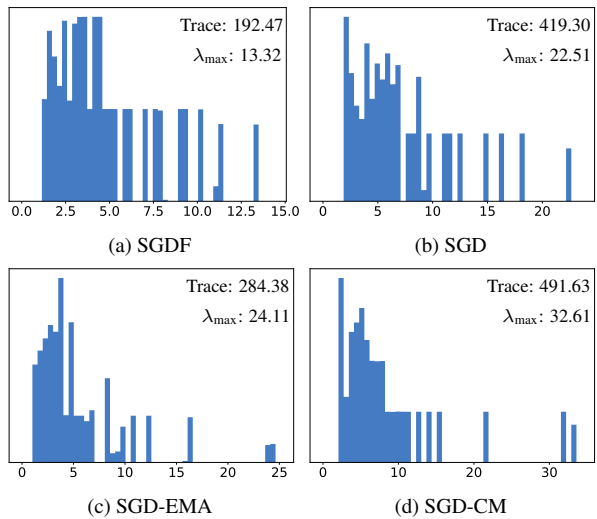

%% file: texfiles/Related_work.tex
\section{Related Works}
Early optimization efforts focused on variance reduction~\cite{johnson2013accelerating,Defazio2014saga,Schmidt2017minimizing,balles2018dissecting} and adaptive learning rates~\cite{Duchi2011Adaptive,2012ADADELTA,dozat2016incorporating}. However, standard adaptive optimizers often converge to sharp minima with weak generalization in high-dimensional non-convex landscapes~\cite{2017The,2015Train,2022On,2014Identifying,lucchi2022theoretical}. This motivated geometry-aware regularization, such as Sharpness-Aware Minimization (SAM) and its variants~\cite{foret2021sharpness,zhuang2022surrogate,Zhang2023}, alongside momentum tuning strategies like Adaptive Inertia~\cite{xie2022adaptive}, to explicitly penalize sharpness and escape sharp basins.

Beyond loss geometry, recent studies highlight the critical role of dynamically balancing gradient bias and variance for stable learning~\cite{ghosh2022gradient,ha2024fine}. To achieve this, parallel works manipulate updates directly: Quasi-Hyperbolic Momentum (QHM)~\cite{ma2018quasi} introduces a static interpolation between current gradients and momentum, while Grokfast~\cite{lee2024grokfast} heuristically amplifies low-frequency gradient signals. Alternatively, methods based on second-order approximations or Kalman filtering~\cite{yao2020adahessian,liu2023sophia,Kalman1960,Vuckovic2018,Ollivier2019,davtyan2022koala,gomes2024adafisher} refine updates using local curvature or Fisher information. While principled, these techniques significantly inflate computational overhead and parameter complexity, limiting their practical scalability. 

In contrast, we formulate momentum fundamentally as an online filtering process to rigorously resolve the bias-variance dilemma without the heavy overhead of curvature estimation. Unlike QHM's static interpolation or Grokfast's heuristic frequency scaling, SGDF fuses past and current gradients via an optimal, time-varying gain that adaptively minimizes mean-squared estimation error. Consequently, SGDF achieves noise-aware, stable updates through a lightweight statistical design, ensuring principled gradient refinement at no additional computational cost.

%% file: texfiles/Limited.tex
\section{Discussion and Future Work}
\label{dicuss}
Our SDE framework readily accommodates existing continuous-time analyses~\cite{stephan2017stochastic,li2019stochastic,xie2022adaptive}. A compelling theoretical direction is coupling first-order gradient statistics with zeroth-order loss moments via Itô's lemma, establishing unified bounds that link local stability to global generalization in non-convex settings. Practically, while SGDF currently requires a memory footprint comparable to standard Adam, integrating block-wise or reduced-state approximations, which are inspired by recent memory-efficient designs like Adam-mini~\cite{zhang2024adam}, presents a highly promising avenue to minimize overhead. Together, these theoretical and empirical extensions will rigorously refine adaptive optimization for large-scale learning systems.

%% file: texfiles/Conclusion.tex
\section{Conclusion}
In this work, we introduced SGDF, a novel optimizer rooted in statistical signal processing. By dynamically minimizing the mean-squared estimation error, SGDF balances gradient bias and variance, overcoming the suboptimal tradeoffs of traditional momentum to yield highly accurate first-moment estimates. Extensive experiments across diverse architectures confirm that this principled approach achieves a superior tradeoff between convergence speed and generalization compared to current state-of-the-art optimizers.

%% file: texfiles/Acknowledgement.tex
\section*{Acknowledgement}
The authors gratefully acknowledge the financial support provided by the Scientific Research Project of Liaoning Provincial Department of Education under Grant LJ212510149013. Rui Yu received no funding in support of this work. Zhipeng Yao thanks to \href{https://araachie.github.io/}{Aram Davtyan} and \href{https://cvg.unibe.ch/people/favaro}{Prof. Dr. Paolo Favaro} in the Computer Vision Group at the University of Bern for discussing and improving the paper.

%% file: texfiles/Appendix.tex
\begin{center}
\Large
\textbf{Dynamic Momentum Recalibration in Online Gradient Learning}\\
\textbf{Appendix}
\end{center}

\section{Bias-Variance Decomposition (Section 2 in main paper)}
\label{sec:appendixa}
\begin{definition}
\label{def:unify_momentum}
The unified momentum update rule is defined as:
\begin{equation}
    m_t = \beta m_{t-1} + u g_t, \quad 
    \theta_t = \theta_{t-1} - \alpha m_t,
\end{equation}
where $\beta \in [0, 1)$ denotes the momentum (decay) coefficient, and $u \ge 1 - \beta$ is a scaling parameter controlling the contribution of the current gradient. The stochastic gradient is given by $g_t = \nabla f_t(\theta_t) + \epsilon_t, \epsilon_t \sim \mathcal{N}(0, \sigma^2 I)$, where $f_t(\theta) = f(\theta; \xi_t)$ denotes the stochastic objective at iteration $t$ with $\xi_t$ sampled from the data distribution $\mathcal{D}$. Specific cases include: 
\begin{itemize}
    \item $u = 1 - \beta$: Exponential Moving Average (EMA),
    \item $u = 1$: Classical Momentum (CM)~\cite{polyak1964some,sutskever2013importance}.
\end{itemize}
\end{definition}

\begin{assumption}
\label{bias-variance}
We make the following assumptions on the smoothness and stochasticity of the objective function $f$:
\begin{enumerate}
    \item \textbf{Lipschitz continuity:} There exists a constant $L > 0$ such that, for any $\theta$ and $\phi$, $\|\nabla f(\theta) - \nabla f(\phi)\| \le L \|\theta - \phi\|.$
    
    \item \textbf{Bounded gradients:} There exists a constant $G > 0$ such that, for all $t$, $\|\nabla f(\theta_t)\| \le G.$
    
    \item \textbf{Bounded gradient noise:} The stochastic gradient noise $\epsilon_t$ is temporally uncorrelated with zero mean, i.e., $\mathbb{E}[\epsilon_t] = \mathbf{0}$, for all $t$. Furthermore, the variance of the stochastic gradients is uniformly bounded, meaning there exists a constant $\sigma > 0$ such that $\mathbb{E}\!\left[\|g_t - \nabla f(\theta_t)\|^2\right] \le \sigma^2.$
\end{enumerate}
\end{assumption}

\begin{lemma}[Bias-Variance Decomposition]
\label{lem:bias_variance}
Let the stochastic gradient be given by $g_t = \nabla f(\theta_t) + \epsilon_t,  \epsilon_t \sim \mathcal{N}(0, \sigma^2 I)$, where $\nabla f(\theta_t)$ denotes the true gradient and $\epsilon_t$ represents zero-mean Gaussian noise. For any gradient estimator $\hat{g}_t = \mathcal{A}(g_1, \ldots, g_t)$ produced by an arbitrary algorithm $\mathcal{A}$, the mean squared error (MSE) satisfies the bias-variance decomposition:
\begin{equation}
    \mathbb{E}\big[\|\hat{g}_t - \nabla f(\theta_t)\|^2\big]
    = 
    \underbrace{\|\mathbb{E}[\hat{g}_t] - \nabla f(\theta_t)\|^2}_{\mathrm{Bias}^2}
    +
    \underbrace{\mathbb{E}\big[\|\hat{g}_t - \mathbb{E}[\hat{g}_t]\|^2\big]}_{\mathrm{Variance}}.
\end{equation}
\end{lemma}

\begin{proof}
\begin{align}
    \mathbb{E}\!\left[\|\hat{g}_t - \nabla f(\theta_t)\|^2\right]
    &= 
    \mathbb{E}\!\left[\|\hat{g}_t - \mathbb{E}[\hat{g}_t] + \mathbb{E}[\hat{g}_t] - \nabla f(\theta_t)\|^2\right] \nonumber \\
    &= 
    \mathbb{E}\!\left[\|\hat{g}_t - \mathbb{E}[\hat{g}_t]\|^2\right]
    + \|\mathbb{E}[\hat{g}_t] - \nabla f(\theta_t)\|^2 \nonumber  + 2\,\mathbb{E}\!\left[\langle \hat{g}_t - \mathbb{E}[\hat{g}_t],\, \mathbb{E}[\hat{g}_t] - \nabla f(\theta_t) \rangle \right] \nonumber \\
    &= 
    \mathrm{Var}(\hat{g}_t)
    + \mathrm{Bias}^2(\hat{g}_t)
    + 2 \underbrace{\langle \mathbb{E}[\hat{g}_t - \mathbb{E}[\hat{g}_t]],\, \mathbb{E}[\hat{g}_t] - \nabla f(\theta_t) \rangle}_{= 0} \nonumber \\
    &= 
    \mathrm{Var}(\hat{g}_t)
    + \mathrm{Bias}^2(\hat{g}_t).
\end{align}
\end{proof}


\begin{lemma}
\label{lem:unify_momentum_sde}
Refer to the SDEs of vanilla SGD~\cite{stephan2017stochastic}, the \cref{def:unify_momentum} with learning rate $\alpha$ can be represented in continuous time as the stochastic differential equation (SDE):
\begin{equation}
    \begin{cases}
        d m(t) = [-(1 - \beta) m(t) + u \nabla f(\theta(t))]\, dt + u \sigma\, dW(t), \\
        d \theta(t) = -\alpha\, m(t)\, dt,
    \end{cases}
\end{equation}
where $m(t)$ is the momentum, $\theta(t)$ is the parameter, $\beta \in [0, 1)$ is the momentum coefficient, $u \ge 1-\beta$ is the gradient scaling factor, $\alpha > 0$ is the learning rate, $\nabla f(\theta(t))$ is the gradient of the objective function, $\sigma$ is the noise standard deviation, and $W(t)$ is a standard $d$-dimensional Wiener process. This approximation holds when the learning rate $\alpha$ is sufficiently small.
\end{lemma}

\begin{proof}
Start with the discrete momentum update rule from \cref{def:unify_momentum}:
\begin{equation}
	m_{n+1} = \beta m_n + u g(t_n), \quad \theta_{n+1} = \theta_n - \alpha m_{n+1},
\end{equation}
where $m_n = m(t_n)$ is the momentum, $\theta_n = \theta(t_n)$ is the parameter, $g(t_n) = \nabla f(\theta_n) + \epsilon_n$ with $\epsilon_n \sim \mathcal{N}(0, \sigma^2 I)$, and $\alpha > 0$ is the learning rate.

Rewrite the momentum update in terms of the increment:
\begin{equation}
	\begin{aligned}
		m_{n+1} - m_n &= \beta m_n + u g(t_n) - m_n \\
		&= -(1 - \beta) m_n + u g(t_n) \\
		&= -(1 - \beta) m_n + u \nabla f(\theta_n) + u \epsilon_n.
	\end{aligned}
\end{equation}

For the parameter:
\begin{equation}
	\theta_{n+1} - \theta_n = -\alpha m_{n+1}.
\end{equation}

To model this as a continuous-time SDE, assume the learning rate $\alpha$ is sufficiently small, controlling the step size of the discrete updates. Define $t_n = n$ to index discrete iterations, each corresponding to a unit time step $dt = 1$. The learning rate $\alpha$ controls the update magnitude but does not rescale time.

For the momentum, interpret the increment as the rate of change over one iteration:
\begin{equation}
	m_{n+1} - m_n \approx [-(1 - \beta) m_n + u \nabla f(\theta_n)]\, dt + u \epsilon_n,
	\quad dt = 1.
\end{equation}

For small constant step sizes $\alpha$, this yields the drift:
\begin{equation}
	d m(t) = [-(1 - \beta) m(t) + u \nabla f(\theta(t))]\, dt.
\end{equation}

For the stochastic part, assume $\epsilon_n = \sigma Z_n,\; Z_n \sim \mathcal{N}(0, I)$, so that the stochastic increment $u \epsilon_n$ has variance $u^2 \sigma^2$ per iteration, matching the Brownian term $u \sigma\, dW(t)$ under the step-time scaling $dt = 1$.

Therefore, in the SDE limit, the momentum dynamics can be written as:
\begin{equation}
    d m(t) = [-(1 - \beta) m(t) + u \nabla f(\theta(t))]\, dt + u \sigma\, dW(t).
\end{equation}

For the parameter update:
\begin{equation}
	\theta_{n+1} - \theta_n = -\alpha m_{n+1} \approx -\alpha m(t) \cdot (\text{time step}),
\end{equation}

where the time step is implicitly $dt$ in the continuous limit, yielding:
\begin{equation}
	d \theta(t) = -\alpha m(t)\, dt.
\end{equation}

Combining both, when $\alpha$ is small, the discrete updates approximate:
\begin{equation}
	\begin{cases}
		d m(t) = [-(1 - \beta) m(t) + u \nabla f(\theta(t))]\, dt + u \sigma\, dW(t), \\
		d \theta(t) = -\alpha m(t)\, dt.
	\end{cases}
\end{equation}

This SDE captures the dynamics of the momentum and parameter updates, with $\alpha$ as the learning rate driving the continuous approximation.
\end{proof}

\begin{remark}[Step-time scaling]
Our continuous-time formulation adopts the \emph{step-time} scaling of Mandt et al.~\cite{stephan2017stochastic}.
An alternative is the \emph{slow-time} scaling $t = n \alpha$, often used in stochastic modified equations~\cite{li2019stochastic}.
In that regime, one typically sets $1 - \beta = \Theta(\alpha)$, and the diffusion term scales with $\sqrt{\alpha}$.
We do not adopt this scaling here, since doing so would modify both the drift and diffusion coefficients, as well as the form of $d\theta$.
\end{remark}

\begin{lemma}
\label{lem:unify_momentum_solution}
Under \cref{lem:unify_momentum_sde}, the solution to the stochastic differential equation (SDE), with initial conditions $m(0) = 0$ and $\theta(0) = \theta_0$, is given by:
\begin{equation}
	\begin{cases}
		m(t) = u \displaystyle\int_0^t e^{-(1 - \beta) (t - s)} \nabla f(\theta(s))\, ds 
		+ u \sigma \displaystyle\int_0^t e^{-(1 - \beta) (t - s)}\, dW(s), \\[6pt]
		\theta(t) = \theta_0 - \alpha \displaystyle\int_0^t m(s)\, ds,
	\end{cases}
\end{equation}
where $W(t)$ is a standard Wiener process, and the integrals represent the stochastic evolution driven by the gradient $\nabla f(\theta(t))$ and noise.
\end{lemma}

\begin{proof}
We solve the coupled stochastic differential equation (SDE) system step-by-step:
\begin{equation}
	\begin{cases}
		d m(t) = [-(1 - \beta) m(t) + u \nabla f(\theta(t))]\, dt + u \sigma\, dW(t), \\
		d \theta(t) = -\alpha\, m(t)\, dt,
	\end{cases}
\end{equation}
with initial conditions $m(0) = 0$ and $\theta(0) = \theta_0$.

The $\theta(t)$ dynamics have drift only (no explicit diffusion term), but $\theta(t)$ is still random because $m(t)$ is random.
Integrate:
\begin{equation}
	\begin{aligned}
		d \theta(t) &= -\alpha\, m(t)\, dt, \\
		\theta(t) - \theta(0) &= -\alpha \int_0^t m(s)\, ds.
	\end{aligned}
\end{equation}

Since $\theta(0) = \theta_0$, we obtain:
\begin{equation}
	\theta(t) = \theta_0 - \alpha \int_0^t m(s)\, ds.
\end{equation}

This expresses $\theta(t)$ as a functional of $m(t)$, which we now determine.

Consider the linear SDE for $m(t)$ with a time-dependent forcing term:
\begin{equation}
	d m(t) = [-(1 - \beta) m(t) + u \nabla f(\theta(t))]\, dt + u \sigma\, dW(t).
\end{equation}

Rewrite it in standard form:
\begin{equation}
	d m(t) + (1 - \beta) m(t)\, dt = u \nabla f(\theta(t))\, dt + u \sigma\, dW(t).
\end{equation}

To solve this, apply the integrating factor $e^{\int_0^t (1 - \beta)\, ds} = e^{(1 - \beta) t}$. Multiply through by $e^{(1 - \beta) t}$:
\begin{equation}
	\begin{aligned}
		e^{(1 - \beta) t}\, d m(t) + (1 - \beta) e^{(1 - \beta) t} m(t)\, dt 
		&= u e^{(1 - \beta) t} \nabla f(\theta(t))\, dt + u \sigma e^{(1 - \beta) t}\, dW(t).
	\end{aligned}
\end{equation}

Recognize the left-hand side as the differential of a product:
\begin{equation}
	\begin{aligned}
		d [e^{(1 - \beta) t} m(t)] 
		&= e^{(1 - \beta) t}\, d m(t) + (1 - \beta) e^{(1 - \beta) t} m(t)\, dt \\
		&= u e^{(1 - \beta) t} \nabla f(\theta(t))\, dt + u \sigma e^{(1 - \beta) t}\, dW(t).
	\end{aligned}
\end{equation}

Integrate both sides from $0$ to $t$, with $m(0) = 0$, this simplifies to:
\begin{equation}
	\begin{aligned}
		e^{(1 - \beta) t} m(t) - e^{(1 - \beta)\cdot 0} m(0) 
		&= u \int_0^t e^{(1 - \beta) s} \nabla f(\theta(s))\, ds 
		+ u \sigma \int_0^t e^{(1 - \beta) s}\, dW(s), \\
		e^{(1 - \beta) t} m(t) 
		&= u \int_0^t e^{(1 - \beta) s} \nabla f(\theta(s))\, ds 
		+ u \sigma \int_0^t e^{(1 - \beta) s}\, dW(s), \\
		m(t) 
		&= u \int_0^t e^{-(1 - \beta)(t - s)} \nabla f(\theta(s))\, ds 
		+ u \sigma \int_0^t e^{-(1 - \beta)(t - s)}\, dW(s),
	\end{aligned}
\end{equation}
where the exponent is adjusted using $e^{(1 - \beta) s}/e^{(1 - \beta) t} = e^{-(1 - \beta)(t - s)}$.

The expression for $m(t)$ depends on $\theta(s)$ via $\nabla f(\theta(s))$, where:
\begin{equation}
	\theta(s) = \theta_0 - \alpha \int_0^s m(\tau)\, d\tau.
\end{equation}

Thus, the complete solution is:
\begin{equation}
	\begin{cases}
		m(t) = u \displaystyle\int_0^t e^{-(1 - \beta)(t - s)} \nabla f(\theta(s))\, ds 
		+ u \sigma \displaystyle\int_0^t e^{-(1 - \beta)(t - s)}\, dW(s), \\[6pt]
		\theta(t) = \theta_0 - \alpha \displaystyle\int_0^t m(s)\, ds.
	\end{cases}
\end{equation}

This integral form encapsulates the coupled dynamics, with $\nabla f(\theta(t))$ linking the equations and the stochastic term $\int e^{-(1 - \beta)(t - s)}\, dW(s)$ as an Itô integral.
\end{proof}

\begin{theorem}
\label{thm:bias_variance_momentum}
Consider the unified momentum estimator $m(t)$ defined by the stochastic differential equation (SDE) from \cref{lem:unify_momentum_sde}, with solution given in \cref{lem:unify_momentum_solution}. Let the bias be defined relative to the expected true gradient: $\mathrm{Bias}(m(t)) = \mathbb{E}[m(t)] - \mathbb{E}[\nabla f(\theta(t))]$. Assuming that the gradient $\nabla f(\theta(t))$ is bounded and Lipschitz continuous, the asymptotic bounds (as $t \to \infty$) for the bias and variance of $m(t)$ as an estimator are given by:
\begin{equation}
     \left\| \mathrm{Bias}(m(t)) \right\|^2 \leq \left( \frac{u^2 \alpha L G}{(1 - \beta)^3} + \frac{u^2 \alpha L \sigma}{\sqrt{2}(1 - \beta)^{2.5}} + \left( \frac{u}{1 - \beta} - 1 \right) G \right)^2,
\end{equation}
where $L$ is the Lipschitz constant, $G$ bounds the gradient norm $\|\nabla f(\theta(t))\|$, and the second term inside the parenthesis explicitly captures the parameter-shift bias induced by the stochastic noise $\sigma$.

\begin{equation}
    \mathrm{Var}(m(t)) \leq \frac{u^2 \sigma^2}{1 - \beta} + \frac{2 u^2 V^2}{(1 - \beta)^2},
\end{equation}
where $\sigma^2$ is the total variance of the stochastic gradient noise, and $V^2$ conservatively bounds the variance of the true gradient sequence, i.e., $\mathrm{Var}(\nabla f(\theta(t))) \leq V^2$.
\end{theorem}

\begin{proof}
We compute the bias and variance of $m(t)$ relative to $\mathbb{E}[\nabla f(\theta(t))]$.

\textbf{1. Bias Calculation}

Consider the unified momentum update rule:
\begin{equation}
    m_t = \beta m_{t-1} + u g_t, \quad \theta_t = \theta_{t-1} - \alpha m_t,
\end{equation}
where $\beta \in [0, 1)$ represents the decay or momentum factor, $u \in [1 - \beta, 1]$ is a scaling parameter controlling the gradient contribution, $\alpha > 0$ is the learning rate, and $g_t = \nabla f(\theta_t) + \epsilon_t$ with $\epsilon_t \sim \mathcal{N}(0, \sigma^2 I)$.

In continuous time, the expectation of $m(t)$ is:
\begin{equation}
    \mathbb{E}[m(t)] = u \int_0^t e^{-(1 - \beta)(t - s)} \mathbb{E}[\nabla f(\theta(s))]\, ds,
\end{equation}
since the stochastic term has zero mean:
\begin{equation}
    \mathbb{E}\!\left[ u \sigma \int_0^t e^{-(1 - \beta)(t - s)}\, dW(s) \right] = 0.
\end{equation}

The squared bias is defined as:
\begin{equation}
\begin{aligned}
    \left( \mathrm{Bias}(m(t)) \right)^2 &= \left( \mathbb{E}[m(t)] - \mathbb{E}[\nabla f(\theta(t))] \right)^2 \\
    &= \left( u \int_0^t e^{-(1 - \beta)(t - s)} \mathbb{E}[\nabla f(\theta(s))]\, ds - \mathbb{E}[\nabla f(\theta(t))] \right)^2.
\end{aligned}
\end{equation}

We assume $\nabla f$ is Lipschitz continuous with constant $L > 0$:
\begin{equation}
    \|\nabla f(\theta) - \nabla f(\phi)\| \leq L \|\theta - \phi\|, \quad \forall \theta, \phi.
\end{equation}

Given $u \geq 1 - \beta$, so $\frac{u}{1 - \beta} \geq 1$. From the continuous-time dynamics $\frac{d\theta}{dt} = -\alpha m(t)$, integrating from $s$ to $t$ ($s < t$) yields:
\begin{equation}
    \theta(s) - \theta(t) = \alpha \int_s^t m(u)\, du.
\end{equation}

To bound $\mathbb{E}[\|\theta(s) - \theta(t)\|]$, we must first bound the magnitude of the momentum $\mathbb{E}[\|m(u)\|]$ considering both the gradient drift and the noise diffusion:
\begin{equation}
    m(u) = u \int_0^u e^{-(1 - \beta)(u - v)} \nabla f(\theta(v))\, dv + u \sigma \int_0^u e^{-(1 - \beta)(u - v)}\, dW(v).
\end{equation}

Taking the expectation of the norm and applying Jensen's inequality to the stochastic term:
\begin{equation}
\begin{aligned}
    \mathbb{E}[\|m(u)\|] &\leq u \int_0^u e^{-(1 - \beta)(u - v)} \mathbb{E}[\|\nabla f(\theta(v))\|]\, dv + \mathbb{E}\!\left[ \left\| u \sigma \int_0^u e^{-(1 - \beta)(u - v)}\, dW(v) \right\| \right] \\
    &\leq \frac{u G}{1 - \beta} + \sqrt{ u^2 \sigma^2 \frac{1 - e^{-2(1 - \beta)u}}{2(1 - \beta)} } \\
    &\leq \frac{u G}{1 - \beta} + \frac{u \sigma}{\sqrt{2(1 - \beta)}} := M.
\end{aligned}
\end{equation}

Thus, taking the expected norm for the parameter difference:
\begin{equation}
    \mathbb{E} \left[ \|\theta(s) - \theta(t)\| \right] \leq \alpha \int_s^t \mathbb{E}[\|m(u)\|]\, du \leq \alpha M (t - s).
\end{equation}

Rewrite the bias by splitting the integral:
\begin{equation}
\begin{aligned}
    \mathrm{Bias}(m(t)) &= u \int_0^t e^{-(1 - \beta)(t - s)} \left( \mathbb{E}[\nabla f(\theta(s))] - \mathbb{E}[\nabla f(\theta(t))] \right)\, ds \\
    &\quad + \mathbb{E}[\nabla f(\theta(t))] \left( u \int_0^t e^{-(1 - \beta)(t - s)}\, ds - 1 \right).
\end{aligned}
\end{equation}

Compute the second integral:
\begin{equation}
    u \int_0^t e^{-(1 - \beta)(t - s)}\, ds = u \frac{1 - e^{-(1 - \beta)t}}{1 - \beta}.
\end{equation}

Apply the triangle inequality and the bounded gradient assumption:
\begin{equation}
    \|\mathrm{Bias}(m(t))\| \leq \underbrace{\left\| u \int_0^t e^{-(1 - \beta)(t - s)} \left( \mathbb{E}[\nabla f(\theta(s))] - \mathbb{E}[\nabla f(\theta(t))] \right)\, ds \right\|}_{:= I_1} + \left| u \frac{1 - e^{-(1 - \beta)t}}{1 - \beta} - 1 \right| G.
\end{equation}

Bound $I_1$ using Lipschitz continuity and our bound $M$:
\begin{equation}
\begin{aligned}
    \|I_1\| &\leq u \int_0^t e^{-(1 - \beta)(t - s)} \mathbb{E}\big[\|\nabla f(\theta(s)) - \nabla f(\theta(t))\|\big]\, ds \\
    &\leq u L \alpha M  \int_0^t e^{-(1 - \beta)(t - s)} (t - s)\, ds.
\end{aligned}
\end{equation}

Evaluate the integral:
\begin{equation}
\begin{aligned}
    \int_0^t e^{-(1 - \beta)(t - s)} (t - s)\, ds &= \int_0^t e^{-(1 - \beta)\tau} \tau\, d\tau \\
    &= \frac{1}{(1 - \beta)^2} - \left( \frac{t}{1 - \beta} + \frac{1}{(1 - \beta)^2} \right) e^{-(1 - \beta)t} \leq \frac{1}{(1 - \beta)^2},
\end{aligned}
\end{equation}
thus $\|I_1\| \leq \frac{u \alpha L M}{(1 - \beta)^2}$. Then:
\begin{equation}
    \|\mathrm{Bias}(m(t))\| \leq \frac{u \alpha L M}{(1 - \beta)^2} + \left| u \frac{1 - e^{-(1 - \beta)t}}{1 - \beta} - 1 \right| G.
\end{equation}

As $t \to \infty$, $e^{-(1 - \beta)t} \to 0$, giving $\left| u \frac{1 - e^{-(1 - \beta)t}}{1 - \beta} - 1 \right| \to \frac{u}{1 - \beta} - 1$. Substitute $M = \frac{u G}{1 - \beta} + \frac{u \sigma}{\sqrt{2(1 - \beta)}}$ and square the bound:
\begin{equation}
\begin{aligned}
    \left\| \mathrm{Bias}(m(t)) \right\|^2 &\leq \left( \frac{u \alpha L M}{(1 - \beta)^2} + \left( u \frac{1 - e^{-(1 - \beta)t}}{1 - \beta} - 1 \right) G \right)^2 \\
    &\leq \left( \frac{u \alpha L}{(1 - \beta)^2} \left( \frac{u G}{1 - \beta} + \frac{u \sigma}{\sqrt{2(1 - \beta)}} \right) + \left( \frac{u}{1 - \beta} - 1 \right) G \right)^2 \\
    &= \left( \frac{u^2 \alpha L G}{(1 - \beta)^3} + \frac{u^2 \alpha L \sigma}{\sqrt{2}(1 - \beta)^{2.5}} + \left( \frac{u}{1 - \beta} - 1 \right) G \right)^2.
\end{aligned}
\end{equation}

\textbf{2. Variance Calculation}

The fluctuation $m(t) - \mathbb{E}[m(t)]$ is:
\begin{equation}
    m(t) - \mathbb{E}[m(t)] = \underbrace{ u \int_0^t e^{-(1 - \beta)(t - s)} \left[ \nabla f(\theta(s)) - \mathbb{E}[\nabla f(\theta(s))] \right]\, ds }_{\mathcal{T}_\text{Grad Diff}} + \underbrace{ u \sigma \int_0^t e^{-(1 - \beta)(t - s)}\, dW(s) }_{\mathcal{T}_\text{Noise Diff}}.
\end{equation}

Using the inequality $\|a + b\|^2 \le 2\|a\|^2 + 2\|b\|^2$, the variance becomes:
\begin{equation}
\begin{aligned}
    \mathrm{Var}(m(t)) &= \mathbb{E}\!\left[ \left\| m(t) - \mathbb{E}[m(t)] \right\|^2 \right] \\
    &\le 2\, \mathbb{E}\!\left[ \|\mathcal{T}_\text{Grad Diff}\|^2 \right] + 2\, \mathbb{E}\!\left[ \|\mathcal{T}_\text{Noise Diff}\|^2 \right].
\end{aligned}
\end{equation}

The noise variance term is derived using the Itô isometry:
\begin{equation}
\begin{aligned}
    2\, \mathbb{E}\!\left[ \|\mathcal{T}_\text{Noise Diff}\|^2 \right] &= 2 u^2 \sigma^2 \mathbb{E}\!\left[ \left( \int_0^t e^{-(1 - \beta)(t - s)}\, dW(s) \right)^2 \right] \\
    &= 2 u^2 \sigma^2 \int_0^t e^{-2(1 - \beta)(t - s)}\, ds \quad (\text{Itô isometry}) \\
    &= 2 u^2 \sigma^2 \frac{1 - e^{-2(1 - \beta)t}}{2(1 - \beta)} \leq \frac{u^2 \sigma^2}{1 - \beta}.
\end{aligned}
\end{equation}

For the gradient variance term, we apply the Cauchy-Schwarz inequality to properly bound the squared norm of the integral:
\begin{equation}
\begin{aligned}
    2\, \mathbb{E}\!\left[ \|\mathcal{T}_\text{Grad Diff}\|^2 \right] &= 2 u^2\, \mathbb{E}\!\left[ \left\| \int_0^t e^{-(1 - \beta)(t - s)/2} \cdot e^{-(1 - \beta)(t - s)/2} \left[ \nabla f(\theta(s)) - \mathbb{E}[\nabla f(\theta(s))] \right]\, ds \right\|^2 \right] \\
    &\leq 2 u^2\, \mathbb{E}\!\left[ \left( \int_0^t e^{-(1 - \beta)(t - s)}\, ds \right) \left( \int_0^t e^{-(1 - \beta)(t - s)} \left\| \nabla f(\theta(s)) - \mathbb{E}[\nabla f(\theta(s))] \right\|^2\, ds \right) \right] \\
    &\leq \frac{2 u^2}{1 - \beta} \int_0^t e^{-(1 - \beta)(t - s)} \mathrm{Var}(\nabla f(\theta(s)))\, ds.
\end{aligned}
\end{equation}

By Assumption~\ref{bias-variance}, $\mathrm{Var}(\nabla f(\theta(s))) \leq V^2$, yielding:
\begin{equation}
\begin{aligned}
    2\, \mathbb{E}\!\left[ \|\mathcal{T}_\text{Grad Diff}\|^2 \right] &\leq \frac{2 u^2 V^2}{1 - \beta} \int_0^t e^{-(1 - \beta)(t - s)}\, ds \\
    &\leq \frac{2 u^2 V^2}{1 - \beta} \left( \frac{1}{1 - \beta} \right) = \frac{2 u^2 V^2}{(1 - \beta)^2}.
\end{aligned}
\end{equation}

Combining both bounds, the total variance is bounded by:
\begin{equation}
    \mathrm{Var}(m(t)) \leq \frac{u^2 \sigma^2}{1 - \beta} + \frac{2 u^2 V^2}{(1 - \beta)^2}.
\end{equation}

\end{proof}

\newpage

\section{Method Derivation (Section 3 in main paper)}
\subsection{Optimal Linear Filter Derivation for Gradient Estimation (Main paper Section 3.1)}
\label{proof_method}
In the stochastic gradient descent (SGD) process, given the sequence of gradients $\{g_i\}_{i=1}^t$, our objective is to estimate $\hat{g}_t$, which incorporates information from both historical gradients and the current gradient. The Optimal Linear Filter provides a mechanism to minimize the mean squared error in this estimation. We start by constructing $\hat{g}_t$ as a simple average and then refine it using the properties of the Optimal Linear Filter.
\begin{equation}
\begin{aligned}
    \hat{g}_t 
    &= \frac{1}{t} \sum_{i=1}^{t} g_i = \frac{1}{t} \left( \sum_{i=1}^{t-1} g_i + g_t \right) = \frac{1}{t} \sum_{i=1}^{t-1} g_i + \frac{1}{t} g_t \\
    &= \frac{1}{t} \bigg[ (t-1) \cdot \frac{1}{t-1} \sum_{i=1}^{t-1} g_i \bigg] + \frac{1}{t} g_t \\
    &= \frac{t-1}{t} \, \bar{g}_{1:t-1} + \frac{1}{t} g_t,
\end{aligned}
\end{equation}

where $\bar{g}_{1:t - 1} = \frac{1}{t - 1} \sum_{i=1}^{t - 1} g_i$ denoting the averaging of the gradient under different \(\theta_t\) to differentiate \(\bar{g}_t\) which in fixed parameter \(\theta\).

To better capture historical information, we replace the arithmetic mean $\bar{g}_{1:t - 1}$ with the momentum term $\widehat{m}_t$. Here we substitute the iteration $m_{t-1}$ with $m_{t}$ because of the absence of $m_0$ and the ease of implementation. Thus, we rewrite $\hat{g}_t$ as follows:
\begin{equation}
\begin{aligned}
	\hat{g}_t &\approx \frac{t - 1}{t} \widehat{m}_t + \frac{1}{t} g_t \\
	&= \left(1 - \frac{1}{t}\right) \widehat{m}_t + \frac{1}{t} g_t \\
	&= \widehat{m}_t - K_t \widehat{m}_t + K_t g_t \\
	&= \widehat{m}_t + K_t (g_t - \widehat{m}_t),
\end{aligned}
\end{equation}

where $K_t = \frac{1}{t}$ serves as an initial estimation gain that balances the influence of $\widehat{m}_t$ and $g_t$.

To achieve an optimal balance, we define $\hat{g}_t$ as a weighted combination of $\widehat{m}_t$ and $g_t$, aiming to minimize the variance of $\hat{g}_t$. Assuming independence between $\widehat{m}_t$ and $g_t$, we express the variance as:

\begin{equation}
\begin{aligned}
	\mathrm{Var}(\hat{g}_t) &= \mathrm{Var}((1 - K_t)\widehat{m}_t + K_t g_t) \\
	&= (1 - K_t)^2 \mathrm{Var}(\widehat{m}_t) + K_t^2 \mathrm{Var}(g_t).
\end{aligned}
\end{equation}

To find the optimal $K_t$, we take the derivative of $\mathrm{Var}(\hat{g}_t)$ with respect to $K_t$ and set it to zero:
\begin{equation}
\begin{aligned}
	\frac{\mathrm{d} \mathrm{Var}(\hat{g}_t)}{\mathrm{d} K_t} &= -2(1 - K_t) \mathrm{Var}(\widehat{m}_t) + 2 K_t \mathrm{Var}(g_t) = 0, \\
	(1 - K_t) \mathrm{Var}(\widehat{m}_t) &= K_t \mathrm{Var}(g_t),
\end{aligned}
\end{equation}

solving for $K_t$ gives:
\begin{equation}
	K_t = \frac{\mathrm{Var}(\widehat{m}_t)}{\mathrm{Var}(\widehat{m}_t) + \mathrm{Var}(g_t)}.
\end{equation}

The final expression for $K_t$ indicates that the optimal interpolation coefficient is the ratio of the variance of the momentum term to the total variance. This embodies the Optimal Linear Filter’s principle: optimally combining historical estimates with new observations to minimize estimation error due to stochastic noise in the gradient signal.

\subsection{Variance Correction (Correction factor in main paper Section 3.1)}
\label{proof_correction}

The momentum term \( m_{t} \) in stochastic gradient descent is defined as:
\begin{equation}
	m_{t} = (1 - \beta_{1}) \sum_{i=1}^{t} \beta_{1}^{t-i} g_{i},
\end{equation}

which means that \( m_t \) is a weighted sum of past gradients, where the weights decrease exponentially over time according to the factor \( \beta_1 \).

To accurately estimate the variance of $m_t$ using the variance of $g_t$, we derive a correction factor under the assumption that the stochastic gradients $g_t$ are independent with bounded variance $\sigma_{g}^{2}$. 

Each weighted gradient term \( \beta_{1}^{t-i} g_i \) has a variance of \( \beta_{1}^{2(t-i)} \sigma_{g}^{2} \), because the variance scaling factor becomes \( \beta_{1}^{2(t-i)} \) in the variance computation due to the quadratic nature of the variance operator.

Given that \( m_t \) is a sum of these weighted terms and assuming independence among \( g_i \), the variance of \( m_t \) is the sum of the variances of all weighted gradients:
\begin{equation}
	\sigma_{m_{t}}^{2} = (1 - \beta_{1})^{2} \sigma_{g}^{2} \sum_{i=1}^{t} \beta_{1}^{2(t-i)}.
\end{equation}

The factor \( (1 - \beta_{1})^2 \) appears from the multiplication factor \( (1 - \beta_{1}) \) in the definition of \( m_t \), which also applies to the variance calculation.

The summation \( \sum_{i=1}^{t} \beta_{1}^{2(t-i)} \) forms a geometric series:
\begin{equation}
	\sum_{i=1}^{t} \beta_{1}^{2(t-i)} = \frac{1 - \beta_{1}^{2t}}{1 - \beta_{1}^{2}}.
\end{equation}

As \( t \rightarrow \infty \) and given that \( \beta_1 < 1 \), we find that \( \beta_1^{2t} \rightarrow 0 \), so the series converges to:
\begin{equation}
	\sum_{i=1}^{t} \beta_1^{2(t-i)} \approx \frac{1}{1 - \beta_1^2}.
\end{equation}

Substituting back, we obtain the long-term variance of \( m_t \) as:
\begin{equation}
	\sigma^2_{m_t} = \frac{\left(1 - \beta_1\right)^2}{1 - \beta_1^2} \sigma^2_g = \frac{1 - \beta_1}{1 + \beta_1} \sigma^2_g.
\end{equation}

Thus, the correction factor we derived is:
\begin{equation}
    \left(\frac{1 - \beta_1}{1 + \beta_1}\right) \cdot (1 - \beta_1^{2t}).
\end{equation}

This correction factor \(\left(\frac{1 - \beta_1}{1 + \beta_1}\right) \cdot (1 - \beta_1^{2t})\) allows us to adjust the variance of the EMA gradient to accurately estimate the variance of the momentum gradient \( m_t \) using the original variance \( \sigma^2_g \). This adjustment reflects the effect of exponentially decaying weights in \( m_t \), yielding a more stable gradient estimate with reduced noise over time.

\subsection{Fusion of Gaussian Distributions (Main paper Section 3.2)}
\label{proof_fusion}

In this section, we address the fusion of two Gaussian distributions to produce a more reliable gradient estimate in the stochastic gradient descent (SGD) process. This fusion combines information from both the historical momentum term $\widehat{m}_t$ and the current gradient $g_t$, resulting in an estimate with reduced uncertainty. Here, "fusion" refers to finding an optimal combined distribution that minimizes mean-square error by utilizing both sources of information.

Consider the following two Gaussian distributions:
\begin{itemize}
	\item The momentum term $\widehat{m}_t$ follows a normal distribution with mean $\mu_{m}$ and variance $\sigma_{m}^2$, denoted as $\widehat{m}_t \sim \mathcal{N}(\mu_{m}, \sigma_{m}^2)$.
	\item The current gradient $g_t$ follows a normal distribution with mean $\mu_{g}$ and variance $\sigma_{g}^2$, denoted as $g_t \sim \mathcal{N}(\mu_{g}, \sigma_{g}^2)$.
\end{itemize}

\paragraph{Linear estimator perspective.} 
Before presenting the probability density product approach, we first derive the fused estimate $\hat{g}_t$ from a weighted linear combination perspective. We define $\hat{g}_t$ as:
\begin{equation}
	\hat{g}_t = (1 - K_t)\widehat{m}_t + K_t g_t,
\end{equation}

where $K_t$ is a variance-based weighting coefficient:
\begin{equation}
	K_t = \frac{\sigma_m^2}{\sigma_m^2 + \sigma_g^2}, \quad \text{thus} \quad 1 - K_t = \frac{\sigma_g^2}{\sigma_m^2 + \sigma_g^2}.
\end{equation}

Assuming independence between $\widehat{m}_t$ and $g_t$, the expectation of $\hat{g}_t$ is:
\begin{equation}
	\mathbb{E}[\hat{g}_t] = (1 - K_t)\mu_m + K_t \mu_g = \frac{\sigma_g^2 \mu_m + \sigma_m^2 \mu_g}{\sigma_m^2 + \sigma_g^2}.
\end{equation}

The variance of $\hat{g}_t$ becomes:
\begin{equation}
	\begin{aligned}
	\mathrm{Var}(\hat{g}_t) &= (1 - K_t)^2 \sigma_m^2 + K_t^2 \sigma_g^2 \\
	&= \left( \frac{\sigma_g^2}{\sigma_m^2 + \sigma_g^2} \right)^2 \sigma_m^2 + \left( \frac{\sigma_m^2}{\sigma_m^2 + \sigma_g^2} \right)^2 \sigma_g^2 \\
	&= \frac{\sigma_g^4 \sigma_m^2 + \sigma_m^4 \sigma_g^2}{(\sigma_m^2 + \sigma_g^2)^2} = \frac{\sigma_m^2 \sigma_g^2}{\sigma_m^2 + \sigma_g^2}.
	\end{aligned}
\end{equation}

\paragraph{Probability density product derivation.}
We now show that the same fused Gaussian distribution arises from multiplying the two individual Gaussian probability densities:
\begin{equation}
	N(\widehat{m}_t; \mu_{m}, \sigma_{m}) \cdot N(g_t; \mu_{g}, \sigma_{g}) = \frac{1}{2\pi\sigma_{m}\sigma_{g}} \exp\left(-\frac{(\hat{g}_t-\mu_{m})^2}{2\sigma_{m}^2} -\frac{(\hat{g}_t-\mu_{g})^2}{2\sigma_{g}^2}\right).
\end{equation}

To derive the fused form, we simplify the exponent by completing the square:
\begin{equation}
	\begin{aligned}
		\text{Exponent} &= -\frac{(\hat{g}_t - \mu_{m})^2}{2 \sigma_{m}^2} - \frac{(\hat{g}_t - \mu_{g})^2}{2 \sigma_{g}^2} \\
		&= -\frac{\sigma_{g}^2(\hat{g}_t - \mu_{m})^2 + \sigma_{m}^2(\hat{g}_t - \mu_{g})^2}{2 \sigma_{m}^2 \sigma_{g}^2} \\
		&= -\frac{\left(\hat{g}_t - \frac{\sigma_{g}^2 \mu_{m} + \sigma_{m}^2 \mu_{g}}{\sigma_{m}^2 + \sigma_{g}^2}\right)^2}{\frac{2 \sigma_{m}^2 \sigma_{g}^2}{\sigma_{m}^2 + \sigma_{g}^2}} + \frac{(\mu_{m} - \mu_{g})^2}{2(\sigma_{m}^2 + \sigma_{g}^2)}.
	\end{aligned}
\end{equation}

Ignoring constant terms, we identify the resulting fused distribution:
\begin{equation}
	\mu_{\hat{g}_t} = \frac{\sigma_{g}^2 \mu_{m} + \sigma_{m}^2 \mu_{g}}{\sigma_{m}^2 + \sigma_{g}^2}, \quad \sigma_{\hat{g}_t}^2 = \frac{\sigma_{m}^2 \sigma_{g}^2}{\sigma_{m}^2 + \sigma_{g}^2}.
\end{equation}

\paragraph{Equivalence and insight.}
This demonstrates that the PDF product view yields the same fused mean and variance as the minimum mean-square error (MMSE) linear estimator. The fused mean $\mu_{\hat{g}_t}$ is closer to the distribution with smaller variance, indicating greater trust in more certain estimates~\cite{faragher2012understanding}. The fused variance $\sigma_{\hat{g}_t}^2$ is always less than either original variance, demonstrating the variance reduction benefit of fusion. 

This equivalence between statistical estimation and probabilistic fusion confirms the theoretical soundness of the SGDF method. It validates the fusion-based formulation from both a Bayesian and a signal processing perspective.

\subsection{Modulating Observation Variance through Power Scaling}
\label{proof_scaling}
The preceding analysis yields the optimal linear fusion gain under the MMSE criterion
\begin{equation}
    K_t \;=\; \frac{\sigma_m^2}{\sigma_m^2 + \sigma_g^2},
\end{equation}

which balances the historical momentum estimate and the instantaneous stochastic gradient according to their uncertainties.

In practice, the variance estimates (especially $\sigma_g^2$) can be noisy or biased due to mini-batch stochasticity and nonstationarity.
To improve robustness while remaining consistent with our convergence analysis, we adopt a power-scaled gain
\begin{equation}
    \tilde{K}_t \;=\; K_t^{\gamma}, 
    \qquad \gamma=\tfrac{1}{2},
\end{equation}
since $K_t \in [0,1]$ element-wise, the scaled gain still satisfies $\|\tilde{K}_t\|_{\infty}\le 1$, which is the sole requirement for our convergence guarantees; therefore, the theoretical results remain unchanged.

With $K_t=\frac{\sigma_m^2}{\sigma_m^2+\sigma_g^2}$, the choice $\gamma=\tfrac12$ gives
\begin{equation}
    \tilde{K}_t \;=\; \sqrt{K_t}
    \;=\; \sqrt{\frac{\sigma_m^2}{\sigma_m^2+\sigma_g^2}}.
\end{equation}

We show that $\tilde{K}_t$ can be written in the same variance-fusion form as $K_t$ by introducing an \emph{effective} observation variance $\tilde{\sigma}_g^2$ such that
\begin{equation}
    \tilde{K}_t \;=\; \frac{\sigma_m^2}{\sigma_m^2+\tilde{\sigma}_g^2}.
\end{equation}

Equating the two expressions and solving for $\tilde{\sigma}_g^2$:
\begin{align}
    \frac{\sigma_m^2}{\sigma_m^2+\tilde{\sigma}_g^2}
    &= \sqrt{\frac{\sigma_m^2}{\sigma_m^2+\sigma_g^2}},
    \nonumber \\
    \frac{\sigma_m^2+\tilde{\sigma}_g^2}{\sigma_m^2}
    &= \sqrt{\frac{\sigma_m^2+\sigma_g^2}{\sigma_m^2}}
    = \sqrt{1+\frac{\sigma_g^2}{\sigma_m^2}},
    \nonumber \\
    1+\frac{\tilde{\sigma}_g^2}{\sigma_m^2}
    &= \sqrt{1+\frac{\sigma_g^2}{\sigma_m^2}},
    \nonumber \\
    \frac{\tilde{\sigma}_g^2}{\sigma_m^2}
    &= \sqrt{1+\frac{\sigma_g^2}{\sigma_m^2}} - 1.
    \label{eq:ratio}
\end{align}

Therefore, according to~\Eqref{eq:ratio}, we have:
\[
\tilde{\sigma}_g^2
\;=\;
\sigma_m^2\!\left(\sqrt{1+\tfrac{\sigma_g^2}{\sigma_m^2}} - 1\right).
\]

Therefore, the scaled gain $\tilde{K}_t=\sqrt{K_t}$ is equivalently a standard variance-based fusion gain with a reparameterized (effective) observation variance:
\begin{equation}
    \tilde{K}_t
    \;=\;
    \frac{\sigma_m^2}{\sigma_m^2 + \tilde{\sigma}_g^2},
    \qquad
    \tilde{\sigma}_g^2
    \;=\;
    \sigma_m^2\!\left(\sqrt{1+\tfrac{\sigma_g^2}{\sigma_m^2}} - 1\right).
\end{equation}
This view shows that power-scaling preserves the fusion structure while introducing a controlled regularization against overconfident (noisy) instantaneous gradient observations.

\newpage

\section{Convergence analysis in convex online learning case (Theorem 3.2 in main paper).}
\label{convex}
\begin{assumption}
	\label{conas}	
	Variables are bounded: \(\exists D, D_\infty \text{ such that } \forall t, \Vert\theta_t - \theta^*\Vert_2 \leq D, \Vert\theta_t - \theta^*\Vert_\infty \leq D_\infty\). Gradients are bounded: \(\exists G, G_\infty \text{ such that } \forall t, \Vert g_t\Vert_2 \leq G, \Vert g_t\Vert_\infty \leq G_\infty\). The interpolation parameter satisfies \( K_{t,i} \in [0, 1] \). Furthermore, we assume the interpolation parameter sequence has sublinear total variation, i.e., \(\sum_{t=1}^{T-1} \vert K_{t,i} - K_{t+1,i} \vert \leq \mathcal{O}(\sqrt{T})\).
\end{assumption}

\begin{definition}
	Let \(f_t(\theta_t)\) be the loss at time \(t\) and \(f_t(\theta^*)\) be the loss of the best possible strategy at the same time. The cumulative regret \(R(T)\) at time \(T\) is defined as:
	\begin{equation}
		R(T) = \sum_{t=1}^{T} \left( f_t(\theta_t) - f_t(\theta^* ) \right)
	\end{equation}
\end{definition}

\begin{definition}
	\label{defA.9}
	A function \(f: \mathbb{R}^d \rightarrow \mathbb{R}\) is convex if for all \(x, y \in \mathbb{R}^d\) and for all \(\lambda \in [0, 1]\),
	\begin{equation}
		\lambda f(x) + (1-\lambda) f(y) \geq f(\lambda x + (1-\lambda) y)
	\end{equation}
\end{definition}

Also, notice that a convex function can be lower bounded by a hyperplane at its tangent.

\begin{lemma}
\label{lem.B.4}
If a function \(f: \mathbb{R}^{d} \rightarrow \mathbb{R}\) is convex, then for all \(x, y \in \mathbb{R}^{d}\) ,
\begin{equation}
    f(x) - f(y) \leq \nabla f(x)^{T}(x-y)
\end{equation}
\end{lemma}

The above lemma can be used to upper bound the regret, and our proof for the main theorem is constructed by substituting the hyperplane with SGDF update rules.

We define \(g_{t} \triangleq \nabla f_{t}\left(\theta_{t}\right)\) and \(g_{t, i}\) as the \(i^{\text{th}}\) element. Let \(\hat{g}_t\) be the effective update direction.

\begin{lemma}
\label{lem.B.5}
Let gradients be bounded by \(\vert g_{t,i} \vert \leq G_\infty\). For any \(T \ge 1\), the sum of squared bounded elements discounted by \(\sqrt{t}\) is strictly bounded by:
\begin{equation}
\sum_{t=1}^{T} \frac{g_{t,i}^2}{\sqrt{t}} \le 2 G_\infty^2 \sqrt{T}
\end{equation}
\end{lemma}

\begin{proof}
Since \(g_{t,i}^2 \le G_\infty^2\), we can bound the summation using the integral test. Since \(1/\sqrt{t}\) is a monotonically decreasing function for \(t \ge 1\):
\begin{equation}
	\begin{aligned}
		\sum_{t=1}^{T} \frac{g_{t,i}^2}{\sqrt{t}} &\leq G_\infty^2 \sum_{t=1}^{T} \frac{1}{\sqrt{t}} \\
		&\leq G_\infty^2 \left( 1 + \int_{1}^{T} \frac{1}{\sqrt{t}} dt \right) \\
		&= G_\infty^2 \left( 1 + 2\sqrt{T} - 2 \right) \leq 2 G_\infty^2 \sqrt{T}
	\end{aligned}
\end{equation}
\end{proof}

\begin{lemma}
\label{lem.B.6}
Let the exponential moving average be \(m_{t,i} = \beta_1 m_{t-1,i} + (1-\beta_1)g_{t,i}\), with bias-correction \(\widehat{m}_{t,i} = m_{t,i} / (1-\beta_1^t)\). Under the update rule \(\hat{g}_{t,i} = \widehat{m}_{t,i} + K_{t,i}(g_{t,i} - \widehat{m}_{t,i})\), the effective update direction is bounded by \(\vert \hat{g}_{t,i} \vert \leq G_\infty\).
\end{lemma}

\begin{proof}
By mathematical induction, since \(m_{t,i}\) is a convex combination of past bounded gradients, we have \(\vert m_{t,i} \vert \leq (1-\beta_1^t)G_\infty\). Therefore, the bias-corrected momentum satisfies \(\vert \widehat{m}_{t,i} \vert \leq G_\infty\). The update direction is a convex combination \(\hat{g}_{t,i} = K_{t,i} g_{t,i} + (1-K_{t,i}) \widehat{m}_{t,i}\). Since both components are bounded by \(G_\infty\) and \(K_{t,i} \in [0, 1]\), we rigorously have \(\vert \hat{g}_{t,i} \vert \leq G_\infty\).

\end{proof}

\begin{lemma}[Bounded Total Variation]
\label{lem.B.7}
Assume the gradient sequence is bounded such that $\|\nabla f_t\|_\infty \leq G_\infty$ for all $t$, and the hyperparameters $\beta_1, \beta_2 \in [0, 1)$ are constants. If the power scaling factor is scheduled as $\gamma_t = \gamma_0 / \sqrt{t}$, the effective interpolation gain $\tilde{K}_{t} = K_{t}^{\gamma_t}$ satisfies:
\begin{equation}
    \sum_{t=1}^{T-1} |K_{t+1}^{\gamma_{t+1}} - K_t^{\gamma_t}| \leq \mathcal{O}(\sqrt{T})
\end{equation}
\end{lemma}

\begin{proof}
Decomposition of the total variation using the triangle inequality.
\begin{align}
    \sum_{t=1}^{T-1} |K_{t+1}^{\gamma_{t+1}} - K_t^{\gamma_t}| &\leq \sum_{t=1}^{T-1} \left( |K_{t+1}^{\gamma_{t+1}} - K_{t+1}^{\gamma_t}| + |K_{t+1}^{\gamma_t} - K_t^{\gamma_t}| \right) \nonumber \\
    &= \sum_{t=1}^{T-1} \underbrace{|K_{t+1}^{\gamma_{t+1}} - K_{t+1}^{\gamma_t}|}_{\text{Part (A)}} + \sum_{t=1}^{T-1} \underbrace{|K_{t+1}^{\gamma_t} - K_t^{\gamma_t}|}_{\text{Part (B)}}
\end{align}

Bound Part (A) representing the variation in the exponent $\gamma_t$. 
By the Mean Value Theorem for $f(\gamma) = x^\gamma$ where $x \in [\delta, 1]$ and $\delta = \frac{\varepsilon}{G_\infty^2 + \varepsilon}$ (with $\varepsilon > 0$ being a small constant), there exists $\xi_t \in [\gamma_{t+1}, \gamma_t]$ such that:
\begin{align}
    |K_{t+1}^{\gamma_{t+1}} - K_{t+1}^{\gamma_t}| &= |K_{t+1}^{\xi_t} \ln(K_{t+1})| \cdot |\gamma_{t+1} - \gamma_t| \nonumber \\
    &\leq C_{\varepsilon} \cdot \gamma_0 \left( \frac{1}{\sqrt{t}} - \frac{1}{\sqrt{t+1}} \right)
\end{align}
Summing over $t$ yields a telescoping series:
\begin{align}
    \sum_{t=1}^{T-1} |K_{t+1}^{\gamma_{t+1}} - K_{t+1}^{\gamma_t}| &\leq C_{\varepsilon} \gamma_0 \sum_{t=1}^{T-1} \left( \frac{1}{\sqrt{t}} - \frac{1}{\sqrt{t+1}} \right) \nonumber \\
    &= C_{\varepsilon} \gamma_0 \left( 1 - \frac{1}{\sqrt{T}} \right) = \mathcal{O}(1)
\end{align}

Bound Part (B) representing the variation in the base $K_t$. 
Assuming $\gamma_0 \leq 1$, and using the Lipschitz continuity of $x^{\gamma_t}$ on $[\delta, 1]$ with $\gamma_t \in (0, 1]$, we observe:
\begin{align}
    |K_{t+1}^{\gamma_t} - K_t^{\gamma_t}| &\leq \left( \sup_{x \in [\delta, 1]} \gamma_t x^{\gamma_t-1} \right) \cdot |K_{t+1} - K_t| \nonumber \\
    &\leq \gamma_t \delta^{\gamma_t-1} \cdot |K_{t+1} - K_t|
\end{align}
Even if $|K_{t+1} - K_t| = \mathcal{O}(1)$ due to constant $\beta_2$ and lack of smoothness, the decay of $\gamma_t$ ensures:
\begin{align}
    \sum_{t=1}^{T-1} |K_{t+1}^{\gamma_t} - K_t^{\gamma_t}| &\leq \text{Const} \cdot \sum_{t=1}^{T-1} \gamma_t \nonumber \\
    &= \text{Const} \cdot \gamma_0 \sum_{t=1}^{T-1} \frac{1}{\sqrt{t}} \nonumber \\
    &\leq 2 \cdot \text{Const} \cdot \gamma_0 \sqrt{T} = \mathcal{O}(\sqrt{T})
\end{align}

Combine the bounds to achieve the final sublinear variation.
\begin{align}
    \sum_{t=1}^{T-1} |K_{t+1}^{\gamma_{t+1}} - K_t^{\gamma_t}| &\leq \mathcal{O}(1) + \mathcal{O}(\sqrt{T}) \nonumber \\
    &= \mathcal{O}(\sqrt{T})
\end{align}

\end{proof}

\begin{theorem}
	\label{the.A.13}
	Assume that \cref{conas} holds, and \(\beta_1 \in [0, 1)\). Let the learning rate be \(\alpha_t = \alpha/\sqrt{t}\). For all \(T \geq 1\), SGDF achieves the following cumulative regret bound:
	\begin{equation}
		R(T) \leq \sum_{i=1}^{d} \left( \frac{D_\infty^2}{2\alpha} + \alpha G_\infty^2 \frac{1+\beta_1}{1-\beta_1} \right) \sqrt{T} + \sum_{i=1}^{d} \frac{\beta_1 G_\infty D_\infty}{1-\beta_1} \left( 2 + \sum_{t=1}^{T-1} \vert K_{t,i} - K_{t+1,i} \vert \right)
	\end{equation}
	Under the assumption that the total variation of the interpolation parameter is sublinear, \ie, $\sum_{t=1}^{T-1} \vert K_{t,i} - K_{t+1,i} \vert \leq \mathcal{O}(\sqrt{T})$, we have $R(T) \leq \mathcal{O}(\sqrt{T})$. Consequently, the average regret converges to zero: \(\lim_{T \to \infty} \frac{R(T)}{T} = 0\).
\end{theorem}

\begin{proof}
Using Lemma~\ref{lem.B.4}, we lower bound the convex functions to establish the regret connection:
\begin{equation}
	f_t(\theta_{t}) - f_t(\theta^*) \leq \langle g_t, \theta_t - \theta^* \rangle = \sum_{i=1}^{d} g_{t,i}(\theta_{t,i} - \theta_{i}^*)
\end{equation}

From Algorithm 1, the update direction incorporates bias correction and interpolation:
\begin{equation}
    \hat{g}_{t,i} = K_{t,i} g_{t,i} + (1-K_{t,i}) \widehat{m}_{t,i} \implies g_{t,i} = \hat{g}_{t,i} + (1-K_{t,i})(g_{t,i} - \widehat{m}_{t,i})
\end{equation}

Recall the momentum definition \(m_{t,i} = \beta_1 m_{t-1,i} + (1-\beta_1) g_{t,i}\), giving \(g_{t,i} - m_{t,i} = \frac{\beta_1}{1-\beta_1}(m_{t,i} - m_{t-1,i})\). Also, \(\widehat{m}_{t,i} = m_{t,i}/(1-\beta_1^t)\). Expanding the difference accurately yields:
\begin{equation}
    \begin{aligned}
        g_{t,i} - \widehat{m}_{t,i} &= (g_{t,i} - m_{t,i}) + \left( m_{t,i} - \frac{m_{t,i}}{1-\beta_1^t} \right) \\
        &= \frac{\beta_1}{1-\beta_1}(m_{t,i} - m_{t-1,i}) - \frac{\beta_1^t}{1-\beta_1^t} m_{t,i}
    \end{aligned}
\end{equation}

Substituting this back, we decompose the inner product into three parts:
\begin{equation}
	\sum_{t=1}^{T} g_{t,i}(\theta_{t,i} - \theta_{i}^*) = \underset{\text{(A)}}{\underbrace{\sum_{t=1}^{T} \hat{g}_{t,i} (\theta_{t,i} - \theta_{i}^*)}} + \underset{\text{(B)}}{\underbrace{ \frac{\beta_1}{1-\beta_1} \sum_{t=1}^{T} (1-K_{t,i}) (m_{t,i} - m_{t-1,i})(\theta_{t,i} - \theta_{i}^*) }} - \underset{\text{(C)}}{\underbrace{ \sum_{t=1}^{T} (1-K_{t,i}) \frac{\beta_1^t}{1-\beta_1^t} m_{t,i}(\theta_{t,i} - \theta_{i}^*) }}
\end{equation}

For part (A), using the parameter update rule \(\theta_{t+1,i} = \theta_{t,i} - \alpha_t \hat{g}_{t,i}\), we have:
\begin{equation}
	\hat{g}_{t,i}(\theta_{t,i} - \theta_{i}^*) = \frac{(\theta_{t,i} - \theta_{i}^*)^2 - (\theta_{t+1,i} - \theta_{i}^*)^2}{2\alpha_t} + \frac{\alpha_t}{2} \hat{g}_{t,i}^2
\end{equation}

Summing over \(T\) and noting \(\alpha_t = \alpha/\sqrt{t}\), we get a telescoping sum:
\begin{equation}
	\begin{aligned}
		\sum_{t=1}^{T} \hat{g}_{t,i}(\theta_{t,i} - \theta_{i}^*) &\leq \frac{D_\infty^2}{2\alpha_1} + \sum_{t=2}^{T} \frac{D_\infty^2}{2} \left( \frac{1}{\alpha_t} - \frac{1}{\alpha_{t-1}} \right) + \frac{\alpha}{2} \sum_{t=1}^{T} \frac{\hat{g}_{t,i}^2}{\sqrt{t}} \\
		&\leq \frac{D_\infty^2 \sqrt{T}}{2\alpha} + \alpha G_\infty^2 \sqrt{T} \quad \text{(Using Lemma \ref{lem.B.5} and \ref{lem.B.6})}
	\end{aligned}
\end{equation}

For part (B), let \(C_{t,i} = 1 - K_{t,i}\). We apply Summation by Parts (Abel transformation):
\begin{equation}
	\begin{aligned}
		&\sum_{t=1}^{T} C_{t,i} (m_{t,i} - m_{t-1,i})(\theta_{t,i} - \theta_{i}^*) \\
		&= C_{T,i} m_{T,i}(\theta_{T,i} - \theta_{i}^*) - C_{1,i} m_{0,i}(\theta_{1,i} - \theta_{i}^*) - \sum_{t=1}^{T-1} m_{t,i} \left[ C_{t+1,i}(\theta_{t+1,i} - \theta_{i}^*) - C_{t,i}(\theta_{t,i} - \theta_{i}^*) \right]
	\end{aligned}
\end{equation}

Since \(m_{0,i}=0\), the boundary term is bounded by \(G_\infty D_\infty\). We rigorously expand the difference inside the summation:
\begin{equation}
	\begin{aligned}
		&C_{t+1,i}(\theta_{t+1,i} - \theta_{i}^*) - C_{t,i}(\theta_{t,i} - \theta_{i}^*) \\
		&= C_{t+1,i}(\theta_{t+1,i} - \theta_{t,i}) + (C_{t+1,i} - C_{t,i})(\theta_{t,i} - \theta_{i}^*) \\
		&= -(1 - K_{t+1,i})\alpha_t \hat{g}_{t,i} + (K_{t,i} - K_{t+1,i})(\theta_{t,i} - \theta_{i}^*)
	\end{aligned}
\end{equation}

Taking the absolute value and substituting back, we obtain:
\begin{equation}
	\begin{aligned}
		\vert \text{(B)} \vert &\leq \frac{\beta_1}{1-\beta_1} \left( G_\infty D_\infty + \sum_{t=1}^{T-1} \vert m_{t,i} \vert \cdot \vert 1-K_{t+1,i} \vert \alpha_t \vert \hat{g}_{t,i} \vert + \sum_{t=1}^{T-1} \vert m_{t,i} \vert \cdot \vert K_{t,i} - K_{t+1,i} \vert \cdot \vert \theta_{t,i} - \theta_{i}^* \vert \right) \\
		&\leq \frac{\beta_1}{1-\beta_1} \left( G_\infty D_\infty + \alpha G_\infty^2 \sum_{t=1}^{T-1} \frac{1}{\sqrt{t}} + G_\infty D_\infty \sum_{t=1}^{T-1} \vert K_{t,i} - K_{t+1,i} \vert \right) \\
		&\leq \frac{\beta_1}{1-\beta_1} \left( G_\infty D_\infty + 2 \alpha G_\infty^2 \sqrt{T} + G_\infty D_\infty \sum_{t=1}^{T-1} \vert K_{t,i} - K_{t+1,i} \vert \right)
	\end{aligned}
\end{equation}

For part (C), the bias correction residual can be tightly bounded. Recall from Lemma \ref{lem.B.6} that $\vert m_{t,i} \vert \leq (1-\beta_1^t)G_\infty$. Substituting this into the expression allows us to exactly cancel the denominator:
\begin{equation}
	\begin{aligned}
		\vert \text{(C)} \vert &\leq \sum_{t=1}^{T} \vert 1-K_{t,i} \vert \frac{\beta_1^t}{1-\beta_1^t} \vert m_{t,i} \vert \vert \theta_{t,i} - \theta_{i}^* \vert \\
		&\leq \sum_{t=1}^{T} 1 \cdot \frac{\beta_1^t}{1-\beta_1^t} (1-\beta_1^t) G_\infty D_\infty \\
		&= G_\infty D_\infty \sum_{t=1}^{T} \beta_1^t \leq \frac{\beta_1}{1-\beta_1} G_\infty D_\infty
	\end{aligned}
\end{equation}

Summing parts (A), (B), and (C) over all $d$ dimensions and grouping the terms by $\sqrt{T}$ and the interpolation total variation, we obtain the highly condensed cumulative regret:
\begin{equation}
	R(T) \leq \sum_{i=1}^{d} \left[ \left( \frac{D_\infty^2}{2\alpha} + \alpha G_\infty^2 \frac{1+\beta_1}{1-\beta_1} \right) \sqrt{T} + \frac{\beta_1 G_\infty D_\infty}{1-\beta_1} \left( 2 + \sum_{t=1}^{T-1} \vert K_{t,i} - K_{t+1,i} \vert \right) \right]
\end{equation}

Given Lemma~\ref{lem.B.7} that \(\sum \vert K_{t,i} - K_{t+1,i} \vert \leq \mathcal{O}(\sqrt{T})\), the cumulative regret satisfies \(R(T) \leq \mathcal{O}(\sqrt{T})\). 

To prove convergence, we evaluate the average regret as \(T \to \infty\):
\begin{equation}
    \lim_{T \to \infty} \frac{R(T)}{T} \leq \lim_{T \to \infty} \frac{\mathcal{O}(\sqrt{T})}{T} = \lim_{T \to \infty} \mathcal{O}\left(\frac{1}{\sqrt{T}}\right) = 0
\end{equation}

\end{proof}

\newpage
\section{Convergence analysis for non-convex stochastic optimization (Theorem 3.3 in main paper).}
\label{non-convex}
We have relaxed the assumption on the objective function, allowing it to be non-convex, and adjusted the criterion for convergence from the statistic $R(T)$ to $\mathbb{E}(T)$. Let's briefly review the assumptions and the criterion for convergence after relaxing the assumption:

\begin{assumption}
\label{nonassume}
	\
\begin{itemize}
    \item A1 Bounded variables (same as convex). $\left\Vert \theta-\theta^{*}\right\Vert _{2}\leq D,\,\,\forall\theta,\theta^{*}$ or for any dimension $i$ of the variable, $\left\Vert \theta_{i}-\theta_{i}^{*}\right\Vert _{2}\leq D_{i},\,\,\forall\theta_{i},\theta_{i}^{*}$
    
    \item A2 The noisy gradient is unbiased. For $\forall t$, the random variable $\zeta_{t} $ is defined as $\zeta_{t}=g_{t}-\nabla f\left(\theta_{t}\right)$, $\zeta_{t}$ satisfy $\mathbb{E}\left[\zeta_{t}\right]=0$, $\mathbb{E}\left[\left\Vert \zeta_{t}\right\Vert _{2}^{2}\right]\leq\sigma^{2} $, and when $t_{1}\neq t_{2}$, $\zeta_{t_{1}}$ and $\zeta_{t_{2}}$ are statistically independent, i.e., $\zeta_{t_{1}} \perp \zeta_{t_{2}}$.
    
    \item A3 Bounded gradient and noisy gradient. At step $t$, the algorithm can access a bounded noisy gradient, and the true gradient is also bounded. $i.e.\ \  \vert  \vert \nabla f(\theta_t ) \vert  \vert \leq G,\  \vert  \vert g_t \vert  \vert \leq G,\ \ \forall t > 1$.
    
	\item A4 The property of function. The objective function $f\left(\theta\right)$ is a global loss function, defined as $f\left(\theta\right)=\lim_{T\longrightarrow\infty}\frac{1}{T}\sum_{t=1}^{T}f_{t}\left(\theta\right)$. Although $f\left(\theta\right)$ is no longer a convex function, it must still be a $L$-smooth function, i.e., it satisfies (1) $f$ is differentiable, $\nabla f $ exists everywhere in the domain; (2) there exists $L>0$ such that for any $\theta_{1}$ and $\theta_{2}$ in the domain, (first definition)
	\begin{equation}
	f\left(\theta_{2}\right)\leq f\left(\theta_{1}\right)+\left\langle \nabla f\left(\theta_{1}\right),\theta_{2}-\theta_{1}\right\rangle +\frac{L}{2}\left\Vert \theta_{2}-\theta_{1}\right\Vert _{2}^{2}
	\end{equation}
	or (second definition)
	\begin{equation}
	\left\Vert \nabla f\left(\theta_{1}\right)-\nabla f\left(\theta_{2}\right)\right\Vert _{2}\leq L\left\Vert \theta_{1}-\theta_{2}\right\Vert _{2}
	\end{equation}
	This condition is also known as \textit{L - Lipschitz}.
\end{itemize}
\end{assumption}

\begin{definition}
	The criterion for convergence is the statistic $\mathbb{E}\left(T\right)$:
	\begin{equation}
		\mathbb{E}\left(T\right)=\min_{t=1,2,\ldots,T}\mathbb{E}_{t}\left[\left\Vert \nabla f\left(\theta_{t}\right)\right\Vert _{2}^{2}\right]
	\end{equation}
	
	When $T\rightarrow\infty$, if the amortized value of $\mathbb{E}\left(T\right)$, $\mathbb{E}\left(T\right)/T\rightarrow0$, we consider such an algorithm to be convergent, and generally, the slower $\mathbb{E}\left(T\right)$ grows with $T$, the faster the algorithm converges.
\end{definition}

\begin{definition}
	Define $\xi_t$ as
	\begin{equation}
		\xi_t =
		\begin{cases} 
			\theta_t & t = 1\\ 
			\theta_t + \frac{\beta_{1}}{1 - \beta_{1}} \left(\theta_t - \theta_{t-1}\right) & t \geq 2 
		\end{cases}
	\end{equation}
\end{definition}

\begin{lemma}
\label{non-assume}
Let \( f \) be an \( L \)-smooth function. Then, for any points \( \xi_t \) and \( \theta_t \), the following inequality holds:
\begin{equation}
    f\left(\xi_{t+1}\right) - f\left(\xi_t\right) \leq \frac{L}{2} \left\Vert \xi_t - \theta_t \right\Vert_2^2 + L \left\Vert \xi_{t+1} - \xi_t \right\Vert_2^2 + \left\langle \nabla f\left(\theta_t\right), \xi_{t+1} - \xi_t \right\rangle
\end{equation}
\end{lemma}

\begin{proof}

Since $f$ is an L-smooth function, 
\begin{equation}
	\left\Vert \nabla f\left(\xi_t\right)-\nabla f\left(\theta_t\right)\right\Vert _{2}^{2}\leq L^{2}\left\Vert \xi_t-\theta_t\right\Vert _{2}^{2}
\end{equation}

Thus,
\begin{equation}
	\begin{aligned} 
		f\left(\xi_{t+1}\right) &- f\left(\xi_t\right) \\ 
		\leq & \left\langle \nabla f\left(\xi_t\right), \xi_{t+1} - \xi_t\right\rangle + \frac{L}{2} \left\Vert \xi_{t+1} - \xi_t\right\Vert_{2}^{2} \\ 
		= & \left\langle \frac{1}{\sqrt{L}}\left(\nabla f\left(\xi_t\right) - \nabla f\left(\theta_t\right)\right), \sqrt{L}\left(\xi_{t+1} - \xi_t\right)\right\rangle + \left\langle \nabla f\left(\theta_t\right), \xi_{t+1} - \xi_t\right\rangle + \frac{L}{2} \left\Vert \xi_{t+1} - \xi_t\right\Vert_{2}^{2} \\ 
		\leq & \frac{1}{2}\left(\frac{1}{L}\left\Vert \nabla f\left(\xi_t\right) - \nabla f\left(\theta_t\right)\right\Vert_{2}^{2} + L\left\Vert \xi_{t+1} - \xi_t\right\Vert_{2}^{2}\right) + \left\langle \nabla f\left(\theta_t\right), \xi_{t+1} - \xi_t\right\rangle + \frac{L}{2} \left\Vert \xi_{t+1} - \xi_t\right\Vert_{2}^{2} \\ 
		\leq & \frac{1}{2L}\left\Vert \nabla f\left(\xi_t\right) - \nabla f\left(\theta_t\right)\right\Vert_{2}^{2} + L\left\Vert \xi_{t+1} - \xi_t\right\Vert_{2}^{2} + \left\langle \nabla f\left(\theta_t\right), \xi_{t+1} - \xi_t\right\rangle \\ 
		\leq & \frac{1}{2L}L^{2}\left\Vert \xi_t - \theta_t\right\Vert_{2}^{2} + L\left\Vert \xi_{t+1} - \xi_t\right\Vert_{2}^{2} + \left\langle \nabla f\left(\theta_t\right), \xi_{t+1} - \xi_t\right\rangle \\ 
		= & \frac{L}{2}\underset{\left(1\right)}{\underbrace{\left\Vert \xi_t-\theta_t\right\Vert _{2}^{2}}}+L\underset{\left(2\right)}{\underbrace{\left\Vert \xi_{t+1}-\xi_t\right\Vert _{2}^{2}}}+\underset{\left(3\right)}{\underbrace{\left\langle \nabla f\left(\theta_t\right),\xi_{t+1}-\xi_t\right\rangle }}
	\end{aligned}
\end{equation}

\end{proof}

\begin{theorem}
	Consider a non-convex optimization problem. Suppose \cref{nonassume} are satisfied, and let \( \alpha_t = \alpha/\sqrt{t} \). For all \( T \geq 1 \), SGDF achieves the following guarantee:
	\begin{equation}
		\mathbb{E}(T) \leq \frac{C_{7}\alpha^2 (\log T + 1) + C_{8}}{\alpha\sqrt{T}}
	\end{equation}
	where $\mathbb{E}(T) =\min_{t=1,2,\ldots,T}\mathbb{E}_{t-1}\left[\left\Vert \nabla f\left(\theta_{t}\right)\right\Vert _{2}^{2}\right]$ denotes the minimum of the squared-paradigm expectation of the gradient, $\alpha$ is the learning rate at the $1$-th step, $C_{7}$ are constants independent of $d$ and $T$, $C_{8}$ is a constant independent of $T$, and the expectation is taken w.r.t all randomness corresponding to ${g_{t}}$.
\end{theorem}

\begin{proof}

According to \cref{non-assume}, we deal with the three terms (1), (2), and (3) separately.

\paragraph{Bounding Term (1):}

When $t=1$, $\left\Vert \xi_{t}-\theta_{t}\right\Vert _{2}^{2}=0$

When $t\geq2$,
\begin{equation}
\begin{aligned}
	& \left\Vert \xi_{t}-\theta_{t}\right\Vert _{2}^{2}=\left\Vert \frac{\beta_{1}}{1-\beta_{1}}\left(\theta_{t}-\theta_{t-1}\right)\right\Vert _{2}^{2}\\
	& =\frac{\beta_{1}^{2}}{\left(1-\beta_{1}\right)^{2}}\alpha_{t-1}^{2}\left\Vert \hat{g}_{t-1}\right\Vert _{2}^{2}\\
	&=\frac{\beta_{1}^{2}}{\left(1-\beta_{1}\right)^{2}}\alpha_{t-1}^{2}\sum_{i=1}^{d}\left( \left(1-K_{t-1,i}\right)\left(\widehat{m}_{t-1,i}\right)^{2}+K_{t-1,i} g_{t-1,i}^2 \right)\\
	&\overset{\left(a\right)}{\leq}\frac{\beta_{1}^{2}}{\left(1-\beta_{1}\right)^{2}}\alpha_{t-1}^{2}\sum_{i=1}^{d}G_{i}^{2}
\end{aligned}
\end{equation}

Where (a) holds because for any $t$:
\begin{enumerate}
	\item \( \left|\widehat{m}_{t,i}\right|\leq\frac{1}{1-\beta_{1}^{t}}\sum_{s=1}^{t}\left(1-\beta_{1}\right)\beta_{1}^{t-s}\left|g_{s,i}\right|\leq\frac{1}{1-\beta_{1}^{t}}\sum_{s=1}^{t}\left(1-\beta_{1}\right)\beta_{1}^{t-s}G_{i}=G_{i} 
	\).
	\item $\| g_{t} \|_2 \leq G,\, \forall t$, or for any dimension of the variable $i$: $\| g_{t,i} \|_2 \leq G_{i},\, \forall t$
\end{enumerate}

\paragraph{Bounding Term (2):}

For the initial iteration $t=1$, we have:
\begin{equation}
\begin{aligned}
	\xi_{2}-\xi_1=&\theta_{2}+\frac{\beta_{1}}{1-\beta_{1}}\left(\theta_{2}-\theta_1\right)-\theta_1\\ 
	= & \frac{1}{1-\beta_{1}}\left(\theta_{2}-\theta_1\right)\\ 
	= & -\frac{\alpha_1}{1-\beta_{1}}\left(\hat{g}_1\right)\\ 
	= & -\frac{\alpha_1}{1-\beta_{1}} \left(\frac{1-K_1}{1-\beta_{1}}m_1+K_{1} g_{1}\right)\\ 
	= &-\frac{\alpha_1}{1-\beta_{1}}\frac{1-K_1}{1-\beta_{1}}\left(\beta_{1}\cancelto{0}{m_{0}}+\left(1-\beta_{1}\right)g_{1}\right)-\frac{\alpha_1}{1-\beta_{1}}K_{1} g_{1}\\ 
	= & -\frac{\alpha_1\left(1-K_1\right)}{1-\beta_{1}}g_{1} -\frac{\alpha_1 K_{1}}{1-\beta_{1}} g_{1}\\
	= & -\frac{\alpha_1}{1-\beta_{1}}g_{1} 
\end{aligned}
\end{equation}

Consequently, the squared $\ell_2$-norm can be bounded as follows:
\begin{equation}
\begin{aligned}
	\left\Vert \xi_{2}-\xi_{1}\right\Vert _{2}^{2}= & \left\Vert -\frac{\alpha_{1}}{1-\beta_{1}}g_{1}\right\Vert _{2}^{2}\\ 
	= & \frac{\alpha_{1}^2}{(1-\beta_{1})^2} \|g_{1}\|_2^2\\
	= & \frac{\alpha_{1}^{2}}{\left(1-\beta_{1}\right)^{2}}\sum_{i=1}^{d}g_{1,i}^{2}\\ 
	\leq & \frac{\alpha_{1}^{2}}{\left(1-\beta_{1}\right)^{2}}\sum_{i=1}^{d}G_{i}^{2}
\end{aligned}
\end{equation}

For subsequent iterations $t\geq2$, the difference between consecutive auxiliary variables expands as:
\begin{equation}
\begin{aligned}
	\xi_{t+1}-\xi_t=&\theta_{t+1}+\frac{\beta_{1}}{1-\beta_{1}}\left(\theta_{t+1}-\theta_t\right) -\theta_t-\frac{\beta_{1}}{1-\beta_{1}}\left(\theta_t-\theta_{t-1}\right)\\ 
	=&\frac{1}{1-\beta_{1}}\left(\theta_{t+1}-\theta_t\right)-\frac{\beta_{1}}{1-\beta_{1}}\left(\theta_t-\theta_{t-1}\right) 
\end{aligned}
\end{equation}

Recalling the parameter update rule, the difference $\theta_{t+1}-\theta_t$ is given by:
\begin{equation}
\begin{aligned}
	\theta_{t+1}-\theta_t= & -\alpha_t\hat{g}_t\\ 
	=& -\frac{\alpha_t(1-K_{t})}{1-\beta_{1}^{t}}m_t - \alpha_t K_{t}g_{t} \\ 
	=&-\frac{\alpha_t(1-K_{t})}{1-\beta_{1}^{t}}\left(\beta_{1}m_{t-1}+\left(1-\beta_{1}\right)g_{t}\right)- \alpha_t K_{t}g_{t}
\end{aligned}
\end{equation}

Substituting this expression back into the expansion of $\xi_{t+1}-\xi_t$ and rearranging the terms, we obtain:
\begin{equation}
\begin{aligned} 
	& \xi_{t+1}-\xi_t\\ 
	= & \frac{1}{1-\beta_{1}}\left(-\frac{\alpha_t(1-K_{t})}{1-\beta_{1}^{t}}\left(\beta_{1}m_{t-1}+\left(1-\beta_{1}\right)g_{t}\right)- \alpha_t K_{t}g_{t}\right) \\
	&-\frac{\beta_{1}}{1-\beta_{1}}\left(-\frac{\alpha_{t-1}(1-K_{t-1})}{1-\beta_{1}^{t-1}}m_{t-1}- \alpha_{t-1}K_{t-1}g_{t-1}\right)\\ 
	=&-\frac{\beta_{1}}{1-\beta_{1}}m_{t-1}\odot\left(\frac{\alpha_t(1-K_{t})}{1-\beta_{1}^{t}}-\frac{\alpha_{t-1}(1-K_{t-1})}{1-\beta_{1}^{t-1}}\right) \\
	&-\left(\frac{\alpha_t(1-K_{t})}{1-\beta_{1}^{t}} + \frac{\alpha_t K_{t}}{1-\beta_{1}}\right)g_{t}  + \frac{\beta_{1}\alpha_{t-1}K_{t-1}}{1-\beta_{1}}g_{t-1}
\end{aligned}
\end{equation}

Using the general inequality $\Vert A + B + C \Vert_2^2 \leq 3\Vert A \Vert_2^2 + 3\Vert B \Vert_2^2 + 3\Vert C \Vert_2^2$, we have:
\begin{equation}
\begin{aligned} 
	\left\Vert \xi_{t+1}-\xi_t\right\Vert _{2}^{2}&\leq 3\left\Vert -\frac{\beta_{1}}{1-\beta_{1}}m_{t-1}\odot\left(\frac{\alpha_t(1-K_{t})}{1-\beta_{1}^{t}} -\frac{\alpha_{t-1}(1-K_{t-1})}{1-\beta_{1}^{t-1}}\right)\right\Vert _{2}^{2}\\
	&+3\left\Vert -\left(\frac{\alpha_t(1-K_{t})}{1-\beta_{1}^{t}} + \frac{\alpha_tK_{t}}{1-\beta_{1}}\right)g_{t}\right\Vert _{2}^{2} + 3\left\Vert \frac{\beta_{1}\alpha_{t-1}K_{t-1}}{1-\beta_{1}}g_{t-1} \right\Vert_{2}^{2}
\end{aligned}
\end{equation}

To properly decouple the step size decay from the dynamic mask flipping, let $\eta_t = \frac{\alpha_t}{1-\beta_1^t}$. We can decompose the mask difference algebraically by adding and subtracting $\eta_{t-1}K_{t}$:
\begin{equation}
\begin{aligned}
	\eta_t(1-K_t) - \eta_{t-1}(1-K_{t-1}) 
	&= \eta_t - \eta_t K_t - \eta_{t-1} + \eta_{t-1} K_{t-1} \\
	&= (\eta_t - \eta_{t-1}) - \eta_t K_t + \eta_{t-1} K_t - \eta_{t-1} K_t + \eta_{t-1} K_{t-1} \\
	&= (\eta_t - \eta_{t-1})(1-K_t) + \eta_{t-1}(K_{t-1} - K_t)
\end{aligned}
\end{equation}

Using the inequality $\Vert A + B \Vert_2^2 \leq 2\Vert A \Vert_2^2 + 2\Vert B \Vert_2^2$, we bound the squared $\ell_2$ norm of this difference. Since $K_{t,i} \in \{0,1\}$, we have $(1-K_{t,i})^2 \leq 1$ and $(K_{t-1,i} - K_{t,i})^2 = |K_{t-1,i} - K_{t,i}|$:
\begin{equation}
\begin{aligned}
	\left\Vert \frac{\alpha_t(1-K_{t})}{1-\beta_{1}^{t}} -\frac{\alpha_{t-1}(1-K_{t-1})}{1-\beta_{1}^{t-1}} \right\Vert_2^2 
	&= \sum_{i=1}^d \left( (\eta_{t} - \eta_{t-1})(1-K_{t,i}) + \eta_{t-1}(K_{t-1,i} - K_{t,i}) \right)^2 \\
	&\leq 2 \sum_{i=1}^d (\eta_{t-1} - \eta_t)^2 (1-K_{t,i})^2 + 2 \sum_{i=1}^d \eta_{t-1}^2 (K_{t-1,i} - K_{t,i})^2 \\
	&\leq 2 \sum_{i=1}^d (\eta_{t-1} - \eta_t)^2 + 2\eta_{t-1}^2 \sum_{i=1}^d |K_{t-1,i} - K_{t,i}| 
\end{aligned}
\end{equation}

Since $\eta_t$ is monotonically decreasing, $\eta_{t-1} - \eta_t \geq 0$, and thus $(\eta_{t-1} - \eta_t)^2 \leq \eta_1(\eta_{t-1} - \eta_t)$. Let $F_t = \sum_{i=1}^d |K_{t-1,i} - K_{t,i}|$ denote the total number of flipped mask bits at step $t$. To maintain a tight bound, we retain $F_t$:
\begin{equation}
\begin{aligned}
	\left\Vert \frac{\alpha_t(1-K_{t})}{1-\beta_{1}^{t}} -\frac{\alpha_{t-1}(1-K_{t-1})}{1-\beta_{1}^{t-1}} \right\Vert_2^2
	&\leq 2d\eta_1 (\eta_{t-1} - \eta_t) + 2\eta_{t-1}^2 F_t \\
	&\leq 2d\eta_1 \left(\frac{\alpha_{t-1}}{1-\beta_1^{t-1}} - \frac{\alpha_t}{1-\beta_1^t}\right) + 2F_t\left(\frac{\alpha_{t-1}}{1-\beta_1^{t-1}}\right)^2
\end{aligned}
\end{equation}

Now, bounding the three terms of $\left\Vert \xi_{t+1}-\xi_t\right\Vert _{2}^{2}$ respectively. For the first term, since $|m_{t-1,i}| \leq G_i$:
\begin{equation}
\begin{aligned}
	&3\left\Vert \frac{\beta_{1}}{1-\beta_{1}}m_{t-1}\odot\left(\frac{\alpha_t(1-K_{t})}{1-\beta_{1}^{t}}-\frac{\alpha_{t-1}(1-K_{t-1})}{1-\beta_{1}^{t-1}}\right)\right\Vert _{2}^{2}\\
	&\leq 3\frac{\beta_1^2}{(1-\beta_1)^2} \left(\max_i G_i\right)^2 \left( 2d\eta_1 \left(\frac{\alpha_{t-1}}{1-\beta_1^{t-1}} - \frac{\alpha_t}{1-\beta_1^t}\right) + 2F_t\left(\frac{\alpha_{t-1}}{1-\beta_1^{t-1}}\right)^2 \right) \\
	&= 6d\frac{\beta_1^2 \eta_1}{(1-\beta_1)^2} \left(\max_i G_i\right)^2 \left(\frac{\alpha_{t-1}}{1-\beta_1^{t-1}} - \frac{\alpha_t}{1-\beta_1^t}\right)  + 6F_t\frac{\beta_1^2}{(1-\beta_1)^2} \left(\max_i G_i\right)^2 \left(\frac{\alpha_{t-1}}{1-\beta_1^{t-1}}\right)^2
\end{aligned}
\end{equation}

For the second term, notice that $1-\beta_1 \leq 1-\beta_1^t$, which implies $\frac{1}{1-\beta_1} \geq \frac{1}{1-\beta_1^t}$. Since the dimensions masked by $K_t$ and $1-K_t$ are completely disjoint, we safely bound the disjoint coefficients using the globally larger denominator $\frac{1}{1-\beta_1}$:
\begin{equation}
\begin{aligned}
	3\left\Vert \left(\frac{\alpha_t(1-K_{t})}{1-\beta_{1}^{t}} + \frac{\alpha_tK_{t}}{1-\beta_{1}}\right)g_{t}\right\Vert _{2}^{2}
	&= 3 \sum_{i=1}^d \left( \frac{\alpha_t(1-K_{t,i})}{1-\beta_1^t} + \frac{\alpha_t K_{t,i}}{1-\beta_1} \right)^2 g_{t,i}^2 \\
	&\leq 3 \left( \frac{\alpha_t}{1-\beta_1} \right)^2 \sum_{i=1}^d G_i^2
\end{aligned}
\end{equation}

For the third term, we maintain the exact $\beta_1$ constant multiplier to prevent invalid bounding:
\begin{equation}
\begin{aligned}
	3\left\Vert \frac{\beta_{1}\alpha_{t-1}K_{t-1}}{1-\beta_{1}}g_{t-1} \right\Vert_{2}^{2}
	&\leq 3 \frac{\beta_1^2 \alpha_{t-1}^2}{(1-\beta_1)^2} \sum_{i=1}^d G_i^2
\end{aligned}
\end{equation}

Therefore, bringing everything together, we obtain the corrected upper bound:
\begin{equation}
\begin{aligned} 
	 \left\Vert \xi_{t+1}-\xi_{t}\right\Vert _{2}^{2}
	\leq & \quad 6d\frac{\beta_1^2 \eta_1}{(1-\beta_1)^2} \left(\max_i G_i\right)^2 \left(\frac{\alpha_{t-1}}{1-\beta_1^{t-1}} - \frac{\alpha_t}{1-\beta_1^t}\right) \\
	&+ 6F_t\frac{\beta_1^2}{(1-\beta_1)^2} \left(\max_i G_i\right)^2 \left(\frac{\alpha_{t-1}}{1-\beta_1^{t-1}}\right)^2 \\
	&+ 3 \left(\frac{\alpha_t}{1-\beta_1}\right)^2 \sum_{i=1}^d G_i^2 + 3 \frac{\beta_1^2 \alpha_{t-1}^2}{(1-\beta_1)^2} \sum_{i=1}^d G_i^2
\end{aligned}
\end{equation}

\paragraph{Bounding Term (3):}

When $t=1$, referring to the case of $t=1$ in the previous subsection, we expand the inner product:
\begin{equation}
\begin{aligned} 
	 \left\langle \nabla f\left(\theta_1\right),\xi_{2}-\xi_1\right\rangle 
	= & \left\langle \nabla f\left(\theta_1\right),-\frac{\alpha_1}{1-\beta_{1}}g_1\right\rangle \\ 
	= & \left\langle \nabla f\left(\theta_1\right),-\frac{\alpha_1}{1-\beta_{1}}\nabla f\left(\theta_1\right)\right\rangle  +\left\langle \nabla f\left(\theta_1\right),-\frac{\alpha_1}{1-\beta_{1}}\zeta_1\right\rangle \\
	= & -\frac{\alpha_1}{1-\beta_{1}}\left\Vert\nabla f\left(\theta_1\right)\right\Vert _{2}^{2} +\left\langle \nabla f\left(\theta_1\right),-\frac{\alpha_1}{1-\beta_{1}}\zeta_1\right\rangle 
\end{aligned}
\end{equation}

Note that we leave the inner product with the zero-mean noise $\zeta_1$ exactly as it is, because it will vanish when taking the conditional expectation later.

When $t\geq2$,
\begin{equation}
\begin{aligned} 
	& \left\langle \nabla f\left(\theta_t\right),\xi_{t+1}-\xi_t\right\rangle \\
	= & \left\langle \nabla f\left(\theta_t\right),-\frac{\beta_{1}}{1-\beta_{1}}m_{t-1}\odot\left(\frac{\alpha_t(1-K_{t})}{1-\beta_{1}^{t}}-\frac{\alpha_{t-1}(1-K_{t-1})}{1-\beta_{1}^{t-1}}\right)\right\rangle \\ 
	& +\left\langle \nabla f\left(\theta_t\right),-\left(\frac{\alpha_t(1-K_{t})}{1-\beta_{1}^{t}} + \frac{\alpha_tK_{t}}{1-\beta_{1}}\right)\nabla f\left(\theta_t\right)\right\rangle +\left\langle \nabla f\left(\theta_t\right),-\left(\frac{\alpha_t(1-K_{t})}{1-\beta_{1}^{t}} + \frac{\alpha_tK_{t}}{1-\beta_{1}}\right)\zeta_t\right\rangle \\
	& +\left\langle \nabla f\left(\theta_{t}\right),\frac{\beta_{1}\alpha_{t-1}K_{t-1}}{1-\beta_{1}}\nabla f\left(\theta_{t-1}\right)\right\rangle +\left\langle \nabla f\left(\theta_{t}\right),\frac{\beta_{1}\alpha_{t-1}K_{t-1}}{1-\beta_{1}}\zeta_{t-1}\right\rangle
\end{aligned}
\end{equation}

Let us deal with the terms after the equal sign separately. 

Start by looking at the first term. We decouple the step size decay and dynamic mask flipping using the identity $\eta_t(1-K_t) - \eta_{t-1}(1-K_{t-1}) = (\eta_t - \eta_{t-1})(1-K_t) + \eta_{t-1}(K_{t-1} - K_t)$ and Hölder's inequality ($|\langle A, B \rangle| \leq \Vert A \Vert_\infty \Vert B \Vert_1$):
\begin{equation}
\begin{aligned} 
	& \left\langle \nabla f\left(\theta_t\right),-\frac{\beta_{1}}{1-\beta_{1}}m_{t-1}\odot\left(\frac{\alpha_t(1-K_{t})}{1-\beta_{1}^{t}}-\frac{\alpha_{t-1}(1-K_{t-1})}{1-\beta_{1}^{t-1}}\right)\right\rangle \\ 
	\leq & \frac{\beta_{1}}{1-\beta_{1}}\left\Vert \nabla f\left(\theta_t\right)\right\Vert _{\infty}\left\Vert m_{t-1}\right\Vert _{\infty} \left\Vert \frac{\alpha_t(1-K_{t})}{1-\beta_{1}^{t}}-\frac{\alpha_{t-1}(1-K_{t-1})}{1-\beta_{1}^{t-1}}\right\Vert _{1}\\ 
	\leq & \frac{\beta_{1}}{1-\beta_{1}}\left(\max_{i}G_{i}\right)^2 \left( \sum_{i=1}^{d}\left(\frac{\alpha_{t-1}}{1-\beta_{1}^{t-1}}-\frac{\alpha_{t}}{1-\beta_{1}^{t}}\right)(1-K_{t,i}) + \sum_{i=1}^d \frac{\alpha_{t-1}}{1-\beta_1^{t-1}}|K_{t-1,i} - K_{t,i}| \right)\\
	\leq & \frac{\beta_{1}}{1-\beta_{1}}\left(\max_{i}G_{i}\right)^2 \cdot d\left(\frac{\alpha_{t-1}}{1-\beta_{1}^{t-1}}-\frac{\alpha_{t}}{1-\beta_{1}^{t}}\right) + \frac{\beta_{1}}{1-\beta_{1}}\left(\max_{i}G_{i}\right)^2 \frac{\alpha_{t-1}}{1-\beta_1^{t-1}} F_t
\end{aligned}
\end{equation}
where $F_t = \sum_{i=1}^d |K_{t-1,i} - K_{t,i}|$ represents the number of indices where the mask changes at step $t$.

For the second and third terms, notice that since $1-\beta_1^t \geq 1-\beta_1$, we have $\frac{\alpha_t(1-K_t)}{1-\beta_1^t} + \frac{\alpha_t K_t}{1-\beta_1} \geq \frac{\alpha_t}{1-\beta_1^t}(1-K_t+K_t) = \frac{\alpha_t}{1-\beta_1^t}$:
\begin{equation}
\begin{aligned}
	&\left\langle \nabla f\left(\theta_t\right),-\left(\frac{\alpha_t(1-K_{t})}{1-\beta_{1}^{t}} + \frac{\alpha_tK_{t}}{1-\beta_{1}}\right)\nabla f\left(\theta_t\right)\right\rangle +\left\langle \nabla f\left(\theta_t\right),-\left(\frac{\alpha_t(1-K_{t})}{1-\beta_{1}^{t}} + \frac{\alpha_tK_{t}}{1-\beta_{1}}\right)\zeta_t\right\rangle\\
	\leq &-\frac{\alpha_t}{1-\beta_{1}^{t}} \left\Vert \nabla f\left(\theta_t\right)\right\Vert _{2}^{2} +\left\langle \nabla f\left(\theta_t\right), -\left(\frac{\alpha_t(1-K_{t})}{1-\beta_{1}^{t}} + \frac{\alpha_tK_{t}}{1-\beta_{1}}\right)\zeta_t\right\rangle
\end{aligned}
\end{equation}

For the fourth term (the cross-gradient term), applying Hölder's inequality directly would yield an un-decaying $\mathcal{O}(\alpha_{t-1})$ penalty. Instead, we use the basic inequality $2\langle a, b \rangle \leq \|a\|_2^2 + \|b\|_2^2$. Since $K_{t-1,i} \in \{0, 1\}$, we have:
\begin{equation}
\begin{aligned} 
	\left\langle \nabla f\left(\theta_{t}\right),\frac{\beta_{1}\alpha_{t-1}K_{t-1}}{1-\beta_{1}}\nabla f\left(\theta_{t-1}\right)\right\rangle 
	&= \frac{\beta_1\alpha_{t-1}}{1-\beta_1} \sum_{i=1}^d K_{t-1,i} \nabla f(\theta_t)_i \nabla f(\theta_{t-1})_i \\
	&\leq \frac{\beta_1\alpha_{t-1}}{2(1-\beta_1)} \sum_{i=1}^d K_{t-1,i} \left( \nabla f(\theta_t)_i^2 + \nabla f(\theta_{t-1})_i^2 \right) \\
	&\leq \frac{\beta_{1}\alpha_{t-1}}{2(1-\beta_{1})} \left( \left\Vert \nabla f\left(\theta_t\right)\right\Vert _{2}^{2} + \left\Vert \nabla f\left(\theta_{t-1}\right)\right\Vert _{2}^{2} \right)
\end{aligned}
\end{equation}

For the fifth term (the noise term involving $\zeta_{t-1}$), taking the absolute value would similarly result in an $\mathcal{O}(\alpha_{t-1})$ error bound. To properly bound it, we decompose the inner product to leverage the expectation later:
\begin{equation}
\begin{aligned} 
	\left\langle \nabla f\left(\theta_{t}\right),\frac{\beta_{1}\alpha_{t-1}K_{t-1}}{1-\beta_{1}}\zeta_{t-1}\right\rangle 
	&= \frac{\beta_1\alpha_{t-1}}{1-\beta_1} \left\langle \nabla f(\theta_{t-1}), K_{t-1} \odot \zeta_{t-1} \right\rangle \\
	&\quad + \frac{\beta_1\alpha_{t-1}}{1-\beta_1} \left\langle \nabla f(\theta_t) - \nabla f(\theta_{t-1}), K_{t-1} \odot \zeta_{t-1} \right\rangle
\end{aligned}
\end{equation}

Notice that when taking the expectation, the first part satisfies $\mathbb{E}_{t-2}\left[\langle \nabla f(\theta_{t-1}), K_{t-1} \odot \zeta_{t-1} \rangle\right] = 0$ because both $\theta_{t-1}$ and $K_{t-1}$ are independent of the noise realization $\zeta_{t-1}$. For the second part, we apply the Cauchy-Schwarz inequality and utilize the $L$-smoothness of the objective function. Given that $\|\theta_t - \theta_{t-1}\|_2 \leq \alpha_{t-1}\|\hat{g}_{t-1}\|_2 \leq \alpha_{t-1}\sqrt{\sum_{i=1}^d G_i^2}$ and $\|\zeta_{t-1}\|_2 \leq 2\sqrt{\sum_{i=1}^d G_i^2}$, we can bound this term strictly by $\mathcal{O}(\alpha_{t-1}^2)$:
\begin{equation}
\begin{aligned} 
	\frac{\beta_1\alpha_{t-1}}{1-\beta_1} \left\langle \nabla f(\theta_t) - \nabla f(\theta_{t-1}), K_{t-1} \odot \zeta_{t-1} \right\rangle 
	&\leq \frac{\beta_1\alpha_{t-1}}{1-\beta_1} \left\Vert \nabla f(\theta_t) - \nabla f(\theta_{t-1}) \right\Vert_2 \left\Vert \zeta_{t-1} \right\Vert_2 \\
	&\leq \frac{\beta_1\alpha_{t-1}}{1-\beta_1} L \left\Vert \theta_t - \theta_{t-1} \right\Vert_2 \left\Vert \zeta_{t-1} \right\Vert_2 \\
	&\leq \frac{\beta_1\alpha_{t-1}}{1-\beta_1} L \left( \alpha_{t-1} \sqrt{\sum_{i=1}^d G_i^2} \right) \left( 2\sqrt{\sum_{i=1}^d G_i^2} \right) \\
	&= \frac{2L\beta_1}{1-\beta_1} \alpha_{t-1}^2 \sum_{i=1}^d G_i^2
\end{aligned}
\end{equation}

Finally, combining everything together:
\begin{equation}
\begin{aligned} 
	& \left\langle \nabla f\left(\theta_{t}\right),\xi_{t+1}-\xi_{t}\right\rangle \\ 
	\leq & \quad \frac{\beta_{1}d}{1-\beta_{1}}\left(\max_{i}G_{i}\right)^2 \left(\frac{\alpha_{t-1}}{1-\beta_{1}^{t-1}}-\frac{\alpha_{t}}{1-\beta_{1}^{t}}\right) 
	+ \frac{\beta_{1}}{1-\beta_{1}}\left(\max_{i}G_{i}\right)^2 \frac{\alpha_{t-1}}{1-\beta_1^{t-1}} F_t \\
	& - \left( \frac{\alpha_{t}}{1-\beta_{1}^{t}} - \frac{\beta_1\alpha_{t-1}}{2(1-\beta_1)} \right) \left\Vert \nabla f\left(\theta_{t}\right)\right\Vert _{2}^{2} + \frac{\beta_1\alpha_{t-1}}{2(1-\beta_1)} \left\Vert \nabla f(\theta_{t-1}) \right\Vert_2^2 \\
	& + \left\langle \nabla f\left(\theta_{t}\right),-\left(\frac{\alpha_t(1-K_{t})}{1-\beta_{1}^{t}} + \frac{\alpha_tK_{t}}{1-\beta_{1}}\right)\zeta_{t}\right\rangle \\
	& + \frac{\beta_1\alpha_{t-1}}{1-\beta_1} \left\langle \nabla f(\theta_{t-1}), K_{t-1} \odot \zeta_{t-1} \right\rangle + \frac{2L\beta_1}{1-\beta_1} \alpha_{t-1}^2 \sum_{i=1}^d G_i^2
\end{aligned}
\end{equation}

\paragraph{Summarizing the results}

Since we have already rigorously bounded the cross-inner product terms and the noise components in the previous subsection, we can proceed directly to summarizing the results. 

First, taking the expectation $\mathbb{E}_t$ over the random distribution of $\zeta_{1},\zeta_{2},\ldots,\zeta_{t}$ on both sides of the inequality. Since the value of $\theta_{t}$ is independent of $g_{t}$, they are statistically independent of $\zeta_{t}$, and $\mathbb{E}_{t}\left[\zeta_{t}\right]=0$. Therefore, all inner products with $\zeta_t$ perfectly vanish:
\begin{equation}
\begin{aligned} 
	& \mathbb{E}_{t}\left[\left\langle \nabla f\left(\theta_{t}\right),-\left(\frac{\alpha_t(1-K_{t})}{1-\beta_{1}^{t}} + \frac{\alpha_tK_{t}}{1-\beta_{1}}\right)\zeta_{t}\right\rangle \right] \\ 
	= & \left\langle -\left(\frac{\alpha_t(1-K_{t})}{1-\beta_{1}^{t}} + \frac{\alpha_tK_{t}}{1-\beta_{1}}\right)\nabla f\left(\theta_{t}\right),\cancelto{0}{\mathbb{E}_{t}\left[\zeta_{t}\right]}\right\rangle = 0 
\end{aligned}
\end{equation}

Combining the bounds from Term (1), Term (2), and the refined Term (3), for $t \geq 2$, we have:
\begin{equation}
\begin{aligned} 
	& \mathbb{E}_{t}\left[f\left(\xi_{t+1}\right)-f\left(\xi_t\right)\right] \\ 
	\leq & \frac{L\beta_1^2}{2(1-\beta_1)^2} \alpha_{t-1}^2 \sum_{i=1}^d G_i^2 + 6dL\frac{\beta_1^2 \eta_1}{(1-\beta_1)^2} (\max_i G_i)^2 \left(\frac{\alpha_{t-1}}{1-\beta_1^{t-1}} - \frac{\alpha_t}{1-\beta_1^t}\right) \\
	& + \left( 6F_t\frac{\beta_1^2}{(1-\beta_1)^2} (\max_i G_i)^2 + 3L \sum_{i=1}^d G_i^2 \right) \left(\frac{\alpha_{t-1}}{1-\beta_1^{t-1}}\right)^2 + 3L \left(\frac{\alpha_t}{1-\beta_1}\right)^2 \sum_{i=1}^d G_i^2 \\
	& + \frac{\beta_1 d}{1-\beta_1} (\max_i G_i)^2 \left(\frac{\alpha_{t-1}}{1-\beta_1^{t-1}} - \frac{\alpha_t}{1-\beta_1^t}\right) + \frac{\beta_1}{(1-\beta_1)^2}(\max_i G_i)^2 \left(\frac{\alpha_{t-1}}{1-\beta_1^{t-1}}\right) F_t \\
	& - \left( \frac{\alpha_t}{1-\beta_1^t} - \frac{\beta_1\alpha_{t-1}}{2(1-\beta_1)} \right) \mathbb{E}_t\left\Vert \nabla f(\theta_t) \right\Vert_2^2 + \frac{\beta_1\alpha_{t-1}}{2(1-\beta_1)} \mathbb{E}_{t-1}\left\Vert \nabla f(\theta_{t-1}) \right\Vert_2^2 \\
	& + \frac{2L\beta_1}{1-\beta_1} \alpha_{t-1}^2 \sum_{i=1}^d G_i^2
\end{aligned}
\end{equation}

To maintain the inequality and simplify the notation, we define the following iteration-independent constants. Note that $\frac{1}{1-\beta_1^t} \leq \frac{1}{1-\beta_1}$, and we absorb the new noise bound into $C_1$:

\begin{enumerate}
	\item For $\alpha_{t-1}^2$: $C_1 \triangleq \frac{L\beta_1^2}{2(1-\beta_1)^2}\sum_{i=1}^d G_i^2 + \frac{3L}{(1-\beta_1)^2}\sum_{i=1}^d G_i^2 + \frac{2L\beta_1}{1-\beta_1}\sum_{i=1}^d G_i^2$
	\item For $\alpha_t^2$: $C_2 \triangleq \frac{3L}{(1-\beta_1)^2}\sum_{i=1}^d G_i^2$
	\item For the telescoping difference: $C_3 \triangleq \frac{6dL\beta_1^2 \alpha_1}{(1-\beta_1)^3} (\max_{i}G_i)^2 + \frac{\beta_1 d}{1-\beta_1} (\max_{i}G_i)^2$
	\item For the dynamic mask flipping terms (both linear and squared): $C_4 \triangleq \frac{\beta_1}{(1-\beta_1)^2}(\max_i G_i)^2 + \frac{6\beta_1^2}{(1-\beta_1)^4}(\max_i G_i)^2 \alpha_1$
\end{enumerate}

For $t=1$:
\begin{equation}
\begin{aligned} 
	\mathbb{E}_1\left[f(\xi_2) - f(\xi_1)\right] 
	\leq & \frac{L}{(1-\beta_1)^2}\alpha_1^2 \sum_{i=1}^d G_i^2 - \frac{\alpha_1}{1-\beta_1} \mathbb{E}_1\left\Vert \nabla f(\theta_1) \right\Vert_2^2
\end{aligned}
\end{equation}

Summing up both sides of the inequality for $t=1,2,\ldots, T$:

\textbf{Left side of the inequality (LHS)}
\begin{equation}
\begin{aligned}
	\sum_{t=1}^{T}\mathbb{E}_{t}\left[f\left(\xi_{t+1}\right)-f\left(\xi_t\right)\right]
	= & \mathbb{E}_{T}\left[f\left(\xi_{T+1}\right)\right]-\mathbb{E}_{0}\left[f\left(\xi_1\right)\right] \\
	\geq & f\left(\theta^{*}\right)-f\left(\theta_1\right) 
\end{aligned}
\end{equation}

\textbf{Right side of the inequality (RHS)}
\begin{equation}
\begin{aligned} 
	\sum_{t=1}^{T} \mathbb{E}_t[\text{RHS}] 
	\leq & \sum_{t=2}^T C_1 \alpha_{t-1}^2 + \sum_{t=2}^T C_2 \alpha_t^2 + \sum_{t=2}^T C_3 \left(\frac{\alpha_{t-1}}{1-\beta_1^{t-1}} - \frac{\alpha_t}{1-\beta_1^t}\right) + \sum_{t=2}^T C_4 \frac{\alpha_{t-1}}{1-\beta_1^{t-1}} F_t \\
	& + \sum_{t=2}^T \frac{\beta_1 \alpha_{t-1}}{2(1-\beta_1)} \mathbb{E}_{t-1}\left\Vert \nabla f(\theta_{t-1}) \right\Vert_2^2 - \sum_{t=2}^T \left( \frac{\alpha_t}{1-\beta_1^t} - \frac{\beta_1\alpha_{t-1}}{2(1-\beta_1)} \right) \mathbb{E}_t\left\Vert \nabla f(\theta_t) \right\Vert_2^2 \\
	& + \frac{L}{(1-\beta_1)^2}\alpha_1^2 \sum_{i=1}^d G_i^2 - \frac{\alpha_1}{1-\beta_1} \mathbb{E}_1\left\Vert \nabla f(\theta_1) \right\Vert_2^2
\end{aligned}
\end{equation}

We rigorously shift the index of the positive gradient expectation term to combine it with the negative counterpart. By shifting $k = t-1$:
\begin{equation}
\begin{aligned}
	& \sum_{t=2}^T \frac{\beta_1 \alpha_{t-1}}{2(1-\beta_1)} \mathbb{E}_{t-1}\left\Vert \nabla f(\theta_{t-1}) \right\Vert_2^2 - \sum_{t=2}^T \left( \frac{\alpha_t}{1-\beta_1^t} - \frac{\beta_1\alpha_{t-1}}{2(1-\beta_1)} \right) \mathbb{E}_t\left\Vert \nabla f(\theta_t) \right\Vert_2^2 \\
	= & \frac{\beta_1 \alpha_1}{2(1-\beta_1)} \mathbb{E}_1\left\Vert \nabla f(\theta_1) \right\Vert_2^2 + \sum_{t=2}^{T-1} \frac{\beta_1 \alpha_t}{2(1-\beta_1)} \mathbb{E}_t\left\Vert \nabla f(\theta_t) \right\Vert_2^2 \\
	& - \sum_{t=2}^{T-1} \left( \frac{\alpha_t}{1-\beta_1^t} - \frac{\beta_1\alpha_{t-1}}{2(1-\beta_1)} \right) \mathbb{E}_t\left\Vert \nabla f(\theta_t) \right\Vert_2^2 - \left( \frac{\alpha_T}{1-\beta_1^T} - \frac{\beta_1\alpha_{T-1}}{2(1-\beta_1)} \right) \mathbb{E}_T\left\Vert \nabla f(\theta_T) \right\Vert_2^2 \\
	\leq & \frac{\beta_1 \alpha_1}{2(1-\beta_1)} \mathbb{E}_1\left\Vert \nabla f(\theta_1) \right\Vert_2^2 - \sum_{t=2}^T \left( \frac{\alpha_t}{1-\beta_1^t} - \frac{\beta_1(\alpha_{t-1} + \alpha_t)}{2(1-\beta_1)} \right) \mathbb{E}_t\left\Vert \nabla f(\theta_t) \right\Vert_2^2
\end{aligned}
\end{equation}
where the inequality holds because we drop the strictly negative term $-\frac{\beta_1 \alpha_T}{2(1-\beta_1)}\mathbb{E}_T\Vert \nabla f(\theta_T) \Vert_2^2$ to upper-bound the expression. 

Now, integrating the $t=1$ expectation term from the initialization, we observe that since $\beta_1 < 1$, the combined coefficient for $t=1$ is strictly negative:
\begin{equation}
\begin{aligned}
	- \frac{\alpha_1}{1-\beta_1} \mathbb{E}_1\left\Vert \nabla f(\theta_1) \right\Vert_2^2 + \frac{\beta_1 \alpha_1}{2(1-\beta_1)} \mathbb{E}_1\left\Vert \nabla f(\theta_1) \right\Vert_2^2 
	= - \frac{\alpha_1(2 - \beta_1)}{2(1-\beta_1)} \mathbb{E}_1\left\Vert \nabla f(\theta_1) \right\Vert_2^2 
	\leq - \frac{\alpha_1}{2(1-\beta_1)} \mathbb{E}_1\left\Vert \nabla f(\theta_1) \right\Vert_2^2 
\end{aligned}
\end{equation}

For $t \geq 2$, we assume the effective learning rate is chosen to be small enough such that this coefficient bounds the descent effectively: $\frac{\alpha_t}{1-\beta_1^t} - \frac{\beta_1(\alpha_{t-1} + \alpha_t)}{2(1-\beta_1)} \geq \frac{\alpha_t}{2(1-\beta_1^t)}$. Combining this with the bounds for the telescoping difference sum ($\leq \frac{\alpha_1}{1-\beta_1}$) and the mask flipping bounded by $S_{\text{total}}$, we have:
\begin{equation}
\begin{aligned} 
	f(\theta^*) - f(\theta_1) \leq (C_1 + C_2) \sum_{t=1}^{T}\alpha_{t}^{2} + C_3 \frac{\alpha_1}{1-\beta_1} + C_4 \frac{\alpha_1}{1-\beta_1} S_\text{total} - \sum_{t=1}^{T} \frac{\alpha_t}{2(1-\beta_1^t)} \mathbb{E}_t\left[\left\Vert \nabla f\left(\theta_{t}\right)\right\Vert _{2}^{2}\right] 
\end{aligned} 
\end{equation}

Rearranging the terms to isolate the gradient norm:
\begin{equation}
\begin{aligned} 
	\sum_{t=1}^{T} \frac{\alpha_t}{2(1-\beta_1^t)} \mathbb{E}_t\left[\left\Vert \nabla f\left(\theta_{t}\right)\right\Vert _{2}^{2}\right] 
	\leq (C_1 + C_2) \sum_{t=1}^{T}\alpha_{t}^{2} + f(\theta_1) - f(\theta^*) + C_3 \frac{\alpha_1}{1-\beta_1} + C_4 \frac{\alpha_1}{1-\beta_1} S_\text{total} 
\end{aligned} 
\end{equation}

Let $C_7 \triangleq 2(C_1 + C_2)$ and let $C_8$ absorb all the non-decaying initial constants, $C_8 \triangleq 2\left(f(\theta_1) - f(\theta^*) + C_3 \frac{\alpha_1}{1-\beta_1} + C_4 \frac{\alpha_1}{1-\beta_1} S_\text{total}\right)$. Note that since $1-\beta_1^t \leq 1$, we universally have $\frac{1}{1-\beta_1^t} \geq 1$. Therefore, multiplying both sides by 2 yields:
\begin{equation}
\begin{aligned}
	\sum_{t=1}^{T}\alpha_{t}\mathbb{E}_t\left[\left\Vert \nabla f\left(\theta_t\right)\right\Vert _{2}^{2}\right] 
	\leq \sum_{t=1}^{T}\frac{\alpha_{t}}{1-\beta_1^t}\mathbb{E}_t\left[\left\Vert \nabla f\left(\theta_t\right)\right\Vert _{2}^{2}\right]
	\leq C_7 \sum_{t=1}^{T}\alpha_{t}^{2} + C_8
\end{aligned}
\end{equation}

Extracting the minimum over the trajectory $\mathbb{E}(T) =\min_{t=1,\ldots,T}\mathbb{E}_{t-1}\left[\left\Vert \nabla f\left(\theta_{t}\right)\right\Vert _{2}^{2}\right]$:
\begin{equation}
\begin{aligned} 
	\mathbb{E}\left(T\right)\sum_{t=1}^{T}\alpha_{t} \leq \sum_{t=1}^{T}\alpha_{t}\mathbb{E}_t\left[\left\Vert \nabla f\left(\theta_t\right)\right\Vert _{2}^{2}\right] \leq C_{7}\sum_{t=1}^{T}\alpha_{t}^{2} + C_{8} \Longrightarrow \mathbb{E}\left(T\right) \leq \frac{C_{7}\sum_{t=1}^{T}\alpha_{t}^{2} + C_{8}}{\sum_{t=1}^{T}\alpha_{t}}
\end{aligned}
\end{equation}

Since $\alpha_{t}=\alpha/\sqrt{t}$, we apply the standard integral bound for the squared step size summation:
\begin{equation}	
\sum_{t=1}^{T} \alpha_t^2 = \alpha^2 \sum_{t=1}^{T} \frac{1}{t} \leq \alpha^2 \left( 1 + \int_1^T \frac{1}{x} dx \right) = \alpha^2 (\log T + 1)
\end{equation}

For the linear step size summation, we can strictly lower bound it by replacing each term with the minimum term in the sequence:
\begin{equation}	
\sum_{t=1}^{T} \alpha_t = \alpha \sum_{t=1}^{T} \frac{1}{\sqrt{t}} \geq \alpha \sum_{t=1}^{T} \frac{1}{\sqrt{T}} = \alpha \frac{T}{\sqrt{T}} = \alpha\sqrt{T}
\end{equation}

Substituting the upper bound of the numerator and the lower bound of the denominator yields the final finite-time convergence bound:
\begin{equation}	
\mathbb{E}(T) \leq \frac{C_{7}\alpha^2 (\log T + 1) + C_{8}}{\alpha\sqrt{T}} = \mathcal{O}\left( \frac{\log T}{\sqrt{T}} \right)
\end{equation}

\end{proof}

\newpage
\section{Detailed Experimental Supplement}
\label{sec:appendixe}
We performed extensive comparisons with other optimizers, including SGD~\cite{1951a}, Adam~\cite{kingma2014adam}, RAdam~\cite{liu2019variance} and AdamW~\cite{loshchilov2017decoupled}, and so on. The experiments include: (a) image classification on CIFAR dataset~\cite{krizhevsky2009learning} with VGG~\cite{2014Very}, ResNet~\cite{he2016deep} and DenseNet~\cite{huang2017densely}, and image recognition with VGG, ResNet, and DenseNet on ImageNet~\cite{Deng2009}.

\subsection{Image classification with CNNs on CIFAR}
\label{app:subsec:cifar}
For all experiments, the model is trained for 200 epochs with a batch size of 128, and the learning rate is multiplied by 0.1 at epoch 150. We performed extensive hyperparameter search as described in the main paper. Here, we report both training and test accuracy in \Figref{fig:cifar10} and \Figref{fig:cifar100}. Detailed experimental parameters we place in \Tabref{tab:cifarhyperparameters}. We summarize the mean best test accuracies and their standard deviations for each algorithm in \Tabref{tab:cifar_values}. The best results are highlighted in bold font. SGDF not only achieves the highest test accuracy but also a smaller gap between training and test accuracy compared with other optimizers. We ran each experiment three times with different seeds \{0, 1, 2\} to ensure the robustness of the results.

\begin{table}[h]
	\centering
	\caption{Hyperparameters used for CIFAR-10 and CIFAR-100 datasets. }
    \resizebox{1.0\linewidth}{!}{
	\begin{tabular}{@{}lcccccccc@{}}
		\toprule
		Optimizer & Learning Rate & $\beta_1$ & $\beta_2$ & Epochs & Schedule & Weight Decay & Batch Size & $\varepsilon$ \\
		\midrule
		SGDF    & 0.5 & 0.9   & 0.999  & 200 & StepLR      & 0.0005 & 128 & 1e-8 \\
		SGD     &  0.1  & 0.9 & -   & 200 & StepLR  & 0.0005 & 128 & -    \\
		Adam    &  0.001 & 0.9 & 0.999 & 200 & StepLR  & 0.0005 & 128 & 1e-8 \\
		RAdam   &  0.001 & 0.9 & 0.999 & 200 & StepLR  & 0.0005 & 128 & 1e-8 \\
		AdamW   &  0.001 & 0.9 & 0.999 & 200 & StepLR  & 0.01   & 128 & 1e-8 \\
		MSVAG   &  0.1 & 0.9 & 0.999 & 200 & StepLR  & 0.0005 & 128 & 1e-8 \\
		AdaBound&  0.001 & 0.9 & 0.999 & 200 & StepLR   & 0.0005 & 128 & - \\
		Sophia  & 0.0001 & 0.965 & 0.99 & 200 & StepLR  & 0.1 & 128 & - \\
		Lion    &  0.00002 & 0.9 & 0.99   & 200 & StepLR  & 0.1 & 128 & -    \\
        AdaBound&  0.001 & 0.9 & 0.999 & 200 & StepLR   & 0.0005 & 128 & 1e-8 \\
		\bottomrule
	\end{tabular}
    }
	\label{tab:cifarhyperparameters}
\end{table}

\begin{table}[ht]
\centering
\caption{Test Accuracies for CIFAR-10 and CIFAR-100 across different models and algorithms.}
\resizebox{1.0\linewidth}{!}{
\begin{tabular}{l|ccc|ccc}
\toprule
\multirow{2.5}{*}{Algorithm} & \multicolumn{3}{c}{CIFAR-10} & \multicolumn{3}{c}{CIFAR-100} \\
\cmidrule{2-7}
& VGG11    & ResNet34   & DenseNet121  & VGG11  & ResNet34  & DenseNet121     \\
\midrule
SGDF & $\textbf{91.61}_{\pm 0.21}$ & $\textbf{95.33}_{\pm 0.19}$ & $\textbf{95.74}_{\pm 0.06}$ & $\textbf{68.12}_{\pm 0.15}$ & $\textbf{77.69}_{\pm 0.64}$ & $\textbf{80.17}_{\pm 0.19}$ \\
SGD  & $89.83_{\pm 0.05}$ & $94.62_{\pm 0.07}$ & $94.52_{\pm 0.03}$ & $63.48_{\pm 0.39}$ & $76.88_{\pm 0.12}$ & $78.77_{\pm 0.27}$ \\
Adam & $88.12_{\pm 0.10}$ & $94.30_{\pm 0.06}$ & $94.37_{\pm 0.17}$ & $56.27_{\pm 0.32}$ & $72.81_{\pm 0.45}$ & $74.67_{\pm 0.45}$ \\
AdamW & $88.59_{\pm 0.20}$ & $94.42_{\pm 0.00}$ & $94.61_{\pm 0.06}$ & $58.09_{\pm 0.69}$ & $72.74_{\pm 0.45}$ & $74.96_{\pm 0.10}$ \\
RAdam & $90.47_{\pm 0.34}$ & $93.41_{\pm 0.21}$ & $93.75_{\pm 0.04}$ & $60.20_{\pm 0.37}$ & $74.08_{\pm 0.35}$ & $75.82_{\pm 0.28}$ \\
MSVAG & $90.08_{\pm 0.13}$ & $94.79_{\pm 0.08}$ & $95.01_{\pm 0.12}$ & $61.55_{\pm 0.23}$ & $75.75_{\pm 0.06}$ & $76.84_{\pm 0.13}$ \\
Lion & $88.04_{\pm 0.06}$ & $93.97_{\pm 0.10}$ & $94.26_{\pm 0.02}$ & $55.59_{\pm 0.15}$ & $72.79_{\pm 0.14}$ & $73.41_{\pm 0.10}$ \\
SophiaG & $88.53_{\pm 0.04}$ & $94.15_{\pm 0.26}$ & $94.53_{\pm 0.13}$ & $58.01_{\pm 1.85}$ & $72.83_{\pm 0.18}$ & $75.81_{\pm 0.23}$ \\
AdaBound & $90.41_{\pm 0.12}$ & $94.93_{\pm 0.12}$ & $95.06_{\pm 0.13}$ & $64.51_{\pm 0.15}$ & $76.37_{\pm 0.29}$ & $77.43_{\pm 0.18}$ \\
AdaBelief & $91.24_{\pm 0.04}$ & $95.18_{\pm 0.01}$ & $95.44_{\pm 0.04}$ & $67.59_{\pm 0.03}$ & $77.47_{\pm 0.34}$ & $79.20_{\pm 0.16}$ \\
\bottomrule
\end{tabular}
}
\label{tab:cifar_values}
\end{table}

\vspace{-1em}
\begin{figure}[htp]
    \centering
    \begin{subfigure}[t]{0.32\textwidth}
        \centering
        \includegraphics[width=\linewidth]{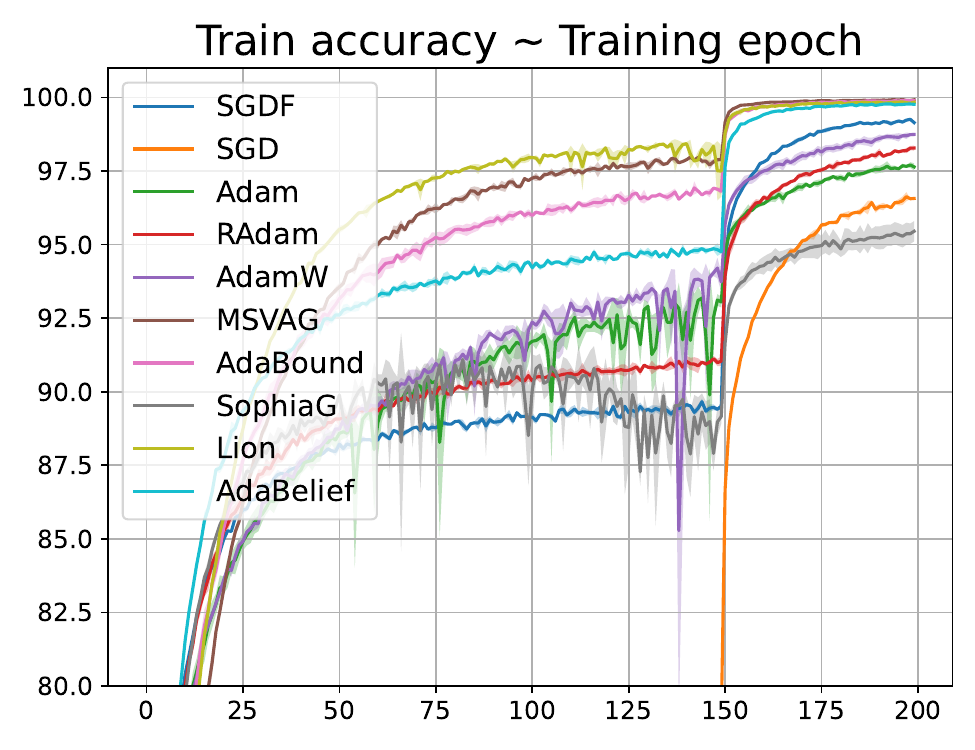}
        \caption{VGG11 on CIFAR-10 (Training)}
        \label{subfig:vgg_train_cifar10}
    \end{subfigure}
    \begin{subfigure}[t]{0.32\textwidth}
        \centering
        \includegraphics[width=\linewidth]{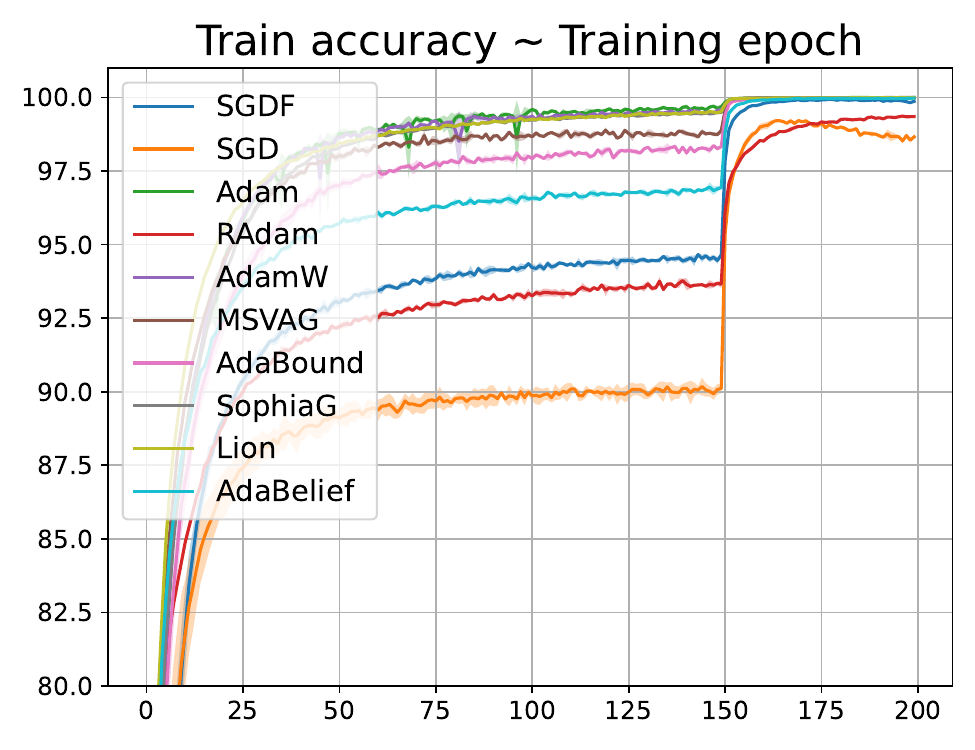}
        \caption{ResNet34 on CIFAR-10 (Training)}
        \label{subfig:resnet_train_cifar10}
    \end{subfigure}
    \begin{subfigure}[t]{0.32\textwidth}
        \centering
        \includegraphics[width=\linewidth]{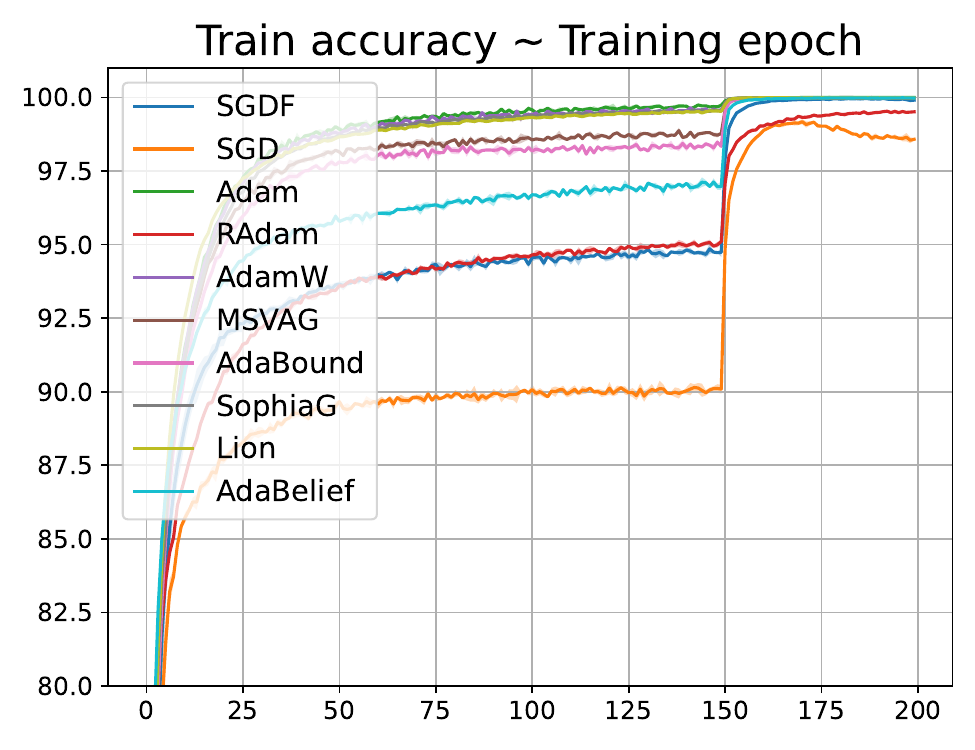}
        \caption{DenseNet121 on CIFAR-10 (Training)}
        \label{subfig:densenet_train_cifar10}
    \end{subfigure}
    \begin{subfigure}[t]{0.32\textwidth}
        \centering
        \includegraphics[width=\linewidth]{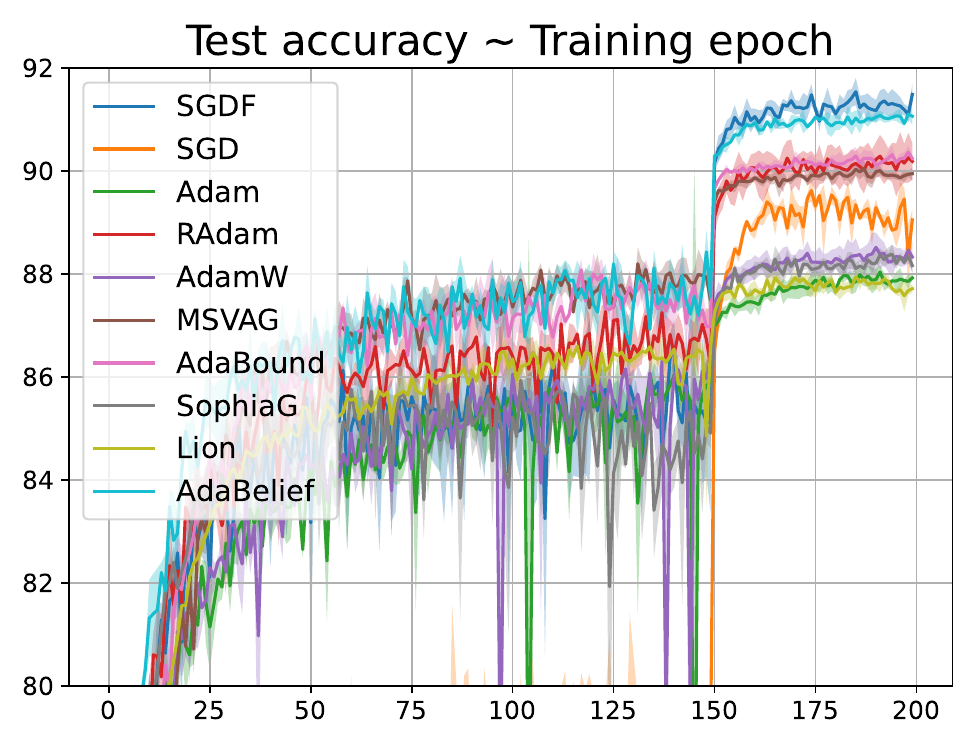}
        \caption{VGG11 on CIFAR-10 (Test)}
        \label{subfig:vgg_test_cifar10}
    \end{subfigure}
    \begin{subfigure}[t]{0.32\textwidth}
        \centering
        \includegraphics[width=\linewidth]{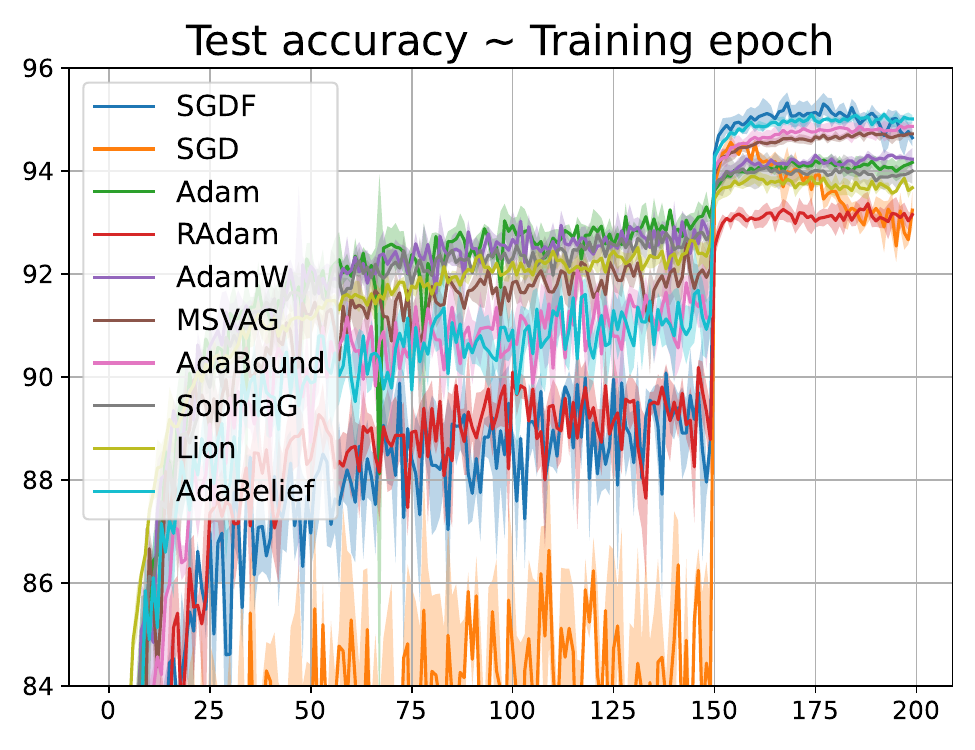}
        \caption{ResNet34 on CIFAR-10 (Test)}
        \label{subfig:resnet_test_cifar10}
    \end{subfigure}
    \begin{subfigure}[t]{0.32\textwidth}
        \centering
        \includegraphics[width=\linewidth]{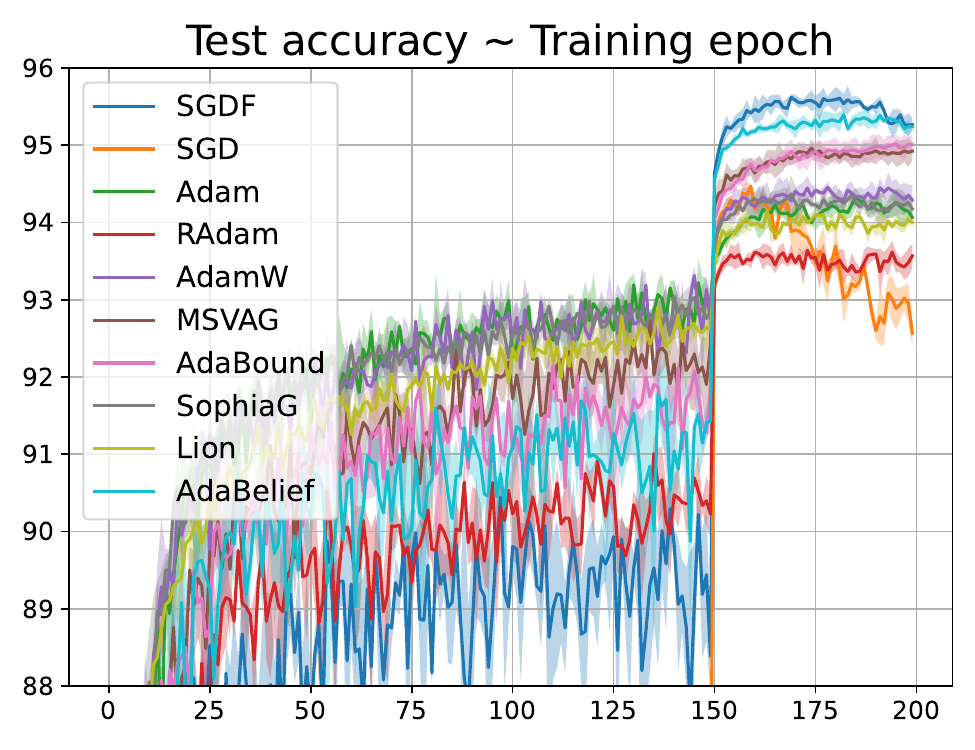}
        \caption{DenseNet121 on CIFAR-10 (Test)}
        \label{subfig:densenet_test_cifar10}
    \end{subfigure}
    \caption{Training (top row) and test (bottom row) accuracy of CNNs on CIFAR-10 dataset. We report confidence interval ([$\mu \pm \sigma$]) of 3 independent runs.}
    \label{fig:cifar10}
\end{figure}

\begin{figure}[htp]
    \centering
    \begin{subfigure}[t]{0.32\textwidth}
        \centering
        \includegraphics[width=\linewidth]{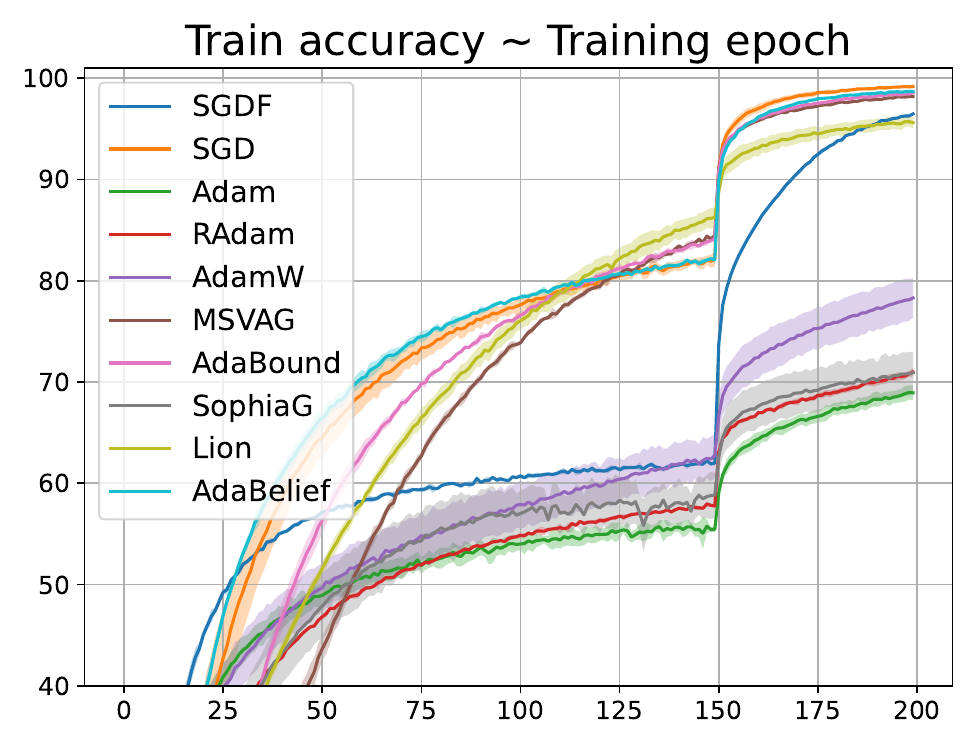}
        \caption{VGG11 on CIFAR-100 (Training)}
        \label{subfig:vgg_train_cifar100}
    \end{subfigure}
    \begin{subfigure}[t]{0.32\textwidth}
        \centering
        \includegraphics[width=\linewidth]{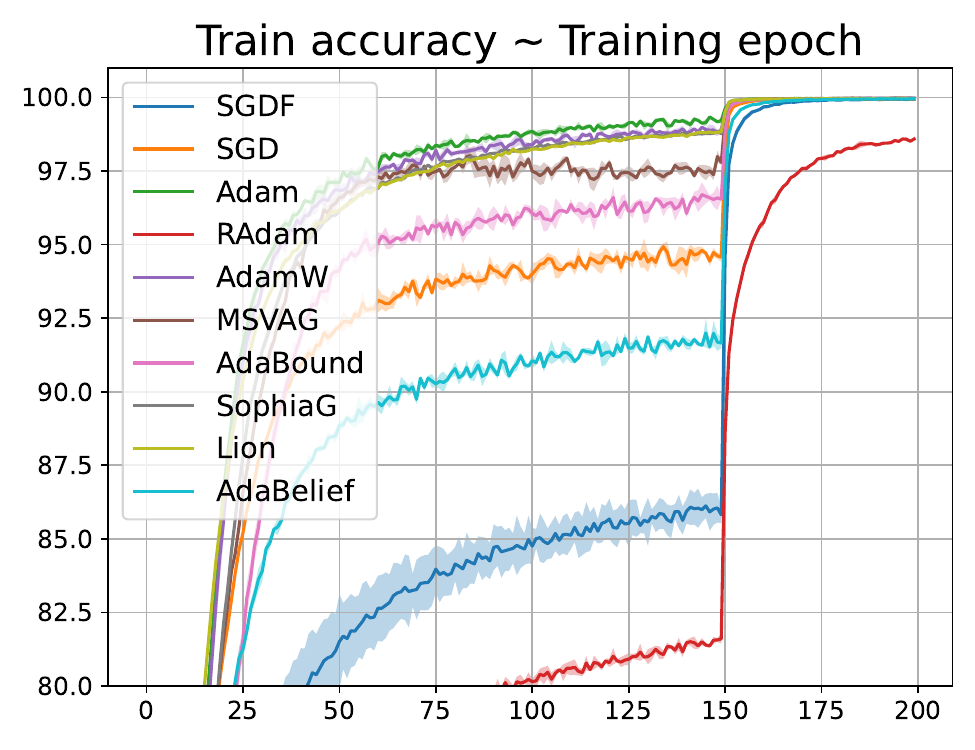}
        \caption{ResNet34 on CIFAR-100 (Training)}
        \label{subfig:resnet_train_cifar100}
    \end{subfigure}
    \begin{subfigure}[t]{0.32\textwidth}
        \centering
        \includegraphics[width=\linewidth]{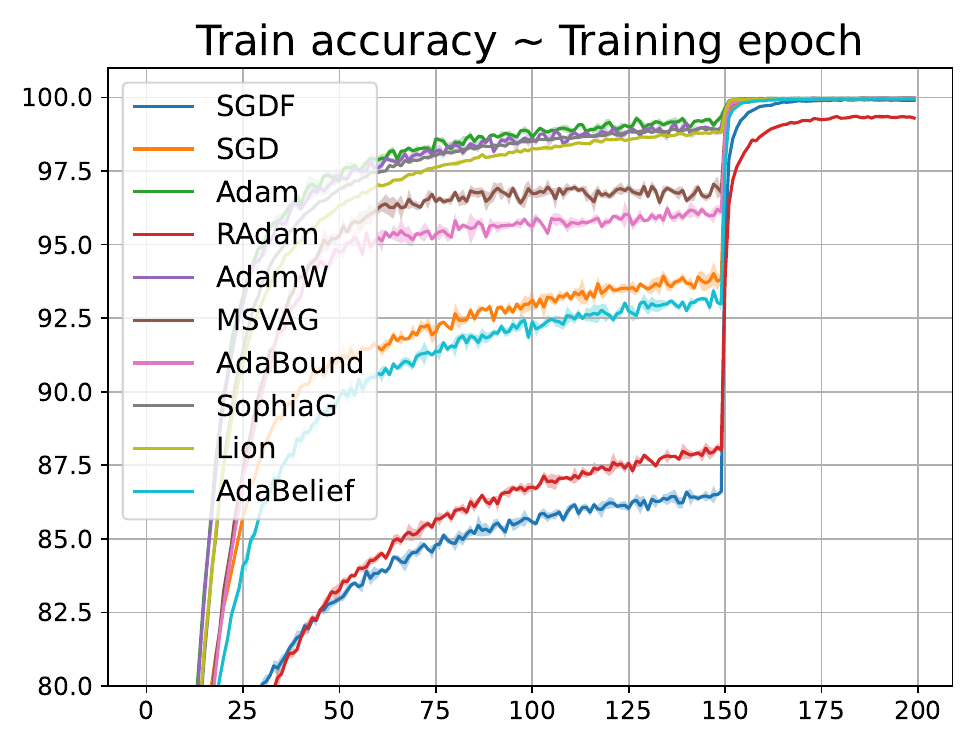}
        \caption{DenseNet121 on CIFAR-100 (Training)}
        \label{subfig:densenet_train_cifar100}
    \end{subfigure}
    \begin{subfigure}[t]{0.32\textwidth}
        \centering
        \includegraphics[width=\linewidth]{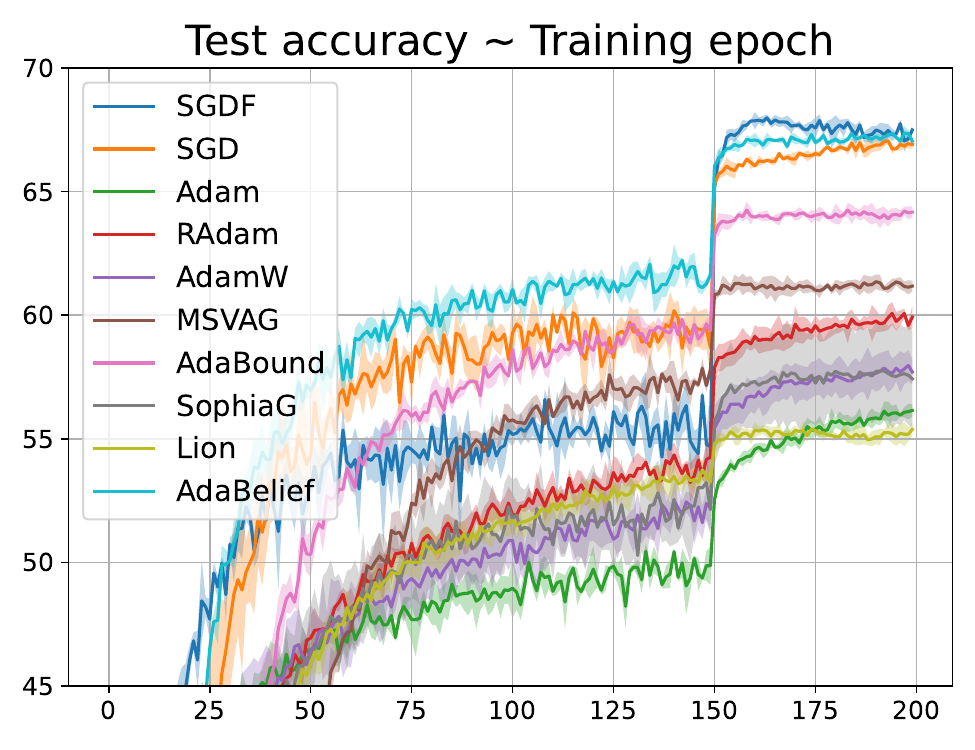}
        \caption{VGG11 on CIFAR-100 (Test)}
        \label{subfig:vgg_test_cifar100}
    \end{subfigure}
    \begin{subfigure}[t]{0.32\textwidth}
        \centering
        \includegraphics[width=\linewidth]{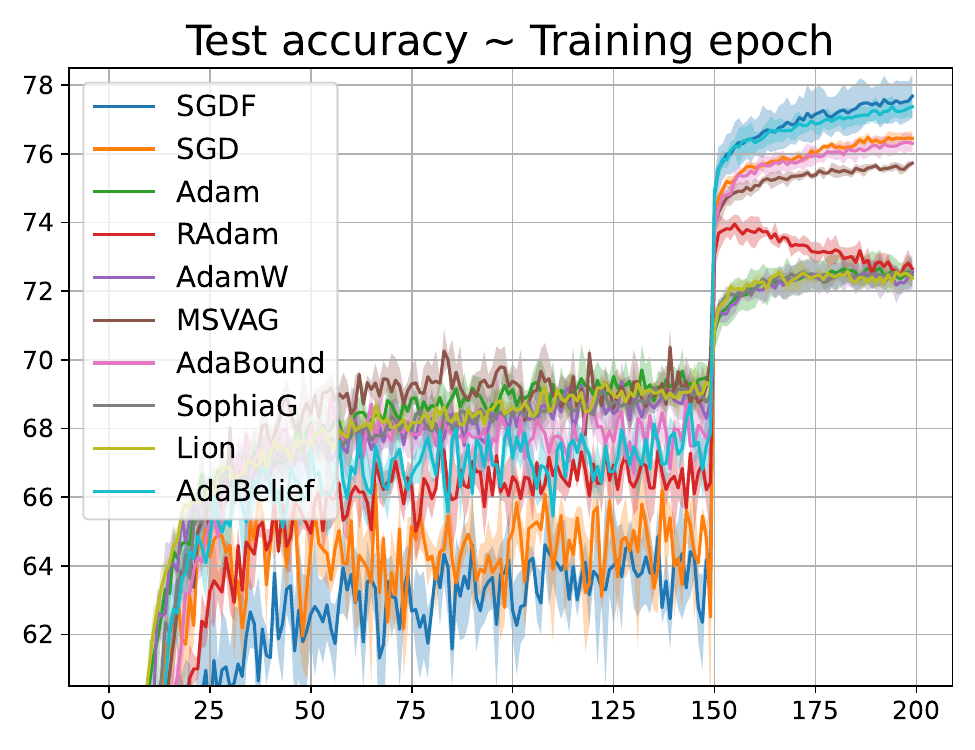}
        \caption{ResNet34 on CIFAR-100 (Test)}
        \label{subfig:resnet_test_cifar100}
    \end{subfigure}
    \begin{subfigure}[t]{0.32\textwidth}
        \centering
        \includegraphics[width=\linewidth]{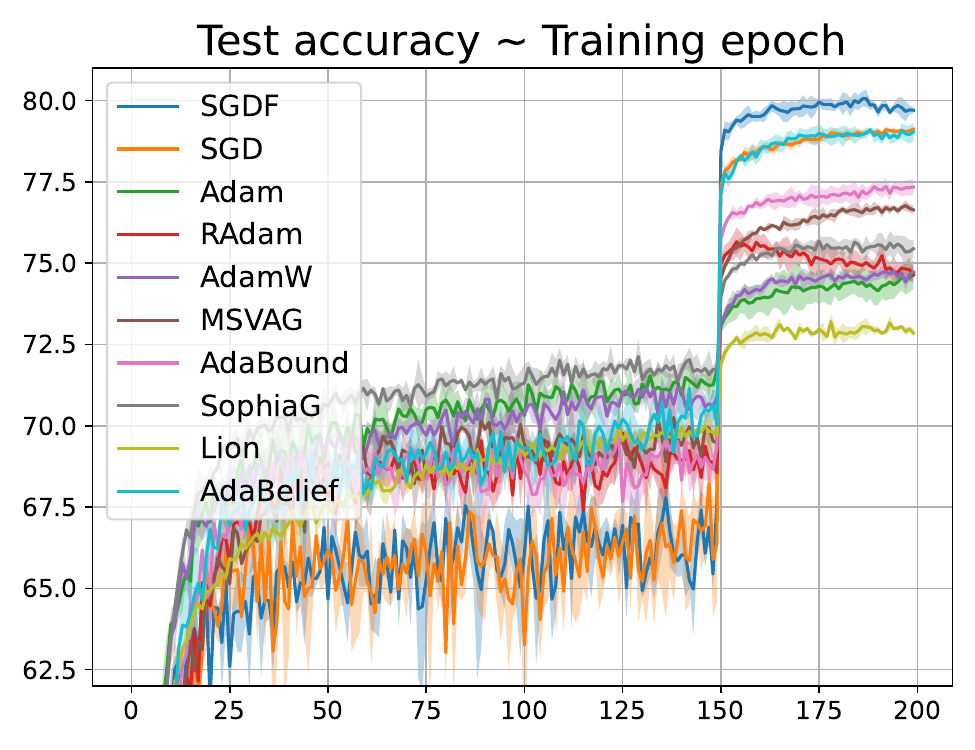}
        \caption{DenseNet121 on CIFAR-100 (Test)}
        \label{subfig:densenet_test_cifar100}
    \end{subfigure}

    \caption{Training (top row) and test (bottom row) accuracy of CNNs on CIFAR-100 dataset. We report confidence interval ([$\mu \pm \sigma$]) of 3 independent runs.}
    \label{fig:cifar100}
\end{figure}

\vspace{3em}

\subsection{Image Classification on ImageNet}
\label{app:subsec:imagenet}
We experimented with a VGG / ResNet / DenseNet on ImageNet classification task. For SGDF and SGD, we set the initial learning rate of 0.5 same as CIFAR experiments. The weight decay is set as \(10^{-4}\)  for both cases to match the settings in~\cite{liu2019variance}. Here $\beta_1$ serves to capture the gradient mean. The more closer $\beta_1$ is to 1, the longer the moving window is and the wider the historical mean is captured. Since ImageNet dataset has more iterations than CIFAR dataset, we directly set $\beta_1$ = 0.5 to prevent $K_t$ from being overly influenced by the historical mean gradient. For sure, setting $\beta_1$ to 0.9, consistent with CIFAR experiments can also be superior to SGD, and adjusting $\beta_1$ to 0.5 or 0.9 according to the size of the dataset and batch size can bring better results. Detailed experimental parameters we place in \Tabref{tab:imagenethyperparameters}. As shown in \Figref{fig:imagenet_curve}, SGDF outperformed SGD. 

\begin{table}[htp]
	\centering
	\caption{Hyperparameters used for ImageNet.}
	\begin{tabular}{@{}lcccccccc@{}}
		\toprule
		Optimizer & Learning Rate & $\beta_1$ & $\beta_2$ & Epochs & Schedule & Weight Decay & Batch Size & $\varepsilon$ \\
		\midrule
		 SGDF    & 0.5 & 0.5   & 0.999  & 100/90 & Cosine      & 0.0001 & 256 & 1e-8 \\
		 SGD     & 0.1  & 0.9 & -     & 100/90 & Cosine      & 0.0001 & 256 & - \\
		\bottomrule
	\end{tabular}
	\label{tab:imagenethyperparameters}
\end{table}
\vspace{-3mm}
\begin{figure}[htp]
    \centering
    \begin{subfigure}[t]{0.45\textwidth}
        \centering
        \includegraphics[width=\linewidth]{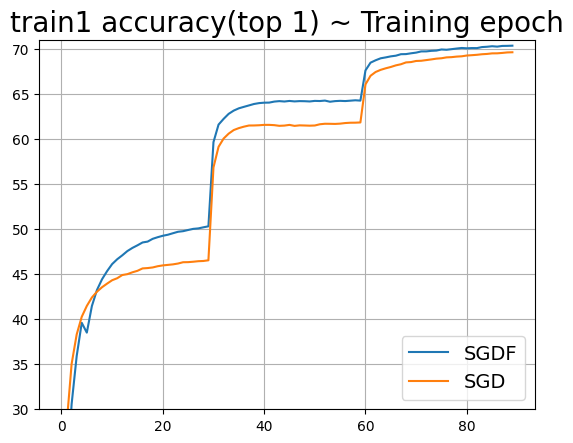}
        \caption{Training accuracy (top-1)}
        \label{subfig:imagenet_train}
    \end{subfigure}
    \begin{subfigure}[t]{0.45\textwidth}
        \centering
        \includegraphics[width=\linewidth]{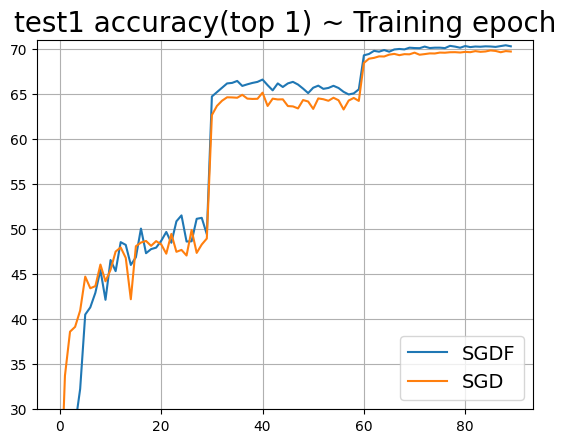}
        \caption{Test accuracy (top-1)}
        \label{subfig:imagenet_test}
    \end{subfigure}
    
    \caption{Training and test accuracy (top-1) of ResNet18 on ImageNet.}
    \label{fig:imagenet_curve}
\end{figure}


\subsection{Objective Detection on PASCAL VOC}
\label{app:subsec:detection}
We conducted object detection experiments on the PASCAL VOC dataset~\cite{pascal-voc-2010}. The model used in these experiments was pre-trained on the COCO dataset~\citep{coco-dataset}, obtained from the official website. We trained this model on the VOC2007 and VOC2012 trainval dataset (17K) and evaluated it on the VOC2007 test dataset (5K). The utilized model was Faster-RCNN~\cite{faster_rcnn} with FPN, and the backbone was ResNet50~\cite{he2016deep}. We train 4 epochs and adjust the learning rate decay by a factor of 0.1 at the last epoch. We show the results on PASCAL VOC~\cite{pascal-voc-2010}. Object detection with a Faster-RCNN model\cite{faster_rcnn}. Detailed experimental parameters we place in \Figref{tab:hyperparameters_object_detection}. The results are reported in \Tabref{tab:object_detection}, and detection examples are shown in \Figref{fig:detection_examples}. These results also illustrate that our method is still efficient in object detection tasks.
\begin{table*}[htp]
	\centering
	\caption{Hyperparameters for object detection on PASCAL VOC using Faster-RCNN+FPN with different optimizers.}
    \resizebox{0.9\linewidth}{!}{
	\begin{tabular}{@{}lcccccccc@{}}
		\toprule
		Optimizer & Learning Rate & $\beta_1$& $\beta_2$ & Epochs & Schedule & Weight Decay &Batch Size & $\varepsilon$ \\
		\midrule
		SGDF       & 0.01 & 0.9 & 0.999 & 4 & StepLR & 0.0001 &2& 1e-8 \\
		SGD        & 0.01 & 0.9 & -     & 4 & StepLR & 0.0001 &2& - \\
		Adam       & 0.0001 & 0.9 & 0.999 & 4 & StepLR & 0.0001 &2& 1e-8 \\
		AdamW      & 0.0001 & 0.9 & 0.999 & 4 & StepLR & 0.0001 &2& 1e-8 \\
		RAdam      & 0.0001 & 0.9 & 0.999 & 4 & StepLR & 0.0001 &2& 1e-8 \\
		\bottomrule
	\end{tabular}
    }
	\label{tab:hyperparameters_object_detection}
\end{table*}

\begin{figure*}[ht]
    \centering
    \begin{subfigure}[t]{0.19\textwidth}
        \centering
        \includegraphics[width=\linewidth]{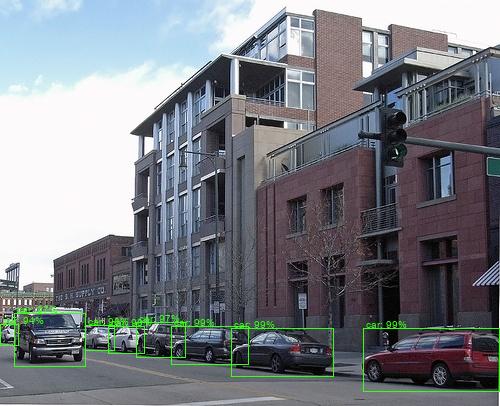}
        \caption{SGDF}
        \label{subfig:sgdf_test1}
    \end{subfigure}
    \begin{subfigure}[t]{0.19\textwidth}
        \centering
        \includegraphics[width=\linewidth]{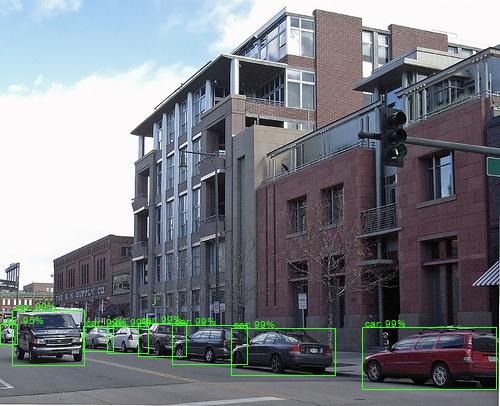}
        \caption{SGDM}
        \label{subfig:sgdm_test1}
    \end{subfigure}
    \begin{subfigure}[t]{0.19\textwidth}
        \centering
        \includegraphics[width=\linewidth]{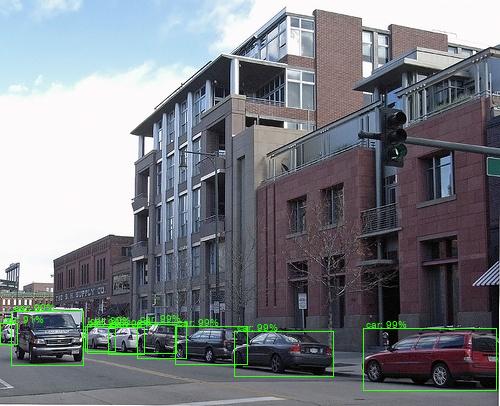}
        \caption{Adam}
        \label{subfig:adam_test1}
    \end{subfigure}
    \begin{subfigure}[t]{0.19\textwidth}
        \centering
        \includegraphics[width=\linewidth]{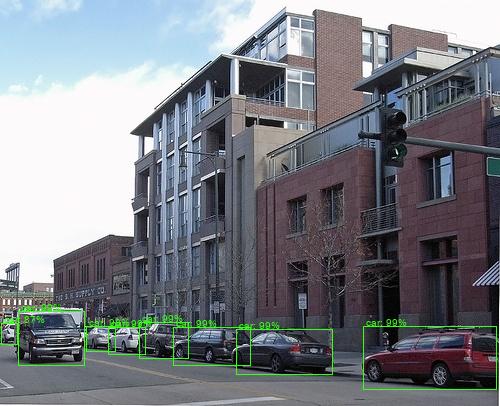}
        \caption{AdamW}
        \label{subfig:adamw_test1}
    \end{subfigure}
    \begin{subfigure}[t]{0.19\textwidth}
        \centering
        \includegraphics[width=\linewidth]{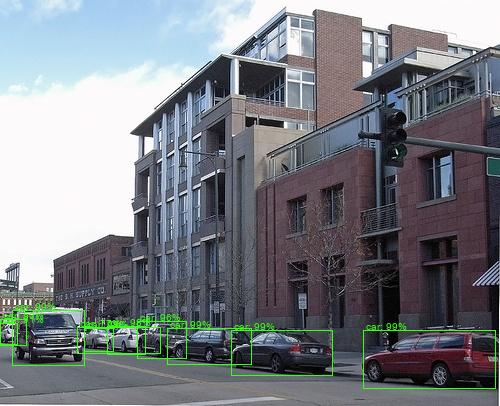}
        \caption{RAdam}
        \label{subfig:radam_test1}
    \end{subfigure}
    \begin{subfigure}[t]{0.19\textwidth}
        \centering
        \includegraphics[width=\linewidth]{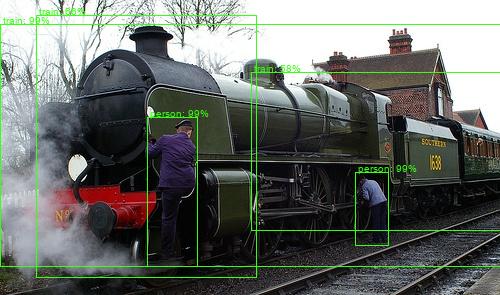}
        \caption{SGDF}
        \label{subfig:sgdf_test2}
    \end{subfigure}
    \begin{subfigure}[t]{0.19\textwidth}
        \centering
        \includegraphics[width=\linewidth]{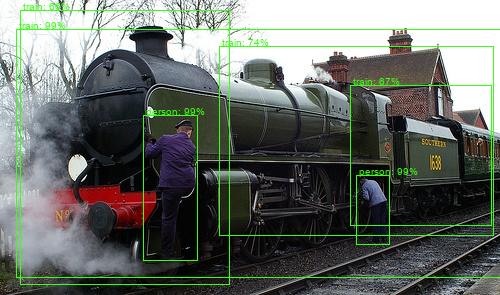}
        \caption{SGDM}
        \label{subfig:sgdm_test2}
    \end{subfigure}
    \begin{subfigure}[t]{0.19\textwidth}
        \centering
        \includegraphics[width=\linewidth]{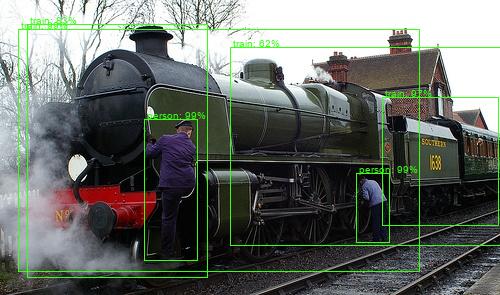}
        \caption{Adam}
        \label{subfig:adam_test2}
    \end{subfigure}
    \begin{subfigure}[t]{0.19\textwidth}
        \centering
        \includegraphics[width=\linewidth]{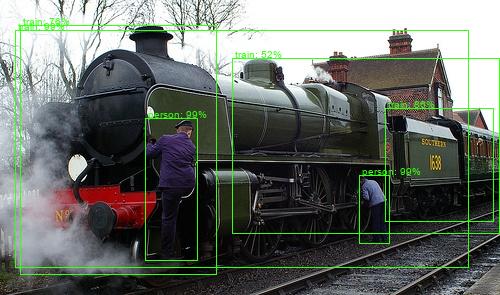}
        \caption{AdamW}
        \label{subfig:adamw_test2}
    \end{subfigure}
    \begin{subfigure}[t]{0.19\textwidth}
        \centering
        \includegraphics[width=\linewidth]{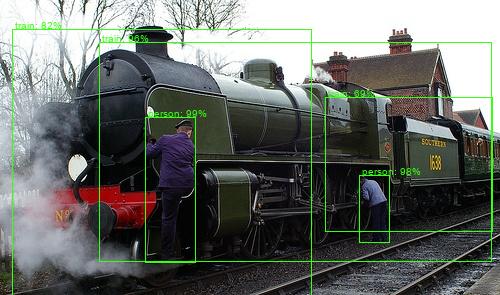}
        \caption{RAdam}
        \label{subfig:radam_test2}
    \end{subfigure}

    \caption{Detection examples using Faster-RCNN + FPN trained on PASCAL VOC.}
    \label{fig:detection_examples}
\end{figure*}


\subsection{Image Generation}
The stability of optimizers is crucial, especially when training Generative Adversarial Networks (GANs). If the generator and discriminator have mismatched complexities, it can lead to imbalance during GAN training, causing the GAN to fail to converge. This is known as model collapse. For instance, Vanilla SGD frequently causes model collapse, making adaptive optimizers like Adam and RMSProp the preferred choice. Therefore, GAN training provides a good benchmark for assessing optimizer stability. For reproducibility details, please refer to the parameter table in \Tabref{tab:image_generation_hyperparameters}.

We evaluated the Wasserstein-GAN with gradient penalty (WGAN-GP)~\cite{salimans2016improved}. Using well-known optimizers~\cite{bernstein2020distance,zaheer2018adaptive}, the model was trained for 100 epochs. We then calculated the Frechet Inception Distance (FID)~\cite{2017GANs} which is a metric that measures the similarity between the real image and the generated image distribution and is used to assess the quality of the generated model (lower FID indicates superior performance). Five random runs were conducted, and the outcomes are presented in \Tabref{tab:fid-wgan-gp}. Results for SGD and MSVAG were extracted from Zhuang~\etal~ 
\cite{zhuang2020adabelief}.

%
\begin{table*}[htbp]
\vspace{-0.5em}
\centering
\caption{FID score of WGAN-GP.}
\resizebox{1.0\textwidth}{!}{%
\begin{tabular}{c|ccccccccc}
\toprule
Method & SGDF & Adam & RMSProp & RAdam & Fromage & Yogi & AdaBound & SGD & MSVAG \\
\midrule
FID & $88.7 \pm 4.9$ & $\textbf{78.6} \pm 4.8$ & $109.2 \pm 14.5$ & $93.4 \pm 8.3$ & $101.5 \pm 28.9$ & $138.7 \pm 21.2$ & $119.8 \pm 24.6$ & $250.3 \pm 30.1$ & $239.7 \pm 5.2$ \\
\bottomrule
\end{tabular}%
}
\label{tab:fid-wgan-gp}
\end{table*}

Experimental results demonstrate that SGDF significantly enhances WGAN-GP model training, achieving a FID score higher than vanilla SGD and outperforming most adaptive optimization methods. The integration of an Optimal Linear Filter in SGDF facilitates smooth gradient updates, mitigating training oscillations and effectively addressing the issue of pattern collapse.

\begin{table}[h]
	\centering
	\caption{Hyperparameters for Image Generation Tasks.}
    \resizebox{0.7\linewidth}{!}{
	\begin{tabular}{@{}lcccccc@{}}
		\toprule
		Optimizer & Learning Rate & $\beta_1$ & $\beta_2$ & Epochs & Batch Size & $\varepsilon$ \\
		\midrule
		SGDF     & 0.01 & 0.5 & 0.999 & 100 & 64 & 1e-8 \\
            SGD     & 0.01 & 0.5 & - & 100 & 64 & - \\
		Adam      & 0.0002 & 0.5 & 0.999 & 100 & 64 & 1e-8 \\
		AdamW     & 0.0002 & 0.5 & 0.999 & 100 & 64 & 1e-8 \\
		Fromage   & 0.01   & 0.5 & 0.999 & 100 & 64 & 1e-8 \\
		RMSProp   & 0.0002 & 0.5 & 0.999 & 100 & 64 & 1e-8 \\
		AdaBound  & 0.0002 & 0.5 & 0.999 & 100 & 64 & 1e-8 \\
		Yogi      & 0.01   & 0.5 & 0.999 & 100 & 64 & 1e-8 \\
		RAdam     & 0.0002 & 0.5 & 0.999 & 100 & 64 & 1e-8 \\
		\bottomrule
	\end{tabular}
    }
	\label{tab:image_generation_hyperparameters}
\end{table}

\subsection{LSTM on Language Modeling}
We experiment with LSTM~\cite{ma2015long} on the Penn-TreeBank dataset~\cite{marcinkiewicz1994building}. For SGDF, we apply gradient clipping with values of 0.1 for the LSTM  1-layer, 0.15 for 2-layer, and 0.25 for 3-layer. The learning rate is decayed by multiplying it by 0.1 at epochs [100, 145] using StepLR scheduling, as detailed in \Tabref{tab:lstmhyperparameters} alongside other hyperparameters. We report the mean and standard deviation across 3 independent runs with random seeds \{0, 1, 2\} to ensure result robustness. As shown in \Figref{fig:lstm}, which presents both training and test perplexity curves for 1-layer, 2-layer, and 3-layer LSTM architectures, SGDF consistently outperforms other optimizers by achieving the lowest perplexity across all model configurations, highlighting its effectiveness for sequence modeling tasks. Notably, Adabelief performs poorly under standard parameters, failing to match the performance of other optimizers.

\begin{figure}[htp]
    \centering
    \begin{subfigure}[t]{0.32\textwidth}
        \centering
        \includegraphics[width=\linewidth]{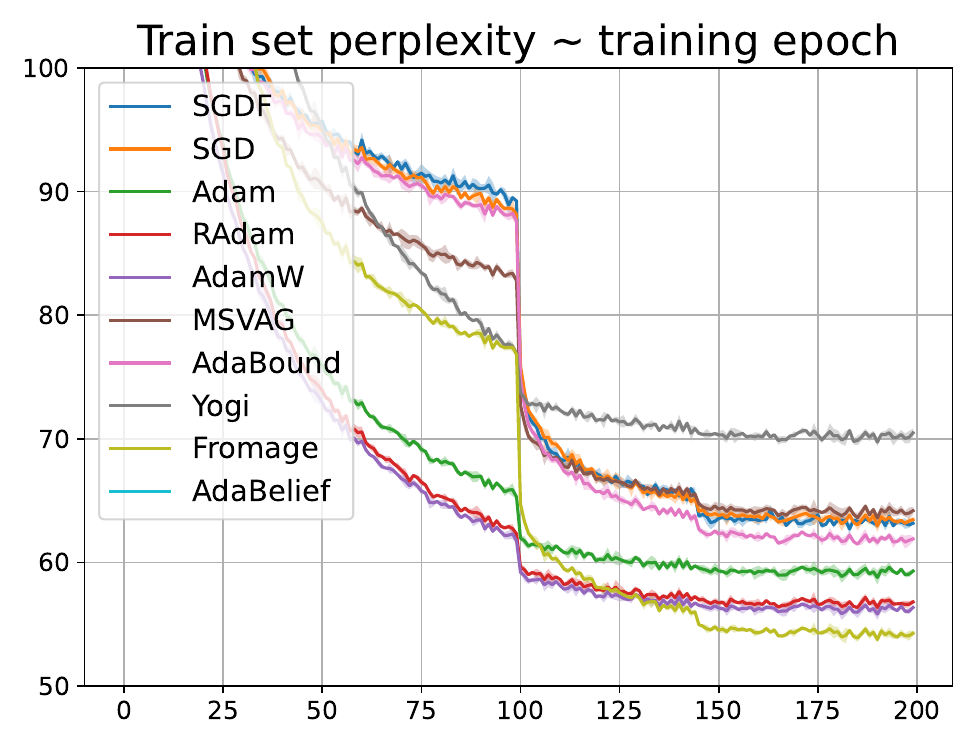}
        \caption{1-layer LSTM}
        \label{subfig:lstm_1layer_train}
    \end{subfigure}
    \begin{subfigure}[t]{0.32\textwidth}
        \centering
        \includegraphics[width=\linewidth]{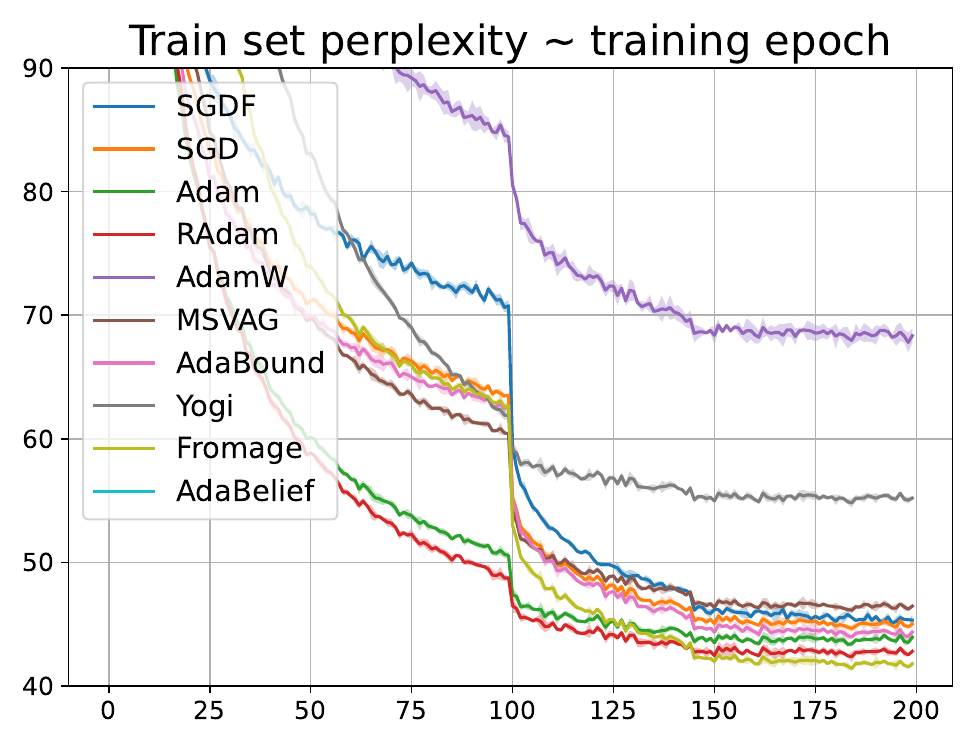}
        \caption{2-layer LSTM}
        \label{subfig:lstm_2layer_train}
    \end{subfigure}
    \begin{subfigure}[t]{0.32\textwidth}
        \centering
        \includegraphics[width=\linewidth]{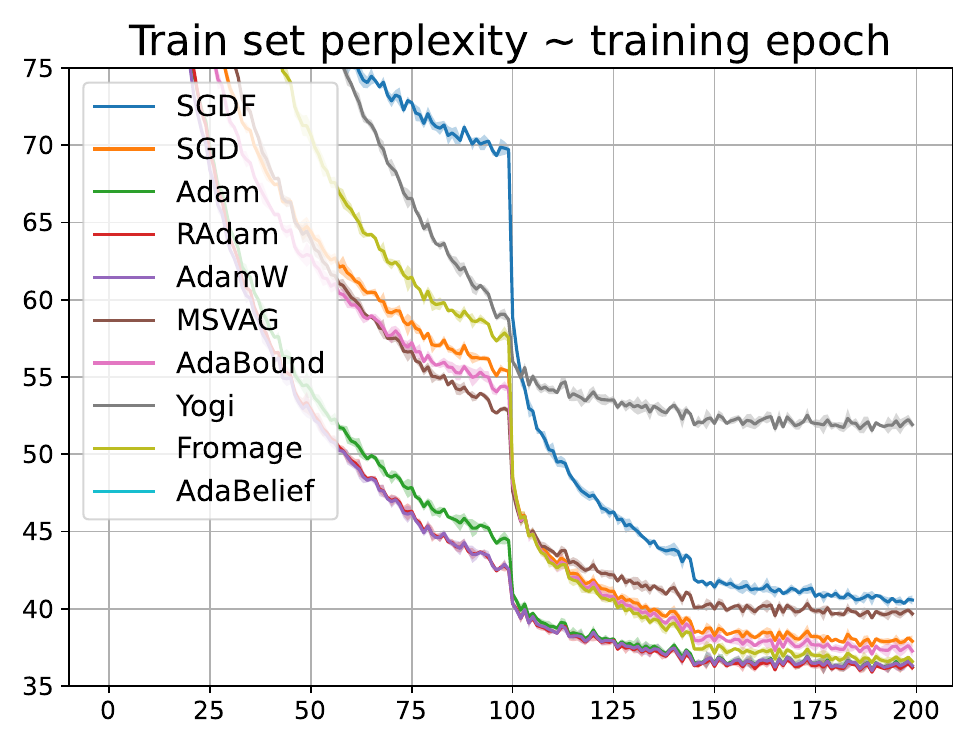}
        \caption{3-layer LSTM}
        \label{subfig:lstm_3layer_train}
    \end{subfigure}
    \begin{subfigure}[t]{0.32\textwidth}
        \centering
        \includegraphics[width=\linewidth]{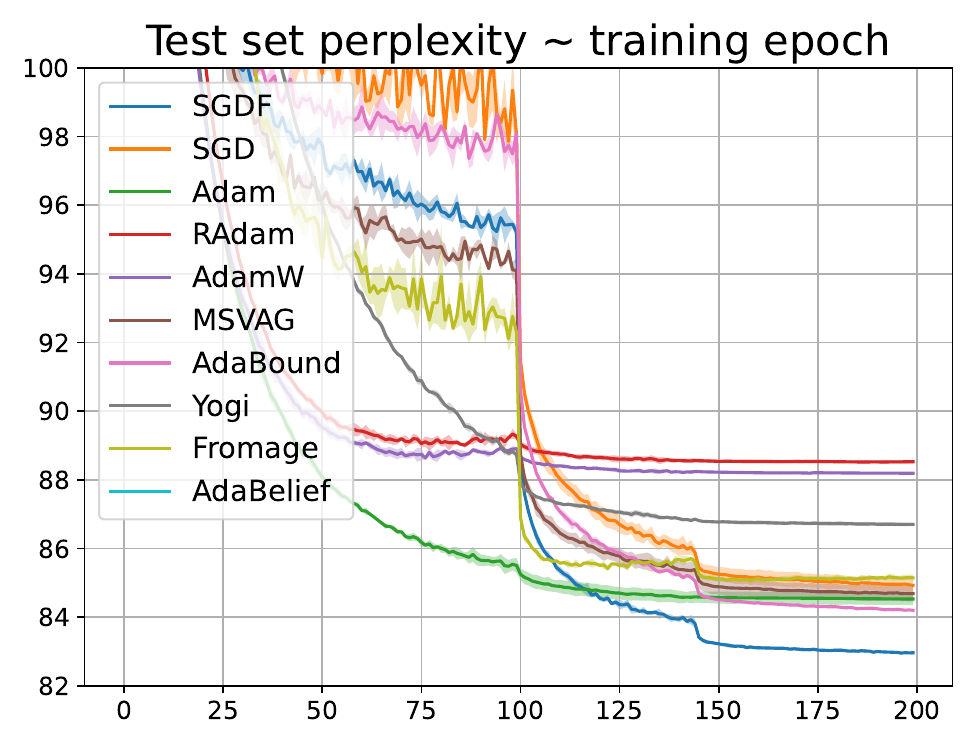}
        \caption{1-layer LSTM}
        \label{subfig:lstm_1layer_test}
    \end{subfigure}
    \begin{subfigure}[t]{0.32\textwidth}
        \centering
        \includegraphics[width=\linewidth]{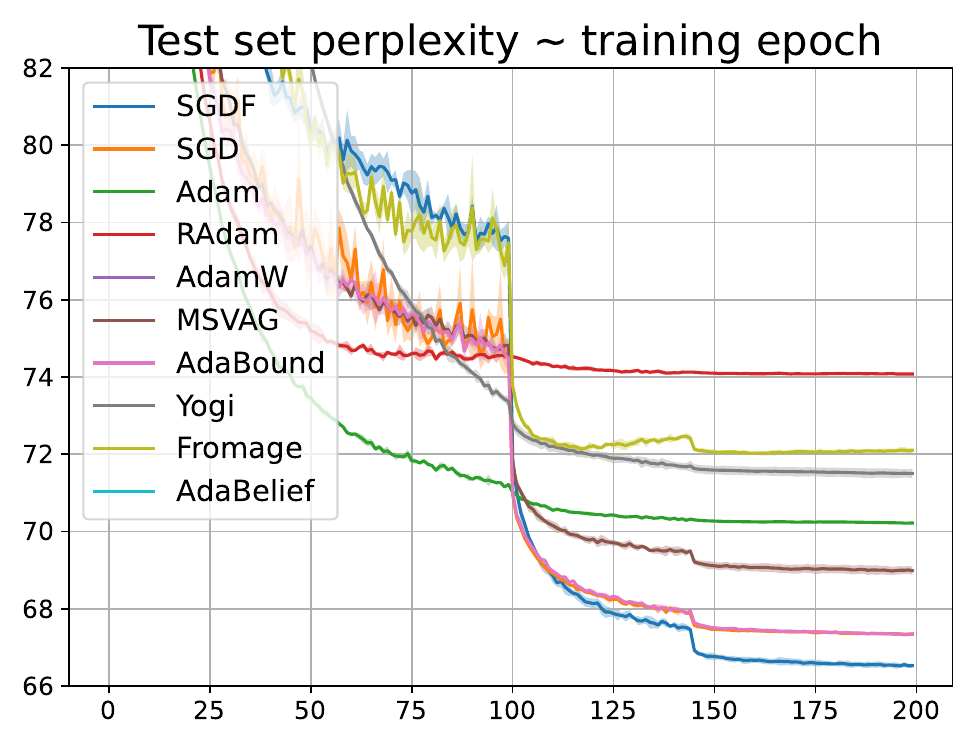}
        \caption{2-layer LSTM}
        \label{subfig:lstm_2layer_test}
    \end{subfigure}
    \begin{subfigure}[t]{0.32\textwidth}
        \centering
        \includegraphics[width=\linewidth]{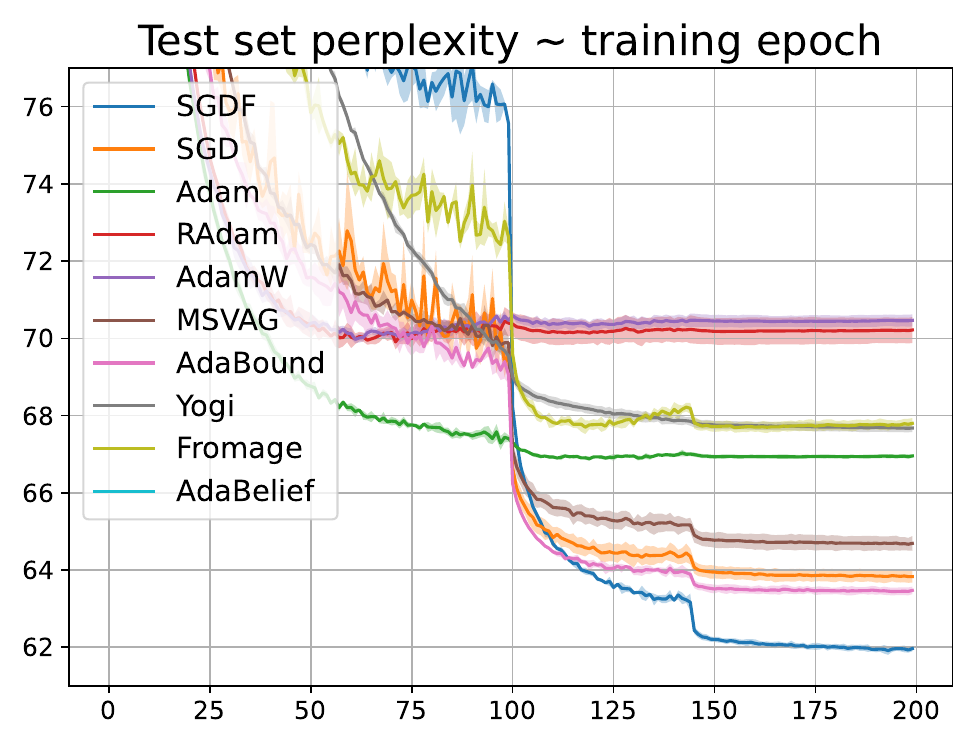}
        \caption{3-layer LSTM}
        \label{subfig:lstm_3layer_test}
    \end{subfigure}

    \caption{Training (top row) and test (bottom row) perplexity on Penn-TreeBank dataset, lower is better.}
    \label{fig:lstm}
\end{figure}

\begin{table}[h]
	\centering
	\caption{Hyperparameters used for LSTM. }
    \resizebox{1.0\linewidth}{!}{
	\begin{tabular}{@{}lcccccccc@{}}
		\toprule
		Optimizer & Learning Rate & $\beta_1$ & $\beta_2$ & Epochs & Schedule & Weight Decay & Batch Size & $\varepsilon$ \\
		\midrule
		SGDF    & 60 & 0.9   & 0.999  & 200 & StepLR      & 1.2e-6 & 20 & 1e-8 \\
		SGD     &  30  & 0.9 & -   & 200 & StepLR  & 1.2e-6 & 20 & -    \\
		Adam    &  0.001 & 0.9 & 0.999 & 200 & StepLR  & 1.2e-6 & 20 & 1e-8 \\
		RAdam   &  0.001 & 0.9 & 0.999 & 200 & StepLR  & 1.2e-6 & 20 & 1e-8 \\
		AdamW   &  0.001 & 0.9 & 0.999 & 200 & StepLR  & 1.2e-6   & 20 & 1e-8 \\
		MSVAG   &  30 & 0.9 & 0.999 & 200 & StepLR  & 1.2e-6 & 20 & 1e-8 \\
		AdaBound&  0.001 & 0.9 & 0.999 & 200 & StepLR   & 1.2e-6 & 20 & - \\
		Yogi  & 0.01 & 0.9 & 0.999 & 200 & StepLR  & 1.2e-6 & 20 & 1e-3 \\
		Fromage    &  0.01 & 0.9 & 0.999   & 200 & StepLR  & 1.2e-6 & 20 & -    \\
        Adabelief    &  0.001 & 0.9 & 0.999   & 200 & StepLR  & 1.2e-6 & 20 & 1e-8    \\
		\bottomrule
	\end{tabular}
    }
	\label{tab:lstmhyperparameters}
\end{table}

\subsection{Post-training in ViT.}
\label{app:subsec:post-training}
To evaluate SGDF's performance, we used Vision Transformers (ViT)~\cite{dosovitskiy2020image} on six benchmark datasets: CIFAR-10, CIFAR-100, Oxford-IIIT-Pets~\cite{parkhi2012cats}, Oxford Flowers-102~\cite{nilsback2008automated}, Food101~\cite{bossard2014food}, and ImageNet-1K. Two ViT variants, ViT-B/32 and ViT-L/32, pretrained on ImageNet-21K, were selected. For fine-tuning, we replaced the original MLP classification head with a new fully connected layer, tailored to the dataset categories. All Transformer backbone weights were retained, preserving the rich representations learned from ImageNet-21K. We increased the image resolution (\eg, from $224 \times 224$ to $384 \times 384$) to improve accuracy, while adjusting positional encoding through 2D interpolation to match the new resolution. For optimization, SGDF was compared to SGD with momentum as a baseline (We research learning set \{ 0.001, 0.003, 0.01, 0.03\} same as~\cite{dosovitskiy2020image}. For ours method, we're not tuning and just mirror the hyperparameter in the CIFAR experiments.), using cosine learning rate decay and no weight decay. A batch size of 512 and global gradient clipping (norm of 1) were used to prevent gradient explosion. All experiments were trained uniformly for 10 epochs and the random seed is set as the current year. We set the random seed to \{0, 1, 2\}. Results are summarized in \Tabref{tab:vitresult}. We summarized the hyperparameter in \Tabref{tab:vithyperparameters}. 
\begin{table}[htp]
	\centering
	\caption{Hyperparameters used for fine-tuning ViT.}
    \resizebox{1.0\linewidth}{!}{
	\begin{tabular}{@{}lccccccccc@{}}
		\toprule
		Optimizer & Learning Rate & $\beta_1$ & $\beta_2$ & Epochs & Schedule & Weight Decay & Batch Size & $\varepsilon$ & Resolution \\
		\midrule
		 SGDF    & 0.5 & 0.9   & 0.999  & 10 & Cosine      & 0 & 512 & 1e-8 & 384\\
		 SGD     & 0.03  & 0.9 & -     & 10 & Cosine      & 0 & 512 & - & 384\\
		\bottomrule
	\end{tabular}
    }
	\label{tab:vithyperparameters}
\end{table}

\subsection{Top Eigenvalues of Hessian and Hessian Trace}
We computed the Hessian spectrum of ResNet-18 trained on the CIFAR-100 dataset for 200 epochs using more optimization methods: SGDF, SGD, SGD-EMA, SGD-CM, Adabelief, Adam, AdamW, and RAdam. We employed power iteration~\citep{2018Hessian} to compute the top eigenvalues of Hessian and Hutchinson’s method~\citep{2020PyHessian} to compute the Hessian trace. Histograms illustrating the distribution of the top 50 Hessian eigenvalues for each optimization method are presented in \Figref{fig:hessian_spectrum}. SGDF brings lower eigenvalue and trace of the hessian matrix, which explains the fact that SGDF demonstrates better performance than SGD as the categorization category increases. Note that~\Figref{subfig:hessian_adamw} shows that AdamW achieves very low hessian matrix eigenvalues and traces, but the final test set accuracy is about 4\% lower than the other methods, and that AdamW's unique decouple weight decay changes the nature of the converged solution (We apply decoupled weight decay to other algorithms and similar results occur). 

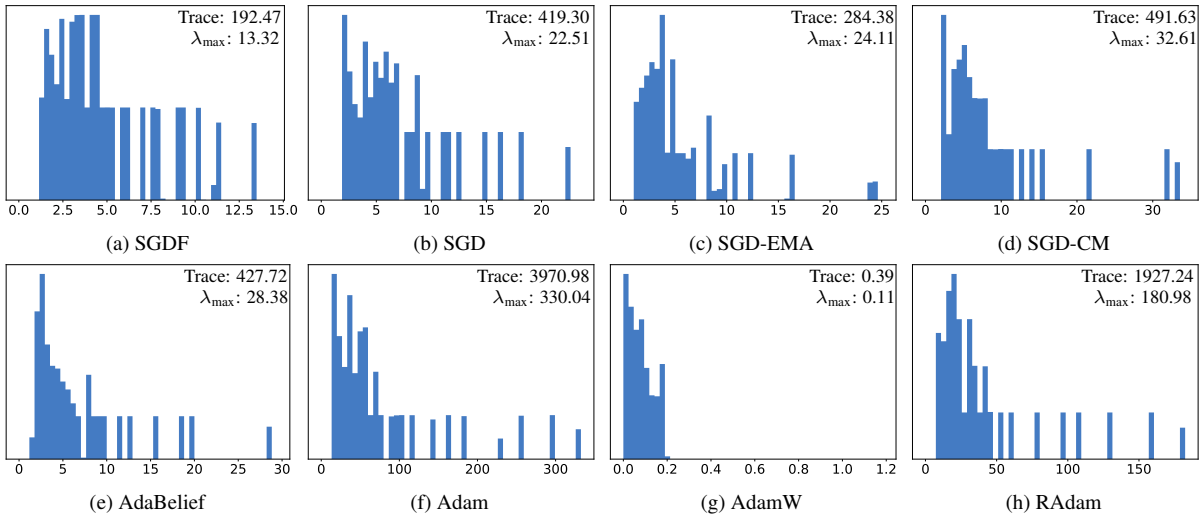
\begin{figure*}[htbp] 
    \centering
    \begin{subfigure}[t]{0.23\textwidth}
        \centering
        \begin{overpic}[width=\textwidth]{experiments/hessian_spectra/histagram_hessian/sgdf.pdf}
            \put(60,67){\scriptsize Trace: 192.47 }
            \put(64,60){\scriptsize $\lambda_{\text{max}}$: 13.32}
        \end{overpic}
        \caption{SGDF}
        \label{subfig:hessian_sgdf}
    \end{subfigure}%
    \begin{subfigure}[t]{0.23\textwidth}
        \centering
        \begin{overpic}[width=\textwidth]{experiments/hessian_spectra/histagram_hessian/sgd.pdf}
            \put(62,67){\scriptsize Trace: 419.30}
            \put(67,60){\scriptsize $\lambda_{\text{max}}$: 22.51}
        \end{overpic}
        \caption{SGD}
        \label{subfig:hessian_sgd}
    \end{subfigure}%
    \begin{subfigure}[t]{0.23\textwidth}
        \centering
        \begin{overpic}[width=\textwidth]{experiments/hessian_spectra/histagram_hessian/ema.pdf}
            \put(63,67){\scriptsize Trace: 284.38}
            \put(68,60){\scriptsize $\lambda_{\text{max}}$: 24.11}
        \end{overpic}
        \caption{SGD-EMA}
        \label{subfig:hessian_ema}
    \end{subfigure}%
    \begin{subfigure}[t]{0.23\textwidth}
        \centering
        \begin{overpic}[width=\textwidth]{experiments/hessian_spectra/histagram_hessian/sgdm.pdf}
            \put(63,67){\scriptsize Trace: 491.63}
            \put(68,60){\scriptsize $\lambda_{\text{max}}$: 32.61}
        \end{overpic}
        \caption{SGD-CM}
        \label{subfig:hessian_sgdm}
    \end{subfigure}
    
    \begin{subfigure}[t]{0.23\textwidth}
        \centering
        \begin{overpic}[width=\textwidth]{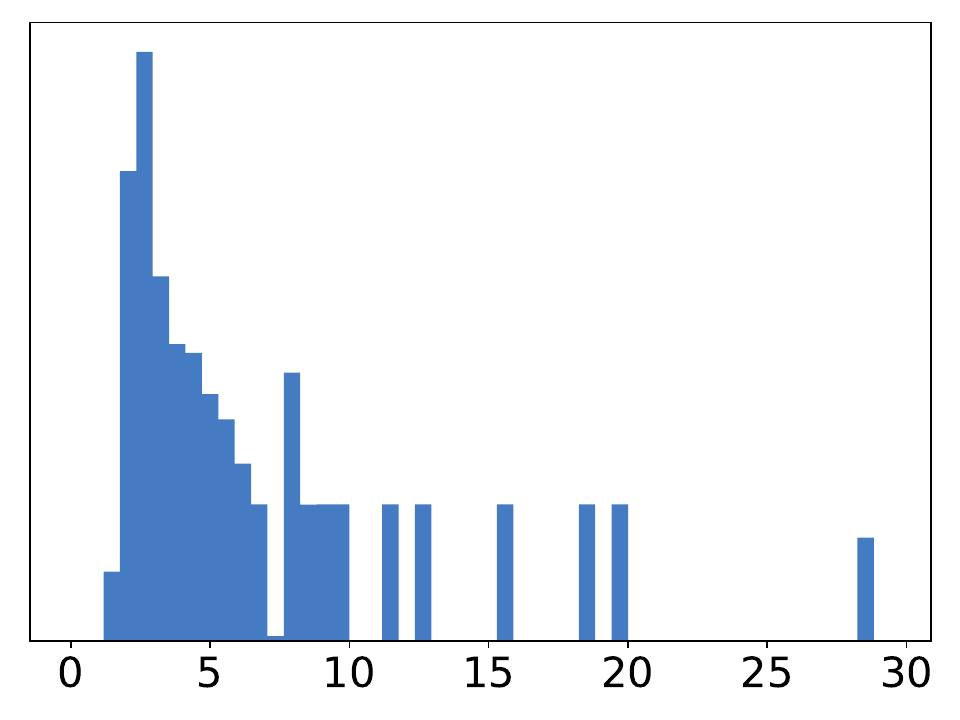}
            \put(62,67){\scriptsize Trace: 427.72 }
            \put(67,60){\scriptsize $\lambda_{\text{max}}$: 28.38}
        \end{overpic}
        \caption{AdaBelief}
        \label{subfig:hessian_adabelief}
    \end{subfigure}%
    \begin{subfigure}[t]{0.23\textwidth}
        \centering
        \begin{overpic}[width=\textwidth]{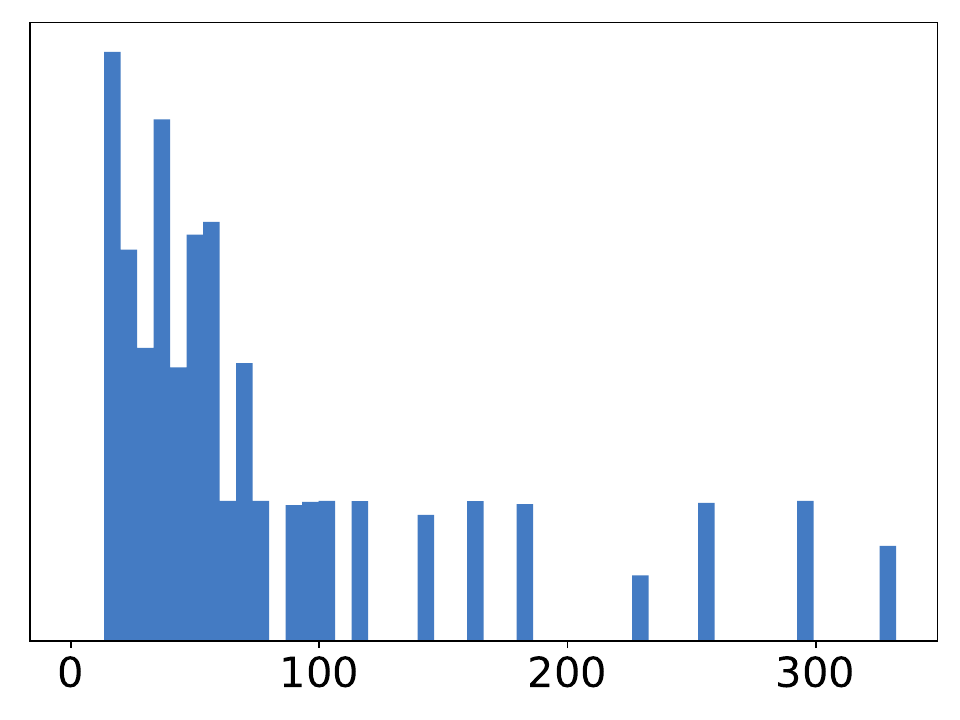}
            \put(59,67){\scriptsize Trace: 3970.98}
            \put(64,60){\scriptsize $\lambda_{\text{max}}$: 330.04}
        \end{overpic}
        \caption{Adam}
        \label{subfig:hessian_adam}
    \end{subfigure}%
    \begin{subfigure}[t]{0.23\textwidth}
        \centering
        \begin{overpic}[width=\textwidth]{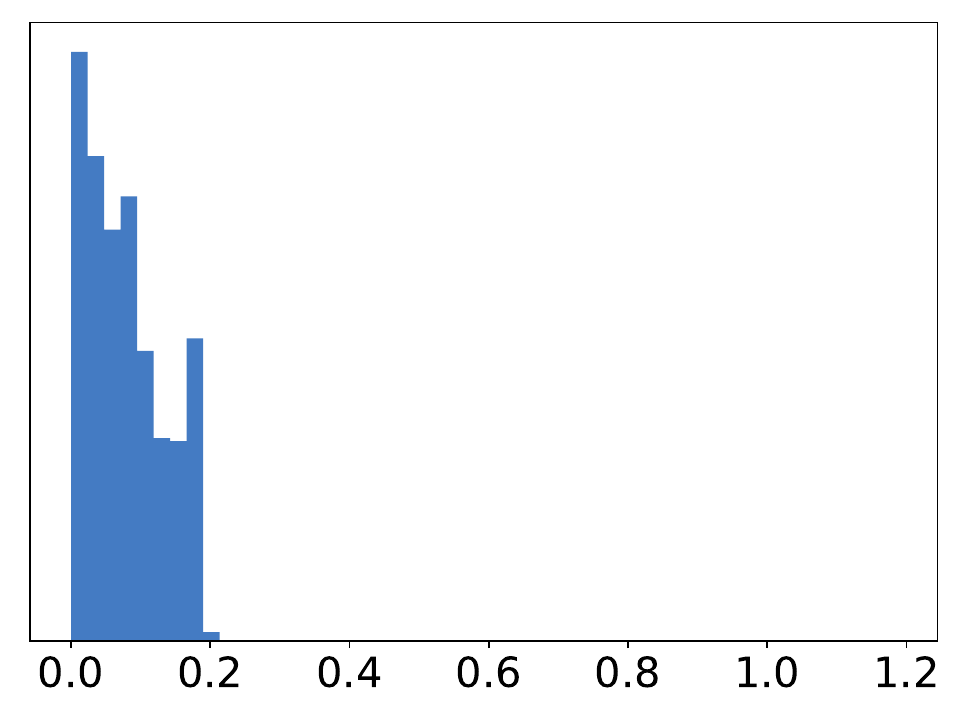}
            \put(69,67){\scriptsize Trace: 0.39}
            \put(71,60){\scriptsize $\lambda_{\text{max}}$: 0.11}
        \end{overpic}
        \caption{AdamW}
        \label{subfig:hessian_adamw}
    \end{subfigure}%
    \begin{subfigure}[t]{0.23\textwidth}
        \centering
        \begin{overpic}[width=\textwidth]{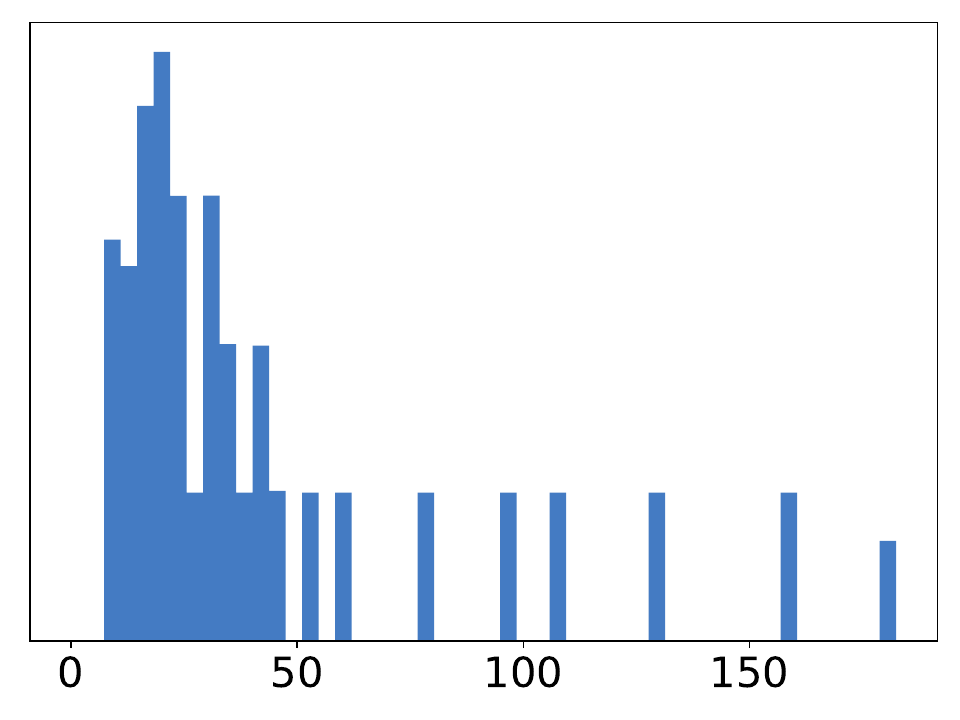}
            \put(59,67){\scriptsize Trace: 1927.24}
            \put(63,60){\scriptsize $\lambda_{\text{max}}$: 180.98}
        \end{overpic}
        \caption{RAdam}
        \label{subfig:hessian_radam}
    \end{subfigure}
    
    \caption{Histogram of Top 50 Hessian Eigenvalues.}
    \label{fig:hessian_spectrum}
\end{figure*}

\vspace{3em}
\subsection{Visualization of Landscapes}
We visualized the loss landscapes of models trained with SGD, SGDM, SGDF, and Adam using the ResNet-18 model on CIFAR-100, following the method in~\citep{li2018visualizing}. All models are trained with the same hyperparameters for 200 epochs, as detailed in \Secref{sec:section4.1}. As shown in \Figref{fig:loss_landscape}, SGDF finds flatter minima. Notably, the visualization reveals that Adam is more prone to converge to sharper minima.
\begin{figure}[htbp]
    \centering
    \begin{subfigure}[t]{0.24\linewidth}
        \centering
        \includegraphics[width=\linewidth]{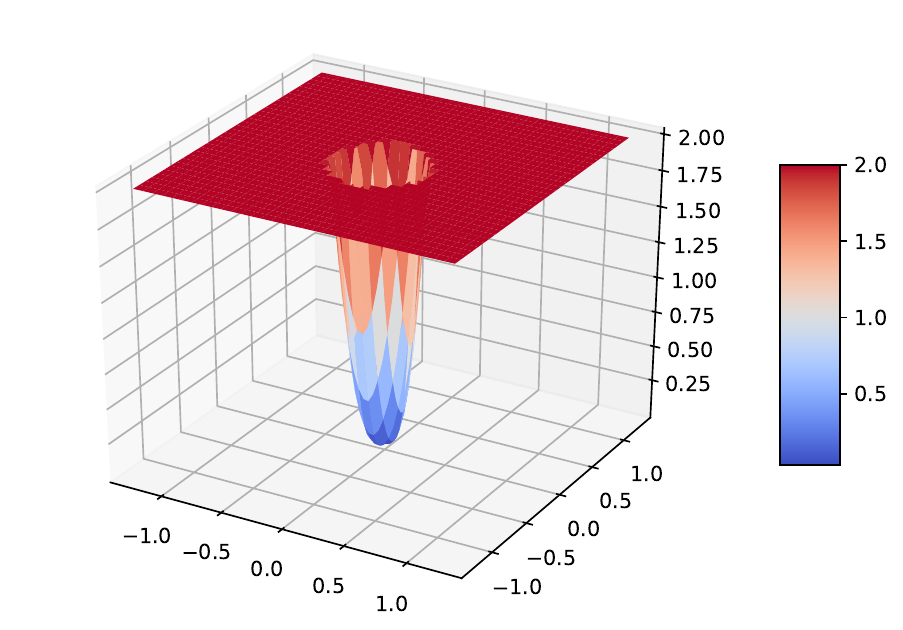}
        \caption{Adam}
        \label{subfig:adam}
    \end{subfigure}
    \hfill
    \begin{subfigure}[t]{0.24\linewidth}
        \centering
        \includegraphics[width=\linewidth]{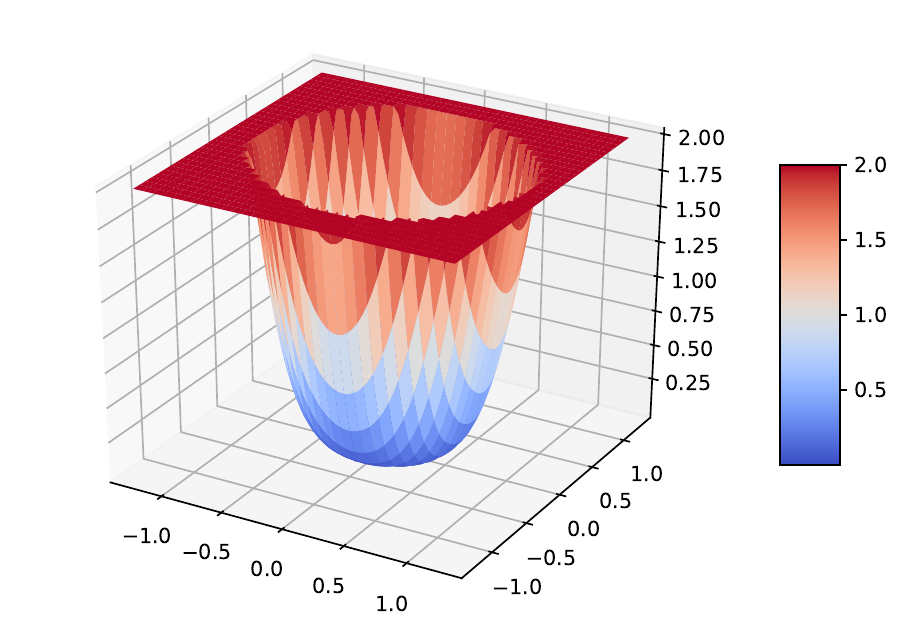}
        \caption{SGD}
        \label{subfig:sgd}
    \end{subfigure}
    \hfill
    \begin{subfigure}[t]{0.24\linewidth}
        \centering
        \includegraphics[width=\linewidth]{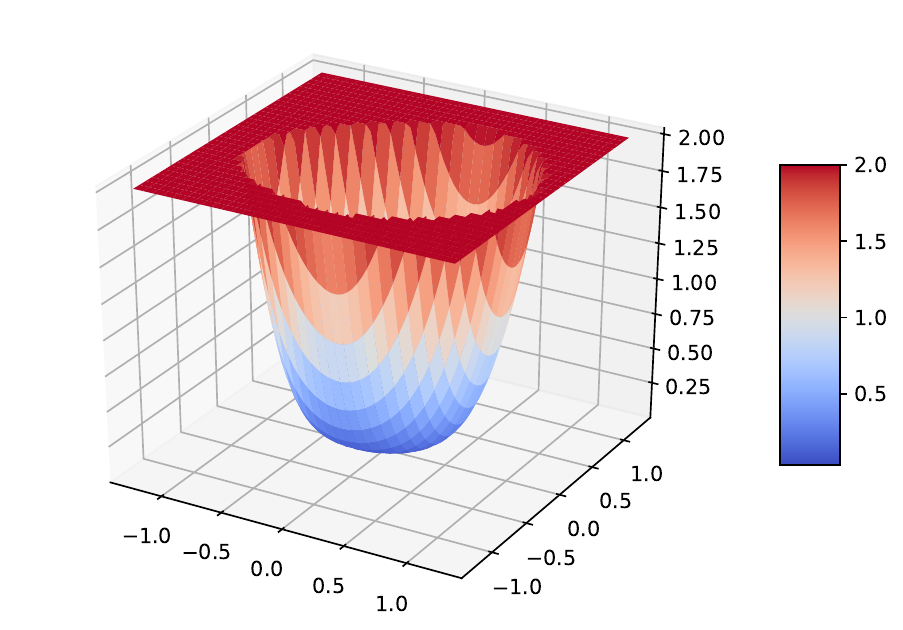}
        \caption{SGDM}
        \label{subfig:sgdm}
    \end{subfigure}
    \hfill
    \begin{subfigure}[t]{0.24\linewidth}
        \centering
        \includegraphics[width=\linewidth]{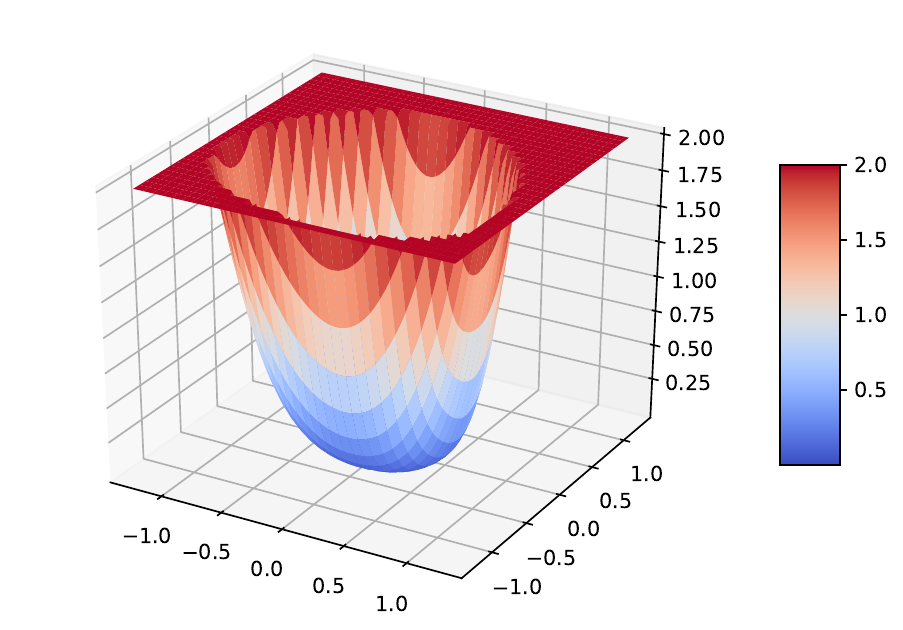}
        \caption{SGDF}
        \label{subfig:sgdf}
    \end{subfigure}

    \caption{Visualization of loss landscape. Adam converges to sharp minima.}
    \label{fig:loss_landscape}
\end{figure}

\subsection{Computational Cost Analysis}

\Tabref{tab:optimizer_flops} reports the per-parameter arithmetic cost of several optimizers. We count elementwise multiplications, additions/subtractions, divisions, and square roots as unit-cost operations. Red numbers indicate the additional overhead of \emph{coupled} weight decay, while green numbers indicate the smaller overhead of \emph{decoupled} weight decay. Compared with plain stochastic gradient methods, adaptive optimizers (e.g., Adam and SGDF) incur extra operations for moment estimation and normalization. Their optimized variants (AdamW and optimized SGDF) reduce compute by removing redundant elementwise divisions and using decoupled weight decay.

\begin{table}[htbp]
\centering
\renewcommand{\arraystretch}{1.25}
\setlength{\tabcolsep}{6pt}
\resizebox{\textwidth}{!}{
\begin{tabular}{lcl}
\toprule
\textbf{Optimizer} & \textbf{Per-Parameter FLOPs} & \textbf{Operation Breakdown} \\
\midrule
\textbf{SGD} &
$\approx$ \textbf{2 ops/param} \textcolor{red}{(+2 ops)} \textcolor{green!60!black}{(+1 op)} &
$\{\,1\times,\;1+\}$ \\[4pt]
\textbf{SGDM} &
$\approx$ \textbf{4 ops/param} \textcolor{red}{(+2 ops)} \textcolor{green!60!black}{(+1 op)} &
$\{\,2\times,\;2+\}$ \\[4pt]
\textbf{Adam} &
\textbf{16 ops/param $\rightarrow$ 14 ops/param} \textcolor{red}{(+2 ops)} \textcolor{green!60!black}{(+1 op)} &
$\{\,7\times,\;5+,\;3\div,\;1\sqrt{\;}\} \;\rightarrow\; \{\,7\times,\;4+,\;2\div,\;1\sqrt{\;}\}$ \\[4pt]
\textbf{SGDF} &
\textbf{22 ops/param $\rightarrow$ 20 ops/param} \textcolor{red}{(+2 ops)} \textcolor{green!60!black}{(+1 op)} &
$\{\,10\times,\;6+,\;2-,\;2\div,\;1\sqrt{\;}\} \;\rightarrow\; \{\,10\times,\;6+,\;1-,\;1\div,\;1\sqrt{\;}\}$ \\[2pt]
\bottomrule
\end{tabular}
}
\caption{Arithmetic cost and operation breakdown of optimizers.
\textcolor{red}{Red: Coupled Weight Decay}, \textcolor{green!60!black}{Green: Decoupled Weight Decay}.}
\label{tab:optimizer_flops}
\end{table}

Both \textbf{Adam} and \textbf{SGDF} reduce per-parameter cost through algebraic simplification and bias-correction restructuring.
For \textbf{Adam}, the update follows \(m_t=\beta_1 m_{t-1}+(1-\beta_1)g_t\), \(v_t=\beta_2 v_{t-1}+(1-\beta_2)g_t^2\), and \(\theta_{t+1}=\theta_t-\eta\,\frac{m_t/(1-\beta_1^t)}{\sqrt{v_t/(1-\beta_2^t)}+\epsilon}\). In \textbf{AdamW}, redundant elementwise divisions are avoided by precomputing \(\text{step\_size}=\eta\,\frac{\sqrt{1-\beta_2^t}}{1-\beta_1^t}\) and using decoupled weight decay: \(\theta \leftarrow (1-\eta\lambda)\theta - \text{step\_size}\,\frac{m_t}{\sqrt{v_t}+\epsilon}\). This reduces roughly two operations per parameter (16 \(\rightarrow\) 14) while preserving the update.

Inspired by this design, \textbf{SGDF} refines adaptive updates via the residual-variance estimate. Let \(\widehat{m}_t=m_t/(1-\beta_1^t)\) and \(\widehat{s}_t=s_t/c_{2,t}\), where \(c_{2,t}=\frac{(1+\beta_1)(1-\beta_2^t)}{(1-\beta_1)(1-\beta_1^{2t})}\). The Wiener gain is \(K_t=\frac{\widehat{s}_t}{\widehat{s}_t+(g_t-\widehat{m}_t)^2+\epsilon}\), and the filtered gradient is \(g'_t=\widehat{m}_t+K_t^{\gamma}(g_t-\widehat{m}_t)\), yielding \(\theta_{t+1}=\theta_t-\eta\,g'_t\). A direct elementwise implementation explicitly forms \(\widehat{s}_t=s_t/c_{2,t}\) (one elementwise division) and evaluates \((g_t-\widehat{m}_t)\) twice (once in \(K_t\), once in \(g'_t\)), resulting in about 22 operations per parameter.

Our optimized implementation avoids explicitly forming \(\widehat{s}_t\) by substituting \(\widehat{s}_t=s_t/c_{2,t}\) into \(K_t\) and multiplying numerator and denominator by \(c_{2,t}\), obtaining the equivalent form \(K_t=\frac{s_t}{s_t+c_{2,t}\left((g_t-\widehat{m}_t)^2+\epsilon\right)}\). This removes one elementwise division (replacing it with a scalar multiplication). In addition, we reuse the residual \(r_t=g_t-\widehat{m}_t\) across the gain computation and the final filtered update, eliminating one redundant elementwise subtraction. With decoupled weight decay, the per-parameter cost is reduced from about 22 \(\rightarrow\) 20 operations.

\begin{table}[htbp]
\centering
\renewcommand{\arraystretch}{1.25}
\setlength{\tabcolsep}{10pt}
\begin{tabular}{lccc}
\toprule
\textbf{Model} & \textbf{SGDM (h)} & \textbf{Adam (h)} & \textbf{SGDF (h)} \\
\midrule
\textbf{VGG13}       & 45.71 & 45.90 & $47.78_{\textcolor{cvprblue}{\downarrow 0.63}}$ \\
\textbf{ResNet101}   & 35.11 & 35.32 & $39.77_{\textcolor{cvprblue}{\downarrow 1.59}}$ \\
\textbf{DenseNet161} & 57.68 & 58.03 & $64.81_{\textcolor{cvprblue}{\downarrow 1.99}}$ \\
\bottomrule
\end{tabular}
\caption{Total training time (h) for 100 epochs with batch size 256 on FP16 AMP in $224^2$ Pixel ImageNet.}
\label{tab:runtime_100epoch}
\end{table}

To complement the theoretical operation analysis, we further conducted an empirical runtime evaluation to quantify the real-world computational efficiency of different optimizers. Each optimizer was benchmarked on FP16 mixed-precision training for 100 epochs with a batch size of 256 across representative CNN backbones (VGG13, ResNet101, DenseNet161) on ImageNet. The measured wall-clock times are summarized in \Tabref{tab:runtime_100epoch}.

\subsection{Ablation Study}
\label{app:subsec:ablation}
We derived a correction factor $(1- \beta_{1})(1 - \beta_{1}^{2t}) / (1 + \beta_{1})$ from the geometric progression to correct the variance of by the correction factor. So we test the SGDF with or without correction in VGG, ResNet, DenseNet on CIFAR. We report both test accuracy in \Figref{fig:correction}. It can be seen that the SGDF with correction exceeds the uncorrected one.
\begin{figure}[t]
    \centering
    \begin{subfigure}[t]{0.33\linewidth}
        \centering
        \includegraphics[width=\linewidth]{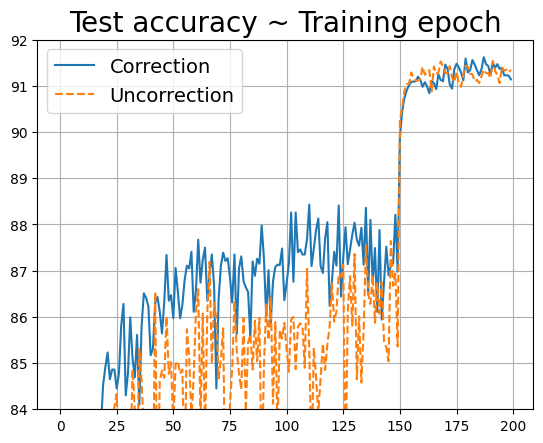}
        \caption{VGG11 on CIFAR-10}
        \label{subfig:cifar10_vgg}
    \end{subfigure}
    \hfill
    \begin{subfigure}[t]{0.33\linewidth}
        \centering
        \includegraphics[width=\linewidth]{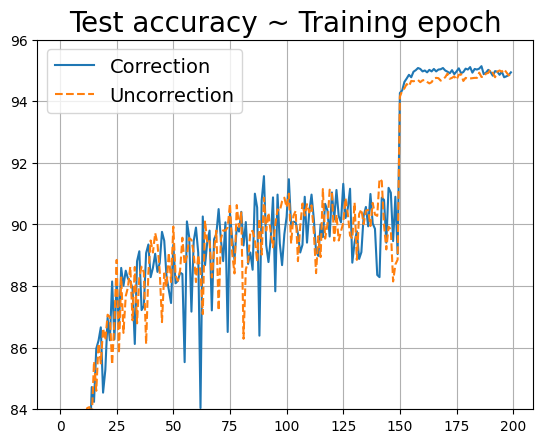}
        \caption{ResNet34 on CIFAR-10}
        \label{subfig:cifar10_resnet}
    \end{subfigure}
    \hfill
    \begin{subfigure}[t]{0.33\linewidth}
        \centering
        \includegraphics[width=\linewidth]{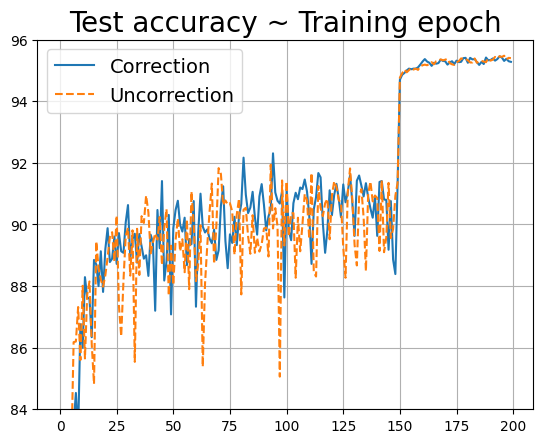}
        \caption{DenseNet121 on CIFAR-10}
        \label{subfig:cifar10_densenet}
    \end{subfigure}
    
    \vspace{2mm}
    
    \begin{subfigure}[t]{0.33\linewidth}
        \centering
        \includegraphics[width=\linewidth]{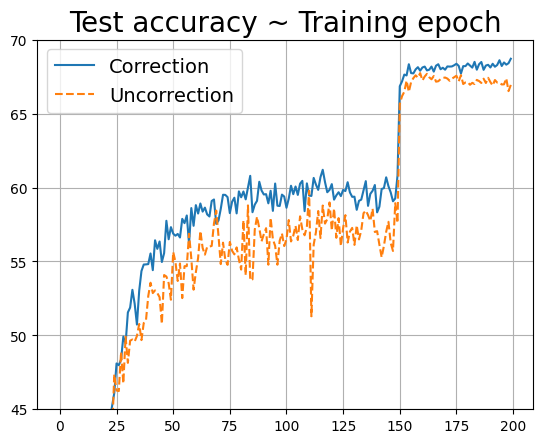}
        \caption{VGG11 on CIFAR-100}
        \label{subfig:cifar100_vgg}
    \end{subfigure}
    \hfill
    \begin{subfigure}[t]{0.33\linewidth}
        \centering
        \includegraphics[width=\linewidth]{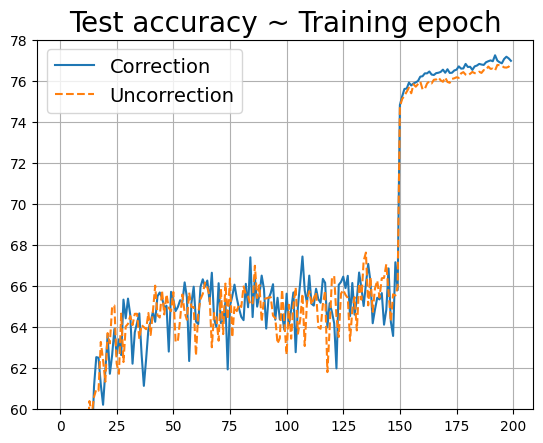}
        \caption{ResNet34 on CIFAR-100}
        \label{subfig:cifar100_resnet}
    \end{subfigure}
    \hfill
    \begin{subfigure}[t]{0.33\linewidth}
        \centering
        \includegraphics[width=\linewidth]{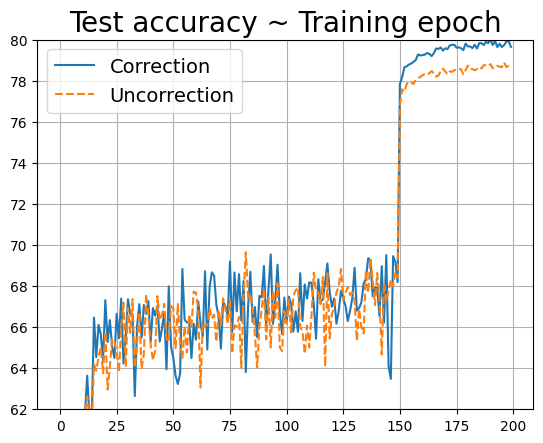}
        \caption{DenseNet121 on CIFAR-100}
        \label{subfig:cifar100_densenet}
    \end{subfigure}
    \caption{SGDF with or without the correction factor. The curve shows the accuracy of the test.}
    \label{fig:correction}
\end{figure}
\begin{figure}[t]
\centering
\begin{minipage}[b]{0.32\textwidth}
    \centering
    \includegraphics[width=\textwidth]{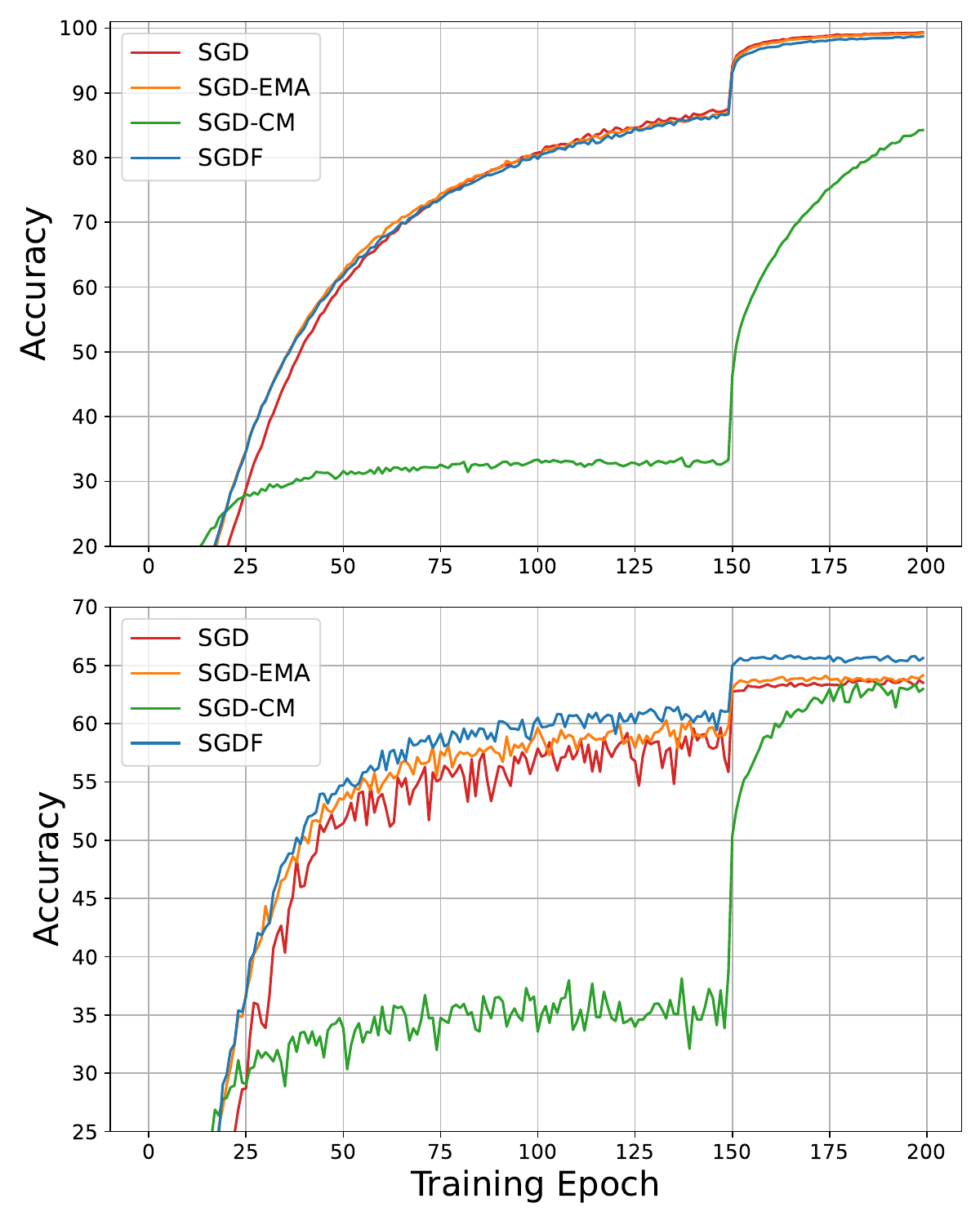}
    \caption{Train the VGG model on the CIFAR-100 dataset using the same initial learning rate of 0.1, and multiply it by a factor of 0.1 at the 150th epoch.}
    \label{fig:accuracy}
\end{minipage}
\hfill
\begin{minipage}[b]{0.66\textwidth}
    \centering
    \phantomcaption
    \begin{subfigure}[b]{0.45\textwidth}
        \centering
        \includegraphics[width=\textwidth]{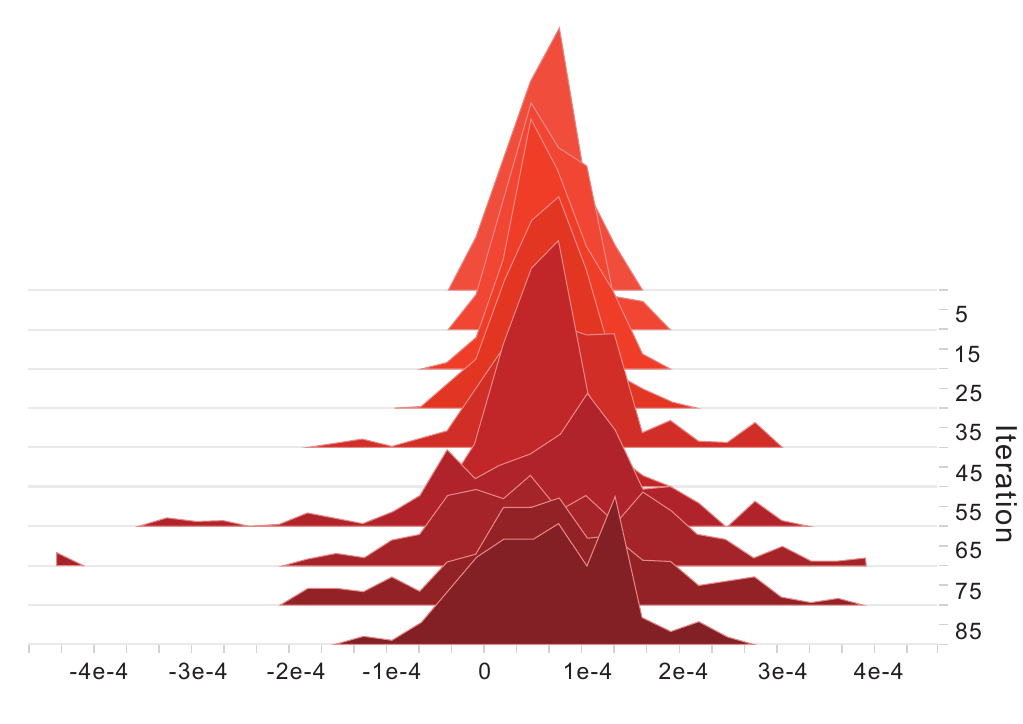}
        \subcaption{SGD}
        \label{subfig:SGD}
    \end{subfigure}%
    \begin{subfigure}[b]{0.45\textwidth}
        \centering
        \includegraphics[width=\textwidth]{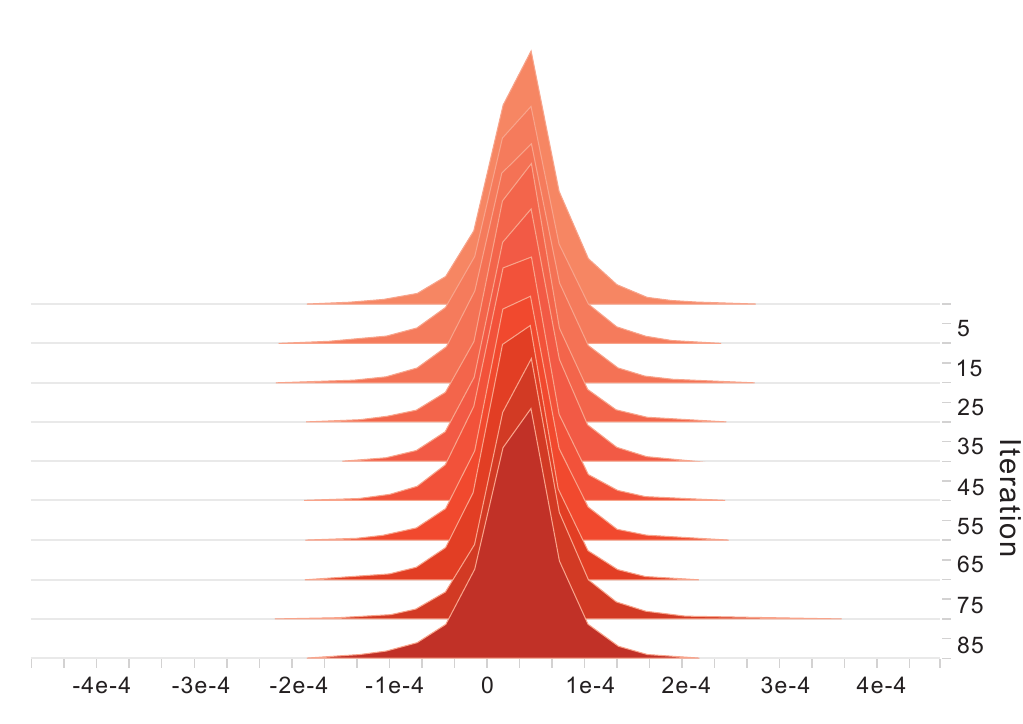}
        \subcaption{SGD-EMA}
        \label{subfig:SGD-EMA}
    \end{subfigure}
    \begin{subfigure}[b]{0.45\textwidth}
        \centering
        \includegraphics[width=\textwidth]{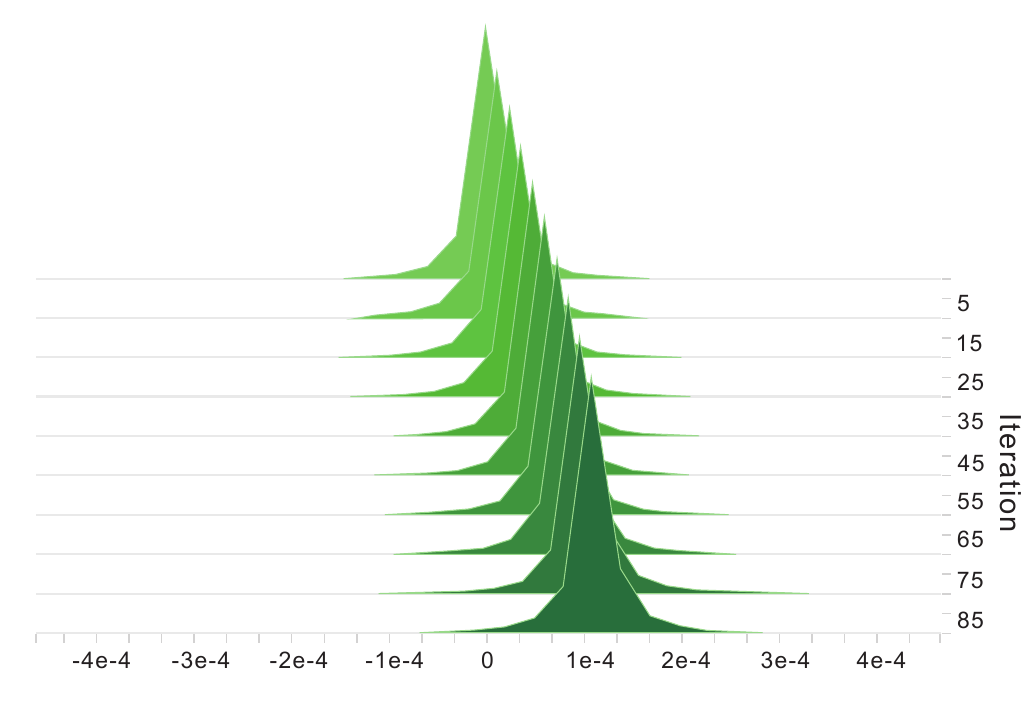}
        \subcaption{SGD-CM}
        \label{subfig:SGD-CM}
    \end{subfigure}
    \begin{subfigure}[b]{0.45\textwidth}
        \centering
        \includegraphics[width=\textwidth]{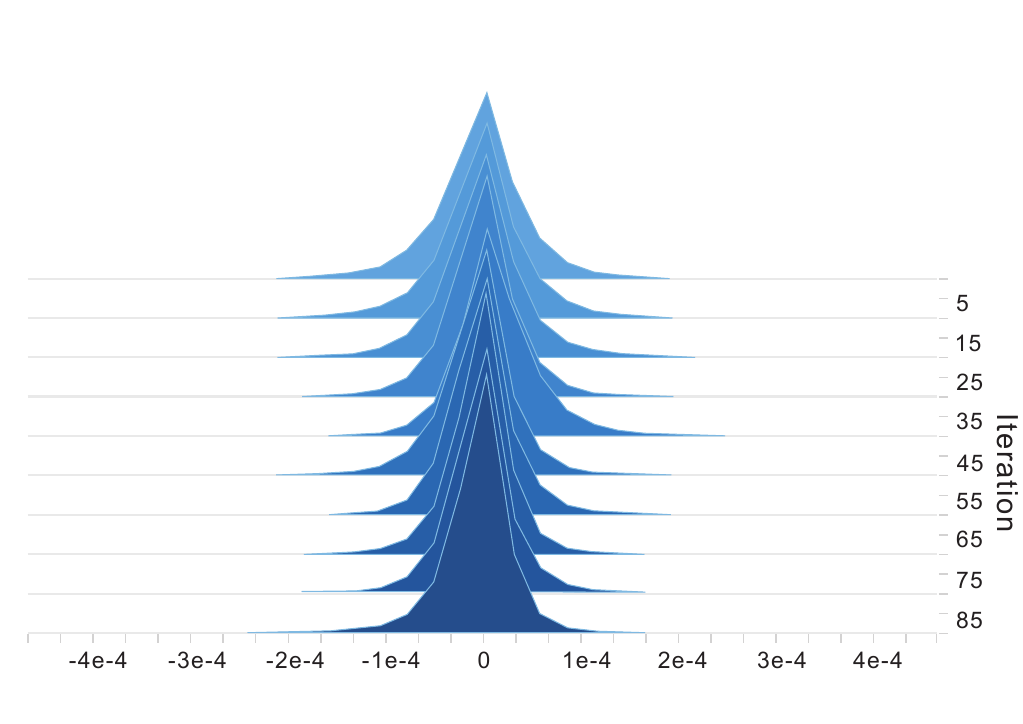}
        \subcaption{SGDF}
        \label{subfig:SGDF}
    \end{subfigure}%
    \addtocounter{figure}{-1}
    \caption{The gradient histogram of the VGG on the CIFAR-100 dataset. The x-axis is the gradient value and the height is the frequency. SGD trains the VGG without BN, the variance of the gradient fluctuates dramatically and the update is unstable.}

    \label{fig:histogram}
\end{minipage}
\end{figure}
To better observe the effect of static momentum coefficients on the gradient estimation, while comparing our time-varying SGDF. We use VGG~\cite{2014Very} because it is a very standard network with no modules that interfere with the gradient, allowing for a better representation of the optimizer's update mechanism. We trained it with different SGD-based methods: Vanilla SGD, SGD with EMA, SGD with Filter, and SGD with CM. Then, we plot curve in \Figref{fig:accuracy} and use kernel density estimates of gradient values distribution over the first 100 iterations in \Figref{fig:histogram}. 

From \Figref{fig:accuracy}, applying SGD with EMA and Filter, convergence is faster than vanilla SGD. EMA has less fluctuation in test curves. WF demonstrates higher test accuracy with the same training set accuracy and reduced generalization gap. On the other hand, CM is slow to converge and results fluctuate because of the larger bias and variance.

\Figref{subfig:SGD} shows high variance and uneven gradient values distribution in Vanilla SGD, resulting in training oscillations that hinder stable convergence. In contract, \Figref{subfig:SGD-EMA} and \Figref{subfig:SGDF} shows concentrated gradient distribution and not distorted. Especially, \Figref{subfig:SGD-CM} shows that SGD-CM smooths values fluctuations but introduces \textit{gradient shift}, causing bias and variance over time. Previous research highlights that momentum struggle to adapt to variations in the curvature of the objective function, potentially causing deviation in updates~\cite{dozat2016incorporating,zhang2017yellowfin}. 

\subsection{Extensibility of Filter-Estimated Gradients}
\label{app:subsec:extensibility}
The study involves evaluating the vanilla Adam optimization algorithm and its enhancement with an Optimal Linear Filter on the CIFAR-100 dataset. \Figref{fig:compare} contains detailed test accuracy curves for both methods across different models. The results indicate that the adaptive learning rate algorithms exhibit improved performance when supplemented with the proposed first-moment filter estimation. This suggests that integrating an Optimal Linear Filter with the Adam optimizer may improve performance.

\begin{figure}[ht]
    \centering
    \begin{subfigure}[t]{0.32\linewidth}
        \centering
        \includegraphics[width=\linewidth]{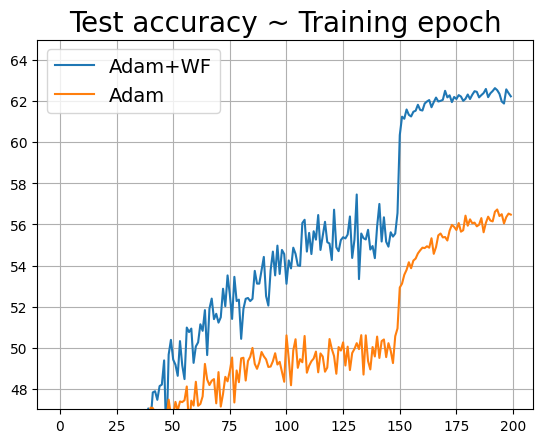}
        \caption{VGG11 on CIFAR-100}
        \label{subfig:vgg_cifar100}
    \end{subfigure}
    \hfill
    \begin{subfigure}[t]{0.32\linewidth}
        \centering
        \includegraphics[width=\linewidth]{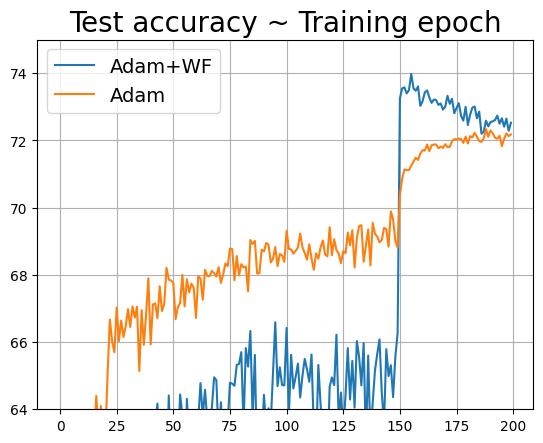}
        \caption{ResNet34 on CIFAR-100}
        \label{subfig:resnet_cifar100}
    \end{subfigure}
    \hfill
    \begin{subfigure}[t]{0.32\linewidth}
        \centering
        \includegraphics[width=\linewidth]{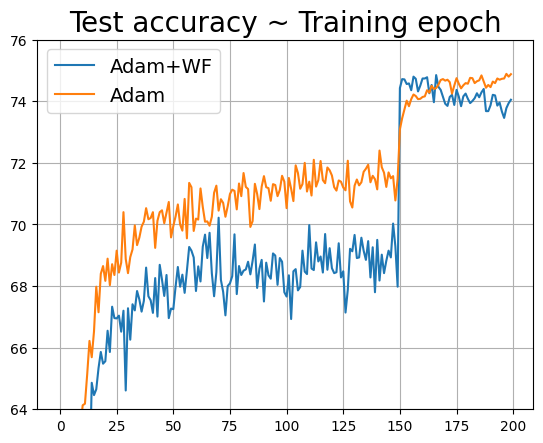}
        \caption{DenseNet121 on CIFAR-100}
        \label{subfig:densenet_cifar100}
    \end{subfigure}

    \caption{Test accuracy of CNNs on CIFAR-100 dataset. We train vanilla Adam and Adam combined with Optimal Linear Filter.}
    \label{fig:compare}
\end{figure}

\subsection{Classical Momentum Discussion}
\label{sec:appendix_discuss}
In our framework, the EMA-Momentum is treated as a low-pass filter, in the nature of noise reduction. Cutkosky~\etal~\cite{cutkosky2020momentum} also proves the property that EMA-Momentum cancels out noise, further supporting our analysis. We further discuss classical momentum here. 

Theoretically, we show that momentum converges faster than SGD in the setting of $u$-strong acceleration, but deep learning optimization does not always conform to this. Leclerc~\etal~\cite{leclerc2020two} tested the classical momentum at different learning rates, taking the momentum factor \{0, 0.5, 0.9\}. It is empirically found that it is at small learning rates that the classical momentum speeds up the convergence of training losses. That is, SGD-CM can be either better or worse than SGD. In addition, Kunstner~\etal~\cite{kunstner2023noise} found that the classical momentum can only show an advantage over SGD when the batch size increases and approaches the full gradient, at which point the noise introduced by random sampling is almost non-existent. In our proof, we mentioned that SGD-CM introduces both bias and variance, but with a full gradient, SGD-CM does not introduce noise and only causes the gradient to produce bias.

We have not analyzed the nature of bias and variance for convergent solutions, but a certain amount of bias may lead to better results when the noise is reduced, and intuitively this may help the algorithm to discuss saddle points or local minima and converge to flatter regions, in a similar nature to the implicitly flat regularity introduced by noise~\cite{yang2023stochastic}. Because the algorithm converges, the gradient at the position of convergence must be stable, and the classical momentum accumulation gradient, with its large values, must go to a smooth plateau in order to avoid oscillations. Also, it is implied that the gradient bias may not produce irretrievable results, since the bias decreases as the gradient converges, and the direction of the gradient may be more important. Sign SGD~\cite{bernstein2018signsgd} takes sign for the gradient, which also converges, and only needs to be applied to the cosine learning rate decay.

Our overall opinion is that CM does not accelerate SGD, but brings better generalization. Early deep learning optimizations focused on reducing the noise introduced by SGD, resulting in several variance reduction algorithms, where reducing variance increases the speed of convergence \cite{bottou2018optimization}. The noise introduced by CM hinders convergence, but bias brings better generalization. Thus, the above empirical observation that the momentum method can only be accelerated at small learning rates is due to the reduced step size of SGD, which naturally slows down the convergence rate. Whereas the bias from CM offsets the effect of the reduced step size, and the step size reduces the variance of the gradient sequence. This also implies why deep learning uses warm-up to make the gradient more stable in the pre-training period\cite{liu2019variance}.